\documentclass[11pt,twoside]{article}
\usepackage{fancyhdr}
\usepackage[colorlinks,citecolor=blue,urlcolor=blue,linkcolor=blue,bookmarks=false]{hyperref}
\usepackage{amsfonts,epsfig,graphicx}
\usepackage{afterpage}
\usepackage{amsmath,amssymb,amsthm,multirow} 
\usepackage{fullpage}
\usepackage[T1]{fontenc} 
\usepackage{epsf} 
\usepackage{graphics} 
\usepackage{amsfonts,amsmath}
\usepackage[sort,numbers]{natbib} 
\usepackage{psfrag,xspace}
\usepackage{color,etoolbox}

\newcommand{\ud}{\mathrm d}
\newcommand{\nml}{\mathcal{N}}

\newcommand{\kl}{\mathrm{KL}}
\newcommand{\cov}{\mathrm{cov}}

\newcommand{\argmin}{\mathrm{argmin}}

\newcommand{\conv}{\mathrm{conv}}

\newcommand{\diag}{\mathrm{diag}}

\newcommand{\tr}{\mathrm{tr}}

\newcommand{\supp}{\mathrm{supp}}

\newcommand{\re}{\mathrm{RE}}

\newcommand{\obs}{\mathrm{obs}}
\newcommand{\mis}{\mathrm{mis}}

\newcommand\indep{\protect\mathpalette{\protect\independenT}{\perp}}
\def\independenT#1#2{\mathrel{\rlap{$#1#2$}\mkern2mu{#1#2}}}

\def\mP{\mathbb{P}}

\newcommand{\missing}{\rho_*}
\newcommand{\missingj}{\rho_j}

\newcommand{\truebeta}{\beta^*}
\newcommand{\truebetasq}{\beta^{*2}}
\newcommand{\truesigma}{\Sigma_0}
\newcommand{\noise}{\epsilon}
\newcommand{\inprod}[2]{\ensuremath{\langle #1 , \, #2 \rangle}}
\newcommand{\newdim}{p}

\usepackage{mathabx}

\usepackage{times}
\usepackage{authblk}

\usepackage[title]{appendix}

\newtheorem{thm}{Theorem}
\newtheorem{lem}{Lemma}

\newtheorem{prop}{Proposition}

\newtheorem{defn}{Definition}

\usepackage{multibib}
\newcites{supp}{References}

\begin{document}
%
%
%
%
%
%
%
%
%
%
%

\title{Rate Optimal Estimation and Confidence Intervals for High-dimensional Regression with Missing Covariates}
\author[1]{Yining Wang}
\author[3]{Jialei Wang}
\author[1,2]{Sivaraman Balakrishnan}
\author[1]{Aarti Singh}
\affil[1]{Machine Learning Department, Carnegie Mellon University}
\affil[2]{Department of Statistics, Carnegie Mellon University}
\affil[3]{Department of Computer Science, University of Chicago}

\maketitle

\begin{abstract}
Although a majority of the theoretical literature in high-dimensional statistics 
has focused on settings which involve fully-observed data, 
settings with missing values and corruptions are common in practice. 
We consider the problems of 
estimation and of constructing component-wise confidence intervals 
in a sparse high-dimensional linear regression model 
when some covariates of the design matrix are missing completely at random.
We analyze a variant of the Dantzig selector \citep{candes2007dantzig} 
for estimating the regression model
and we use a de-biasing argument to construct component-wise confidence intervals.
Our first main result is to establish upper bounds on the estimation error 
as a function of the model parameters (the sparsity level $s$, the expected fraction of observed covariates $\missing$, and a measure of the signal strength $\|\truebeta\|_2$).
We find that even in an idealized setting where the covariates are assumed to be missing completely at random, somewhat surprisingly and in contrast to the fully-observed setting, there is a dichotomy in the dependence on model parameters and much faster rates are obtained if the covariance matrix of the random design is known. To study this issue further, our second main contribution is to provide lower bounds on the estimation error showing that this discrepancy in rates is unavoidable in a minimax sense. We then consider the problem of high-dimensional inference in the presence of missing data. We construct and analyze confidence intervals using a de-biased estimator. In the presence of missing data, inference is complicated by the fact that the de-biasing matrix is correlated with the pilot estimator and this necessitates the design of a new estimator and a novel analysis. 
We also complement our mathematical
study with extensive simulations on synthetic and semi-synthetic data that show the accuracy of our asymptotic predictions for 
finite sample sizes.

%
%
\end{abstract}

%
%
%
%
%
%
%
%
%
%

%
%

\section{Introduction}
High-dimensional statistics concerns the setting where the dimension of the statistical model is comparable to, or even far exceeds, 
the sample-size. In this context, meaningful statistical estimation is impossible in the absence of additional structure. Accordingly, significant research in high-dimensional statistics (see for instance \citep{tibshirani1996regression,efron2004least,donoho2006compressed,candes2006robust,fan2001variable}) has focused on high-dimensional linear regression with sparsity constraints where the goal is estimate or perform inference on a sparse, high-dimensional vector $\truebeta$ given access to noisy linear measurements. 

Modern datasets are frequently afflicted with missing-values and corruptions. 
As a canonical example consider the gene-expression dataset from \citet{nielsen2002molecular}. This dataset records $p=5520$ genes for $n=46$ patients with soft tissue tumors. 
A total of 6.7\% entries are missing; furthermore, 78.6\% of the 5520 genes and all of the 46 patients have at least one missing covariate.
Motivated by the analysis of corrupted high-dimensional datasets several researchers 
have considered settings with corrupted covariates: focusing on developing high-dimensional  analogues of the classical Expectation-Maximization (EM) algorithm \citep{stadler14}, studying their algorithmic convergence properties \citep{balakrishnan2017,han2015,xinyang15}, and understanding statistical rates of convergence for other estimators \citep{loh2012corrupted,loh2012high,loh2015regularized,belloni2016linear,rosenbaum2010sparse,rosenbaum2013improved,chen2013noisy}.

Despite extensive past work, several challenging and important 
open questions remain in establishing
the correct dependence of the rates of convergence in missing data problems 
on model parameters (the sparsity level $s$, the expected fraction of unobserved covariates $\missing$, and the signal strength $\|\truebeta\|_2$). Understanding these dependencies 
for the problems of high-dimensional estimation and inference are the focus of this work.

\subsection{Preliminaries}
We focus on a random design regression model where 
we observe i.i.d. samples of $y \in \mathbb{R}$, 
linked to a covariate $X \in \mathbb{R}^{\newdim}$ through the linear model: 
\begin{align}
\label{eq:highdim}
y_i = \inprod{X_i}{\truebeta} + \noise_i,
\end{align}
where $\noise_i$ is i.i.d. mean zero Gaussian noise, i.e. $\noise_i \sim N(0,\sigma_{\varepsilon}^2).$ 
Popular estimators include the LASSO \citep{tibshirani1996regression}, the SCAD \citep{fan2001variable}
and the Dantzig selector \citep{candes2007dantzig},
whose asymptotic rates of convergence and model selection properties are well understood \citep{zhao2006model,bach2008consistency,bickel2009simultaneous,wainwright2009sharp}.
We further consider the setting where covariates are missing completely at random, i.e. rather than observe the covariates $X_i$, we observe $\overline{X}_i$ where,
\begin{align}
\label{eq:xbar}
\overline{X}_{ij} = \begin{cases}
\star~~\text{with probability}~1 - \missingj \\
X_{ij}~~\text{otherwise,}
\end{cases}
\end{align}
where we assume that the probabilities $\missingj$ are known and define 
\begin{align*}
\missing = \min_{1\leq i \leq p} \missingj.
\end{align*} 
Our goal is to either estimate or to construct 
coordinate-wise confidence intervals for the unknown vector $\truebeta.$ In the high-dimensional setting, the number of observed samples $n$ can be much smaller than $\newdim$ and consistent estimation is impossible without additional structural assumptions. Accordingly, we study sparse models where $\truebeta$ has at most $s$ nonzero components, where $s$ is allowed to grow with $\newdim$ and $n$, but satisfies $s \ll n$. 

We emphasize that in this model, and indeed in many practical settings (for instance in the dataset of \cite{nielsen2002molecular}), most samples will have corrupted covariates and as a result complete-case analyses \cite{little86} are wasteful. Methods based on data imputation \cite{little86} typically require stronger knowledge about the generative process which can be difficult to justify in a high-dimensional setting and taking into account the imputation error in subsequent inference can be challenging.

\subsection{Related work}
Classical work on statistical estimation and inference in the presence of missing data is extensive (see for instance \cite{carroll1995measurement,hwang86,little86} and references therein), and we focus in this section on closely related works focusing on the sparse high-dimensional setting. 

\citet{rosenbaum2010sparse} proposed the Matrix Uncertainty (MU)-selector for high-dimensional regression under an error-in-variables model,
where the design matrix $X$ is observed with deterministic 
measurement error $W$ that is bounded in the matrix maximum norm. 
Optimization algorithms and minimax rates when $W$ is Gaussian white noise are considered in the work \citep{belloni2016linear}.
The MU-selector was generalized to handle the missing data setting in the paper \citep{rosenbaum2013improved}, and it was found that de-biasing the estimator of the covariance matrix led to improved error bounds. \citet{datta2015cocolasso} proposed \textsc{CocoLasso}, a variant of the LASSO for error-in-variable models
where a covariance estimate $\widehat\Sigma$ is first projected onto a positive semi-definite cone so that the resulting LASSO problem is convex.
Both additive and multiplicative measurement error models were considered in this work
and corresponding rates of convergence were derived.

\citet{loh2012high} analyzed a gradient descent algorithm for optimizing a non-convex LASSO-type loss function and derived rates of convergence
from both statistical and optimization perspectives.
Their analysis shows a dependency on $1/\missing^4$  for the $\ell_2^2$ 
estimation error. 
{A similar rate of convergence was established in \citep{chen2013noisy} for orthogonal matching pursuit (OMP) type estimators,
and \citet{rosenbaum2013improved,datta2015cocolasso} for MU-selector and \textsc{CocoLasso} formulations.}
On the lower bound side, \cite{loh2012corrupted} derived lower bounds on the minimax rate, under the assumptions of identity covariance for the design points and bounded signal level $\|\truebeta\|_2$.
Their lower bounds depend linearly on $1/\missing$. 
{\cite{belloni2016linear} showed that the dependency on $\|\beta^*\|_2$ is necessary for error-in-variable models of high-dimensional regression.
However, subtle differences exist between the error-in-variables models considered in \cite{belloni2016linear} and the missing data model consider in this paper, 
which are reflected in the dependency on the missing rate $\rho_*$ and the interplay between the two terms of $\sigma_\varepsilon$ and $\|\beta^*\|_2$,
which exhibit different levels of dependency on $\rho_*$.
}

The gap between the upper and lower bounds of prior work on estimation \citep{loh2012high,loh2012corrupted} motivate part of this work. We show that in the setting where the design covariance is assumed known a linear dependence on $1/\missing$ is achievable, whereas in the case when the covariance matrix is unknown a dependence on $1/\missing^2$ is unavoidable. We
provide a sharper upper bound than that of \citet{loh2012high}, and further provide a novel lower bound for the setting with unknown covariance. These results taken together reveal an interesting phenomenon where the rates of estimation depend on whether the covariance matrix of the random design is assumed to be known\footnote{Taking the viewpoint of semi-supervised estimation \cite{chapelle10,lafferty07}, these results show that, in contrast to linear regression in the uncorrupted setting, unlabeled data, i.e. covariates $\overline{X}_i$ with no associated $y_i$ can be useful in settings with missing data.}.
From a practical standpoint, when $\missing$ is small the difference between estimators that have dependence $1/\missing^4$ and those that depend on $1/\missing$ can be significant and we investigate these issues further via extensive simulations. 

Recent work in high-dimensional statistics 
has focused on inference for (low-dimensional projections of) $\truebeta$ \citep{javanmard2014confidence,cai2014geometric,zhang2014confidence,vandegeer2014asymptotically}. We consider this problem, in the missing completely at random model described in~\eqref{eq:highdim}, and analyze the performance of a de-biased version of the Dantzig selector. 
An important distinction between existing de-biasing methods and ours is that the presence of missing data causes the de-biasing matrix to be correlated with the estimator $\widehat\beta$. This in turn complicates the analysis and results in a limiting distribution that depends on the missing covariates. We use a variant of the CLIME estimator \citep{cai2011constrained} to resolve this correlation issue
and propose a data-driven estimator for the limiting variance of the de-biased estimator. 

{
While we were preparing this manuscript, \cite{belloni2017confidence} posted a paper that discusses the similar problem of constructing confidence bands for high-dimensional linear models
with measurement errors by considering an estimator based on orthogonal score functions.
Though the results of \cite{belloni2017confidence} could also be applied to missing data settings, the optimal dependency on the observation rate $\rho$ was not studied.
}




\subsection{Outline}
The remainder of the paper is organized as follows. In Section~\ref{sec:estimation} we consider the problem of estimation in the presence of missing data: in particular, Theorem~\ref{thm:estimate_upper} analyzes a variant of the Dantzig selector in both the setting where the covariance of $X$ is taken to be known and in the setting where the covariance is unknown. 
Under appropriate assumptions, these results show a $1/\missing$ dependence in the setting where the covariance is known and a $1/\missing^2$ dependence when the covariance is unknown. 
The dependency over $1/\missing$ is better than existing estimators \citep{loh2012high,chen2013noisy} under similar settings, which depend on $1/\missing^4$.
We turn to lower bounds in Theorems~\ref{thm:minimax} and~\ref{thm:minimax_rho2}, where we provide in turn minimax lower bounds for the known and unknown covariance settings, showing roughly that the previously obtained dependencies are optimal. 
In Section~\ref{sec:ci} we consider the problem of high-dimensional inference in the presence of missing data. In Theorem~\ref{thm:asymptotic_variance} we derive the limiting distribution of a de-biased Dantzig selector, while in Theorem~\ref{thm:estimate_variance} we provide an estimate of the limiting variance to allow for a practical, data-driven construction of confidence intervals. 
We provide extensive simulations on synthetic and semi-synthetic data in Section~\ref{sec:simulation}, and discuss our results and open problems in Section~\ref{sec:discussion}. We provide detailed technical proofs in Section~\ref{sec:proofs} with remaining technical aspects deferred to the Appendix.

\subsection{Notation}
For a vector $x$, we use $\|x\|_p :=\big(\sum_j{|x_j|^p}\big)^{1/p}$ to denote the $\ell_p$-norm of $x$.
For a matrix $A$, we use $\|A\|_{L_p}$ to denote the operator $p$-norm of $A$;
that is, $\|A\|_{L_p} = \sup_{x\neq 0}\|Ax\|_p/\|x\|_p$.
We also write $\|A\|_{L_\infty}$ for the maximum norm of a matrix: $\|A\|_{L_\infty} = \max_{j,k}|A_{jk}|$.
For a positive semi-definite matrix $A$, we denote by $\lambda_{\max}(A)$ and $\lambda_{\min}(A)$ the largest and smallest eigenvalues of $A$.
We use $\mathbb B_p(M)=\{x:\|x\|_p\leq M\}$ to denote the $\ell_p$ ball of radius $M$ centered at the origin.

\section{Rate-optimal Estimation}
\label{sec:estimation}

In this section we present our main results on estimation in the high-dimensional missing completely at random model. We begin with a description of our estimator which is a modified version of the Dantzig selector. As with the modified LASSO estimator (see \cite{loh2012corrupted}) the modified Dantzig selector requires a plug-in estimate of the covariance matrix. 
In contrast to the modified LASSO, the modified Dantzig selector remains a convex program even if the plug-in covariance matrix is not positive semi-definite and this leads to computational advantages as well as a simpler analysis. We subsequently state the assumptions that underlie our analysis, and then give precise statements of our upper and lower bounds. We defer proofs of these results to Section~\ref{sec:proofs}.

\subsection{The modified Dantzig selector}
We abuse notation slightly and use $\overline{X}$  to denote the observed covariates in~\eqref{eq:xbar} with zero-imputation, i.e. with each $\star$ replaced by 0.
We denote unbiased estimators of $X$ and its covariance matrix by $\widetilde X\in\mathbb R^{n\times p}$ and $\widetilde\Sigma\in\mathbb R^{p\times p}$ which we define as
\begin{align}
\label{eq:sigmatilde}
\widetilde X_{ij} := \frac{\overline{X}_{ij}}{\rho_j}, \;\;\;\;
\widetilde\Sigma := \frac{1}{n}\widetilde X^\top\widetilde X - D~\diag\left(\frac{1}{n}\widetilde X^\top\widetilde X\right),
\end{align}
where $D=\diag(1-\rho_1,\cdots,1-\rho_p)$ is a known $p\times p$ diagonal matrix.
It is a simple observation that, conditioned on $X$, $\mathbb E[\widetilde X] = X$ and $\mathbb E[\widetilde\Sigma] = \widehat\Sigma = \frac{1}{n}X^\top X$.
Our modified Dantzig selector is defined as the solution to the convex program: 
\begin{equation}
\widehat\beta_n \in \argmin_{\beta\in\mathbb R^p}\left\{ \|\beta\|_1: \left\|\frac{1}{n}\widetilde X^\top y - \widetilde\Sigma\beta\right\|_{\infty}\leq\widetilde\lambda_n\right\},
\label{eq:noisy_dantzig}
\end{equation}
where $\widetilde\lambda_n>0$ is a tuning parameter. 
Eq.~(\ref{eq:noisy_dantzig}) is a variant of the Dantzig selector \cite{candes2007dantzig} and is in principle similar to the MU-selector in \cite{rosenbaum2010sparse}.
We note again that the estimator in~\eqref{eq:noisy_dantzig} is always a convex optimization problem (regardless of whether $\widetilde\Sigma$ is positive semi-definite) and hence 
can be efficiently computed.

We also consider a variant of the modified Dantzig selector for the idealized scenario where the population covariance $\truesigma = \mathbb{E}[X^\top X]$, for the design matrix is known.
In particular, we define $\widecheck\beta_n$ as the solution of
\begin{equation}
\widecheck\beta_n \in \argmin_{\beta\in\mathbb R^p}\left\{ \|\beta\|_1: \left\|\frac{1}{n}\widetilde X^\top y - \Sigma_0\beta\right\|_{\infty}\leq\widecheck\lambda_n\right\},
\label{eq:noisy_dantzig_population}
\end{equation}
where we replace the covariance estimate $\widetilde\Sigma$ with the known population covariance $\Sigma_0$.
Noting that the high-dimensional covariance matrix $\truesigma$ is rarely known in practice,
we introduce and analyze this estimator primarily as a theoretical benchmark. 


\subsection{Assumptions}
The analysis in subsequent sections of our paper rely on certain assumptions on the covariates, the noise and the missingness mechanism:
\begin{enumerate}
\item[(A1)] \emph{Homogenous Gaussian noise}: For each $i \in \{1,\ldots,n\}$, the stochastic noise is independent and identically distributed with $\varepsilon_i \sim N(0,\sigma_\varepsilon^2)$ for some (known) $\sigma_\varepsilon<\infty$.
\item[(A2)] \emph{Sub-Gaussian random design}: Each row of $X$ is sampled i.i.d.~from some underlying sub-Gaussian distribution with covariance $\Sigma_0$ and (known) sub-Gaussian parameter $\sigma_x<\infty$.
We further suppose that the population covariance is well-conditioned, i.e. that $0 < \lambda_{\min}(\Sigma_0) \leq \lambda_{\max}(\Sigma_0) < \infty$. 
For notational simplicity we take $\Sigma_0$ to be implicit and use $\lambda_{\min},\lambda_{\max}$ instead in the rest of this paper.
\item[(A3)] \emph{Missing completely at random}: Each covariate $j \in \{1,\ldots,p\}$ has entries missing completely at random with probability of observing each entry being equal to $\missingj$, and define 
$\rho_* = \min_{1\leq j\leq p}\rho_j > 0$.
 \item[(A4)] \emph{Sparsity}: The support set $J_0=\supp(\truebeta)=\{j:|\truebeta_j|\neq 0\}$ satisfies $|J_0|\leq s$ for some $s\ll n$.
\end{enumerate}
The assumptions are standard in theoretical work on high-dimensional regression with missing data. We note that assumption (A2) implies (with high probability) a deterministic Restricted Eigenvalue (RE) condition~\citep{bickel2009simultaneous} on the sample covariance of $X$.

\subsection{Rates of convergence and minimax lower bounds}
We now turn our attention to providing rates of convergence and minimax lower bounds on the estimation error. 
Theorem \ref{thm:rates_convergence} establishes upper bounds on the mean square estimation error of $\truebeta$.
Eq.~(\ref{eq:rates_knownSigma}) corresponds to the setting where the population covariance $\Sigma_0$ is known
and Eq.~(\ref{eq:rates_unknownSigma}) holds when $\Sigma_0$ is unknown.

The following result applies to the modified Dantzig selectors in~\eqref{eq:noisy_dantzig} and~\eqref{eq:noisy_dantzig_population}, where the tuning parameters are chosen as:
\begin{align*}
\widecheck\lambda_n, \widetilde\lambda_n \asymp (\sigma_x^2\|\truebeta\|_2+\sigma_x\sigma_\varepsilon)\sqrt{\frac{\log p}{\rho_* n}}.
\end{align*}
\begin{thm}
\label{thm:estimate_upper}
Assume that (A1) to (A4) are satisfied.
\begin{itemize}
\item {\bf Known Covariance: } 
If $\frac{\log p}{\rho_*^2 n}\to 0$
then
\begin{equation}
\|\widecheck\beta_n-\truebeta\|_2 = O_\mP\left\{\frac{\sigma_x^2}{\lambda_{\min}}\left(\|\truebeta\|_2\sqrt{\frac{s \log p}{\rho_* n}} + \frac{\sigma_\varepsilon}{\sigma_x}\sqrt{\frac{s \log p}{\rho_*n}}\right)\right\}.
\label{eq:rates_knownSigma}
\end{equation}
\item {\bf Unknown Covariance: } 
If $\max\left\{\frac{\sigma_x^4s\log(\sigma_xp/\rho_*)}{\rho_*^3\lambda_{\min}^2n}, \frac{\log p}{\rho_*^4 n}\right\}\to 0$, then 
\begin{equation}
\|\widehat\beta_n-\truebeta\|_2 = O_\mP\left\{\frac{\sigma_x^2}{\lambda_{\min}}\left(\|\truebeta\|_2\sqrt{\frac{s \log p}{\rho_*^2 n}} + \frac{\sigma_\varepsilon}{\sigma_x}\sqrt{\frac{s \log p}{\rho_*n}}\right)\right\}.
\label{eq:rates_unknownSigma}
\end{equation}
\end{itemize}
\label{thm:rates_convergence}
\end{thm}

\begin{rems}
\begin{enumerate}
\item The two results show that at least from the perspective of upper bounds there is a gap in the rates achieved by the modified Dantzig selector in the known and unknown covariance settings. In particular, the squared estimation error  where $\truesigma$ is known scales as $1/\missing$ while
in the setting where $\truesigma$ is unknown scales as $1/\missing^2$.
\item Compared to \citet{loh2012high} our bounds are better by an $O(1/\rho^*)$ factor for $\widehat\beta_n$ when $\Sigma_0$ is unknown and an $O(1/\rho_*^{3/2})$ factor better when $\Sigma_0$ is known.
Our bounds are not directly comparable to the work of \citet{rosenbaum2010sparse} which considers a fixed-design setting with no stochastic model assumed over $X$.
We however remark that error bounds in \citet{rosenbaum2010sparse} depend on $\|\truebeta\|_{1}$, which could be a factor of $\sqrt{s}$ worse than $\|\truebeta\|_2$.
{The dependency on $\|\beta^*\|_1$ of MU-selector type estimators was later improved by \cite{belloni2016ell} by considering an additional $\ell_\infty$ norm regularization.
The latter paper however considers the general error-in-variable models, and dependency on $\rho^*$ in a missing data model is not explicitly stated.
}
\item The conditions between $n$ and other model parameters that we require for the error bounds to hold arise from the use of Bernstein-type concentration inequalities. In the missing data setting, controlling the deviation of the empirical and true covariance matrix of $X$ (for instance) requires a careful analysis of moments of the observed matrix $\overline{X}$ and a subsequent application of Bernstein-type concentration inequalities. This leads to two distinct tail behaviours,
the more typical sub-Gaussian tail behaviour depending on the variance of the summands when $n$ is sufficiently large and the small-sample sub-exponential tail behaviour. To ease readability, we focus on the sub-Gaussian behaviour by assuming the sample size is sufficiently large. We discuss this further in Section~\ref{sec:discussion}.

\item We also note that in contrast to bounds for regression without missing data the upper bounds here, somewhat counterintuitively, deteriorate as $\|\truebeta\|_2$ gets larger. This has been observed in prior work \cite{loh2012high,balakrishnan2017} and is roughly due to the fact that as $\|\truebeta\|_2$ grows (keeping $\missing$ fixed) more information is missing in each sample.
\item We note that bounds on the $\ell_1$ estimation error follow in a straightforward way using the relationships that under the conditions of the theorem with high-probability we have that, $\|\widehat\beta_n-\truebeta\|_1\leq 2\sqrt{s}\|\widehat\beta_n-\truebeta\|_2$ and $\|\widecheck\beta_n-\truebeta\|_1\leq 2\sqrt{s}\|\widecheck\beta_n-\truebeta\|_2$. 

\end{enumerate}
\end{rems}

\noindent We now turn our attention to minimax lower bounds for the estimation error. 
We focus first on the case when the covariance matrix $\truesigma$ is assumed to be known. In this setting, we follow a similar argument to that of prior work \cite{loh2012corrupted} but we maintain the dependence on the various model parameters (particularly, $\sigma_\varepsilon$ and $\|\truebeta\|_2$) in the lower bound.

\begin{thm}
\label{thm:minimax}
{\bf Known Covariance: } 
Suppose $4\leq s<4p/5$, $\frac{s\log(p/s)}{\rho_* n}\to 0$ and $\Sigma_0=I$. Then there exists a universal constant $C_0 > 0$ and an arbitrary constant $c > 0$ such that,
\begin{multline}
\inf_{\widehat\beta_n}\sup_{\truebeta\in\mathbb B_2(M)\cap\mathbb B_0(s)}\mathbb E\|\widehat\beta_n-\truebeta\|_2^2 \\
\geq C_0\cdot \min\left\{\sigma_\varepsilon^2+\frac{1-\rho_*}{1+2c}M^2, e^{0.5c^2(1-\rho_*)s}\sigma_\varepsilon^2\right\}\cdot\min\left\{\sqrt{\frac{s\log(p/s)}{(1-\rho_*)^2n}}, \frac{s\log(p/s)}{\rho_* n}\right\}. 
\label{eq:minimax}
\end{multline}
\end{thm}
\begin{rems}\normalfont
\begin{enumerate}
\item  In the setting when $\frac{(1-\rho_*)^2s\log(p/s)}{\rho_*^2n}\to 0$ the lower bound can be simplified to:
\begin{align*}
C_0\cdot \min\left\{\sigma_\varepsilon^2+\frac{1-\rho_*}{1+2c}M^2, e^{0.5c^2(1-\rho_*)s}\sigma_\varepsilon^2\right\}\frac{s\log(p/s)}{\rho_* n}.
\end{align*}
Furthermore, if the missing rate $(1-\rho_*)$ is at least a constant and the sparsity level $s$ or the noise level $\sigma_{\varepsilon}$ is not too small,
the term $e^{c^2(1-\rho_*)s}\sigma_{\varepsilon}^2$ is negligible because it increases exponentially with $s$ (and thus does not contribute to the minimum). In this case, noting that in our lower bound both $\lambda_{\min}$ and $\sigma_x=1$, we see that
the lower bound matches the upper bound in~\eqref{eq:rates_knownSigma} upto a universal constant.
\item We note that the second term in the lower bound arises from an interesting aspect of the missing data problem, roughly $n/\exp((1-\missing)s)$ samples obtained from the model are uncorrupted. In this case, as indicated by our lower bound a complete-case analysis (simply throwing away the samples with missing covariates) will lead to a matching upper bound, i.e. an upper bound that does not depend on $\|\truebeta\|_2$.

\end{enumerate}
\end{rems}

 In the case when $\truesigma$ is unknown, our primary goal is to show that the $1/\missing^2$ dependence in the upper bound is unavoidable. 
To accomplish this we need to consider packing sets of the parameters where both the covariance matrix $\truesigma$ and the unknown regression vector $\truebeta$ are varied. This calculation is quite technical, and as we discuss further in Section~\ref{sec:discussion}, we are unable to prove a sharp lower bound on the mean-squared estimation error. Instead we consider lower bounding the minimax estimation error for estimating a single coordinate of the vector $\truebeta$, and show that this task already requires a sample-size that scales as $1/\missing^2$.
Formally, we fix a small positive constant $\gamma_0 \in (0,1/2)$ and define, 
\begin{align*}
\Lambda(\gamma_0) = \{\Sigma_0\in\mathbb S_+^{p}: 1-\gamma_0 \leq \lambda_{\min}(\Sigma_0)\leq \lambda_{\max}(\Sigma_0)\leq 1+\gamma_0\},
\end{align*}
where $\mathbb S_+^p$ is the class of all positive definite $p\times p$ matrices. We have the following result:

\begin{thm} 
\label{thm:minimax_rho2}
Suppose that $s\geq 4$, $\max\{\frac{\sigma_\varepsilon^2}{M^2\rho_* n}, \frac{1}{\gamma_0\rho_*^2 n}\}\to 0$.
Then for any fixed $j\in\{1,\ldots,p\}$ there is a universal constant $C_1 > 0$ and an arbitrary constant $c > 0$ such that,
\begin{equation*}
\inf_{\widehat\beta_n}\sup_{\substack{\truebeta\in\mathbb B_2(M)\cap\mathbb B_0(s)\\ \Sigma_0\in\Lambda(\gamma_0)}}\mathbb E|\widehat\beta_{nj}-\truebeta_j|^2
\geq C_1\cdot\max\left\{\frac{\sigma_\varepsilon^2}{\rho_*n}, \min\left(\frac{1-\rho_*}{1+2c}M^2, e^{0.5c^2(1-\rho_*)s}\sigma_\varepsilon^2\right)\frac{1}{\rho_*^2 n}\right\}.
\end{equation*}
\end{thm}
\begin{rems}
\begin{enumerate}
\item Once again for simplicity considering the case when the sparsity level $s$ is not too small, the lower bound scales as roughly $\|\truebeta\|_2^2/(\missing^2 n),$ indicating that the $1/\missing^2$ dependence obtained in the upper bound is unavoidable in general. 
\item Our lower bound is for the error of estimating a single co-ordinate of $\truebeta$, and is derived from a careful perturbation of the covariance matrix $\truesigma$ and regression vector $\truebeta$ for which we are able to analyze the KL divergence quite precisely. Extending our lower bound to obtain an $s \log (p/s)$ scaling seems to be a challenging but important avenue for further investigation and we discuss this issue further in Section~\ref{sec:discussion}.
\end{enumerate}
\end{rems}

\section{Confidence intervals for regression coefficients}
\label{sec:ci}
In this section we turn our attention to the problem of constructing confidence intervals for coordinates of $\truebeta$.
We describe a method that builds confidence intervals for $\truebeta$
by de-biasing the modified Dantzig selector. The de-biasing method builds on recent work \cite{vandegeer2014asymptotically} and requires a sufficiently accurate estimate of the precision matrix $\truesigma^{-1}$. This in turn requires the following additional assumption:
\begin{enumerate}
\item[(A5)] There exist known constants $b_0,b_1<\infty$ such that each row (and column) of $\Sigma_0^{-1}$ belongs to $\mathbb B_0(b_0)\cap\mathbb B_1(b_1)$, i.e. each row of $\Sigma_0^{-1}$ is $b_0$-sparse and $\|\truesigma^{-1}\|_{L_1} \leq b_1$.
\end{enumerate}
Condition (A5) allows us to use CLIME \cite{cai2011constrained} or the node-wise LASSO \cite{meinshausen2006} to estimate an approximate inverse of $\Sigma_0$
that asymptotically de-biases the estimate $\widehat\beta_n$ from~\eqref{eq:noisy_dantzig}.
Similar conditions for high-dimensional inference were studied in \citep{vandegeer2014asymptotically}.
We discuss potential settings where (A5) could be relaxed in Section~\ref{sec:discussion}.

\subsection{The de-biased modified Dantzig selector}
In this section, we first introduce our de-biased estimator and then analyze its asymptotic distribution. In the next section we provide a data-driven method to estimate the limiting variance of the de-biased estimator. The de-biasing procedure uses an estimate of the precision matrix which we obtain by solving the CLIME optimization program from \cite{cai2011constrained}. Formally, we choose a tuning parameter
\begin{align*}
\widetilde\nu_n\asymp \sigma_x^2b_1\sqrt{\frac{\log p}{\rho_*^2 n}}.
\end{align*}
Recalling, the matrix $\widetilde{\Sigma}$ in~\eqref{eq:sigmatilde} we define 
$\widehat\Theta$ to be the $p\times p$ matrix:
\begin{equation}
\widehat\Theta \in \argmin_{\Theta\in\mathbb R^{p\times p}}\left\{\|\Theta\|_1: \|\widetilde\Sigma\Theta-I_{p\times p}\|_{\infty} \leq \widetilde\nu_n \;\;\text{and}\;\;
\|\Theta\widetilde\Sigma-I_{p\times p}\|_{\infty}\leq\widetilde\nu_n\right\}.
\label{eq:clime}
\end{equation}
The analysis of this estimator is standard. For completeness we include a proof of the following result in the supplementary materials:
\begin{lem}
Under (A1), (A3) and (A5),
suppose $\frac{\log p}{\rho_*^2 n}\to 0$.
Then with probability $1-o(1)$ it holds that
$\max\{\|\widehat\Theta\|_{L_1},\|\widehat\Theta\|_{L_{\infty}}\}\leq b_1$ and that
\begin{align*}
\max\{\|\widehat\Theta-\Sigma_0^{-1}\|_{L_1}, \|\widehat\Theta-\Sigma_0^{-1}\|_{L_\infty}\} \leq 2\widetilde\nu_n b_0b_1.
\end{align*}
\label{lem:clime}
\end{lem}
\noindent We refer to $\widehat{\Theta}$ as the modified CLIME estimator. Given the modified Dantzig estimator $\widehat\beta_n$ in~\eqref{eq:noisy_dantzig} and the modified CLIME estimator we construct the de-biased estimator $\widehat\beta_n^u$:
\begin{equation}
\widehat\beta_n^u = \widehat\beta_n + \widehat\Theta\left(\frac{1}{n}\widetilde X^\top y-\widetilde\Sigma\widehat\beta_n\right).
\label{eq_debiasing}
\end{equation}
Our next main result derives the limiting distribution of the de-biased estimator. 
Define the matrix $\widehat\Upsilon$ as:
$$
\widehat\Upsilon_{jk} = \left\{\begin{array}{ll}
\frac{1}{n}\sum_{i=1}^n{\sum_{t\neq j}{\frac{1-\rho_t}{\rho_j\rho_t}X_{ij}^2X_{it}^2[\truebeta_t]^2}},& j=k;\\
\frac{1}{n}\sum_{i=1}^n{\sum_{t\neq j,k}{\frac{1-\rho_t}{\rho_t}X_{ij}X_{ik}X_{it}^2[\truebeta_t]^2}},& j\neq k,
\end{array}\right.
$$
and the matrix $\widehat\Gamma\in\mathbb R^{p\times p}$ as
$$
\widehat\Gamma = \frac{\sigma_\varepsilon^2}{n}X^\top X + \frac{\sigma_\varepsilon^2}{n}\widetilde D\diag(X^\top X) + \widehat\Upsilon,
$$
where $\widetilde D=\diag(\frac{1}{\rho_1}-1,\cdots,\frac{1}{\rho_p}-1).$ With these definitions in place we have the following result:



\begin{thm}
\label{thm:asymptotic_variance}
Suppose that,
\begin{equation}
\sigma_x^4b_0b_1^2\sqrt{\frac{\log^2 p}{\rho_*^4n}}
\left(\frac{\sigma_\varepsilon\sqrt{\rho_*}}{\sigma_x}+\|\truebeta\|_2\right)\left(1+\frac{s}{\lambda_{\min}b_0b_1}\right)
 \to 0.
\label{eq:inference_condition_remark}
\end{equation}
then for any variable subset $S\subseteq[p]$ with constant size it holds that with probability $1-o(1)$ over the random design $X$, 
$$
\sqrt{n}\left(\widehat\beta_n^u-\truebeta\right)_{S} \;\overset{d}{\to}\;
N\left(0, \left[\Sigma_0^{-1}\widehat\Gamma\Sigma_0^{-1}\right]_{SS}\right) 
\;\;\;\;\text{conditioned on $X$.}
$$
\end{thm}
\begin{rems}
\normalfont
\begin{enumerate}
\item We obtain the above result as a special case of a more general result. In particular, the initial estimator $\widehat{\beta}_n$ only needs to satisfy the condition that,
\begin{equation}
\sigma_x^2b_0b_1\widetilde\nu_n\left(\frac{\sigma_\varepsilon}{\sigma_x}\sqrt{\frac{\log p}{\rho_*}} + \|\truebeta\|_2\sqrt{\frac{\log p}{\rho_*^2}} + \frac{\sqrt{n}\|\widehat\beta_n-\truebeta\|_1}{\sigma_x^2b_0b_1}\right) \overset{p}{\rightarrow} 0,
\label{eq:asymptotic_variance_condition}
\end{equation}
for the conclusion of the theorem to hold.

\item It is possible to demonstrate the rate optimality of the above theorem in a certain regime. In more details,  consider the case when $\Sigma_0=I$ and the observation rates $\rho_1=\rho_2=\cdots=\rho_p=\rho_*$.
Fix a single coordinate $j$ and let $V_j :=\text{Var}(\sqrt{n}(\widehat\beta_n^u-\truebeta)_j)$ denote the rescaled mean-squared error of the $j$-th coordinate.
By Theorem \ref{thm:asymptotic_variance},
when $n$ is sufficiently large
\begin{equation}
V_j \;\overset{p}{\to}\; \widehat\Gamma_{jj} 
\;\overset{p}{\to}\; \frac{\sigma_\varepsilon^2}{\rho_*} +\frac{1-\rho_*}{\rho_*^2} \sum_{t\neq j}{[\truebeta_t]^2} 
\leq \frac{\sigma_\varepsilon^2}{\rho_*} + \frac{1-\rho_*}{\rho_*^2}\|\truebeta\|_2^2.
\label{eq:vj}
\end{equation}
Comparing this with Theorem \ref{thm:minimax_rho2}, we observe that the variance $V_j$ matches the minimax rates of coordinate-wise estimation
up to a universal constant. Formally, under the additional assumption $\sigma_\varepsilon^2\gg e^{-0.5c^2(1-\rho_*)s}\|\truebeta\|_2^2$ that $\sigma_\varepsilon$ is not exponentially small, 
we have that
$$
\limsup_{p,n\to\infty}\frac{V_j^2}{\inf_{\widehat\beta_n}\sup_{\substack{\beta \in\mathbb B_2(\|\truebeta\|_2)\cap\mathbb B_0(s),\Sigma \in\Lambda(\gamma_0)}}n\mathbb E|\widehat\beta_{nj}-\beta_{j}|^2} \leq 2C_1^{-1}(1+2c),
$$
where $C_1>0$ is the universal constant in Theorem \ref{thm:minimax_rho2}.
\item Although the de-biased estimator we propose is inspired by prior work \citep{javanmard2014confidence,cai2014geometric,zhang2014confidence,vandegeer2014asymptotically} the analysis in the missing data case is complicated by the fact that estimates of both $\widehat\Theta$ and $\widehat\beta_n$ depend on the randomness induced by the missing entries. To circumvent this issue we rely on a careful argument that relates $\widehat{\Theta}$ to its deterministic counterpart $\truesigma^{-1}$.
\item Finally, we note that the limiting covariance depends on several unobserved quantities, most problematically the true regression vector $\truebeta$ and unobserved entries of the design matrix $X$. We overcome these issues and provide and analyze a data-driven estimate of the limiting covariance matrix in the next section.
\end{enumerate}
\end{rems}

\subsection{Data-driven approximation of the limiting covariance}
To aid in the practical construction of confidence intervals we propose an estimate of the asymptotic variance and study its rates of convergence. 
Our estimates are constructed by replacing the unobserved design matrix $X$ with $\widetilde{X}$ defined in~\eqref{eq:sigmatilde} and the true regression vector $\truebeta$ with the modified Dantzig estimate $\widehat{\beta}_n$. 
Formally, we define 
\begin{align*}
\widetilde\Gamma=\frac{\sigma_\varepsilon^2}{n}\widetilde X^\top\widetilde X + \widetilde\Upsilon,
\end{align*}
where
$$
\widetilde\Upsilon_{jk} = \frac{1}{n}\sum_{i=1}^n{\sum_{t\neq j,k}(1-\rho_t)\widetilde X_{ij}\widetilde X_{ik}\widetilde X_{it}^2\widehat\beta_{nt}^2},
$$
for $j,k\in\{1,\cdots,p\}$. 
The following theorem shows that $\widehat\Theta\widetilde\Gamma\widehat\Theta^\top$ is a good approximation of $\Sigma_0^{-1}\widehat\Gamma\Sigma_0^{-1}$
when $n$ is sufficiently large:
\begin{thm}
Suppose the conclusion in Lemma \ref{lem:clime} holds, $\frac{\log p}{\rho_*^4 n}\to 0$ and $\|\widehat\beta_n-\truebeta\|_2\overset{p}{\to} 0$.
Then
$$
\left\|\widehat\Theta\widetilde\Gamma\widehat\Theta^\top - \Sigma_0^{-1}\widehat\Gamma\Sigma_0^{-1}\right\|_{\infty} =
O_\mP\left(\frac{\sigma_x^4b_1^2\log^2 p}{\rho_*^2}\left\{\left(\|\truebeta\|_2^2+\frac{\missing \sigma_\varepsilon^2}{\sigma_x^2}\right)\left(b_0\widetilde\nu_n + \sqrt{\frac{\log p}{\missing n}}\right) + \|\truebeta\|_2\|\widehat\beta_n-\truebeta\|_1\right\}\right).
$$
\label{thm:estimate_variance}
\end{thm}
\begin{rem}
Based on Theorems \ref{thm:asymptotic_variance} and \ref{thm:estimate_variance},
an asymptotic $(1-\alpha)$ confidence interval of $\truebeta_{j}$ can be computed as
\begin{equation}
{\rm CI}_j(\alpha) = \left[ \widehat \beta^u_{nj} - \frac{\Phi^{-1}(1 - \alpha/2) \sqrt{ (\widehat\Theta\widetilde\Gamma\widehat\Theta^\top )_{jj}} }{\sqrt{n}},  \widehat \beta^u_{nj} + \frac{\Phi^{-1}(1 - \alpha/2) \sqrt{ (\widehat\Theta\widetilde\Gamma\widehat\Theta^\top )_{jj}} }{\sqrt{n}} \right],
\label{eq:ci}
\end{equation}
where $\Phi^{-1}(\cdot)$ is the inverse function of the CDF of the standard Gaussian distribution.
\noindent We now turn our attention to studying the finite-sample behaviour of the modified Dantzig selector and its associated confidence intervals in a variety of simulations.
\end{rem}
\section{Simulation results}
\label{sec:simulation}
In this section, we report a variety of simulation results on synthetic and semi-synthetic data aimed at assessing the modified Dantzig selector, the limiting behaviour of the de-biased estimator and the coverage of the confidence interval proposed in~\eqref{eq:ci}.

\subsection{Synthetic data}
We fix $\sigma_\varepsilon = 0.1$ and set $\Sigma_0 = \Omega^{-1}$ where $\Omega$ is chosen to be the following banded matrix:
\begin{align*}
\Omega_{ij} = \begin{cases}
0.5^{|i-j|} & {\rm if } \, |i-j| \leq 5 \\
0 & {\rm otherwise}
\end{cases}.
\end{align*}
We assume a uniform observation rate $\rho_1=\cdots=\rho_p=\missing$, which ranges from 0.5 to 0.9.
 The support set $J_0 \subset [p]$ of $\truebeta$ is selected uniformly at random, with $|J_0| = 10$. 
 $\truebeta$ is then generated as $\truebeta_j \sim \mathrm{Bernoulli}\{+1,-1\}$ independently for $j\in J_0$ and $\truebeta_j=0$ for $j\notin J_0$.
Both the modified Dantzig selector~\eqref{eq:noisy_dantzig} and the modified CLIME estimator~\eqref{eq:clime} are computed using the alternating direction method of multipliers (ADMM) algorithm.

\subsubsection{Verification of asymptotic normality} 
We run 1000 independent realizations of our experiments
 and study the distributions of $\sqrt{n}(\widehat\beta_n^u-\truebeta)$.
 We plot the empirical distribution of
 $$
\widehat \delta_j = \frac{\sqrt{n} (\widehat \beta^u_{nj} - \truebeta_j )}{\sqrt{ (\widehat\Theta\widetilde\Gamma\widehat\Theta^\top )_{jj}}}
 $$
 together with the standard normal distribution.
 Figure \ref{fig:1500-500} shows that the empirical distribution of $\widehat\delta_j$ agrees quite well with that of the standard normal distribution.
In addition, we find that more samples are required to ensure asymptotic normality when observation rates are low (e.g., $\missing=0.5$).

\begin{figure*}[p]
\begin{center}
\includegraphics[width=0.52 \textwidth]{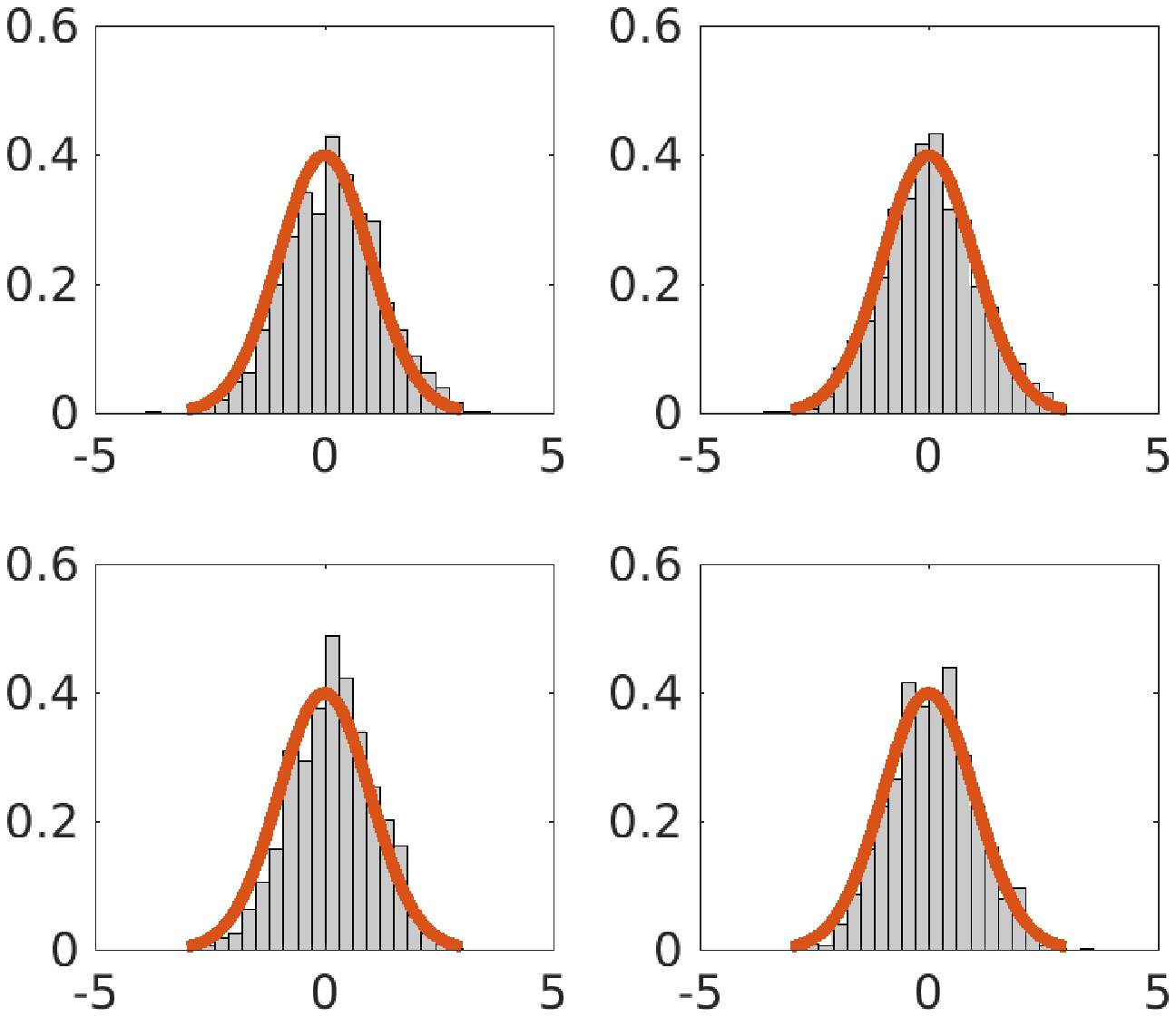}%
\includegraphics[width=0.52 \textwidth]{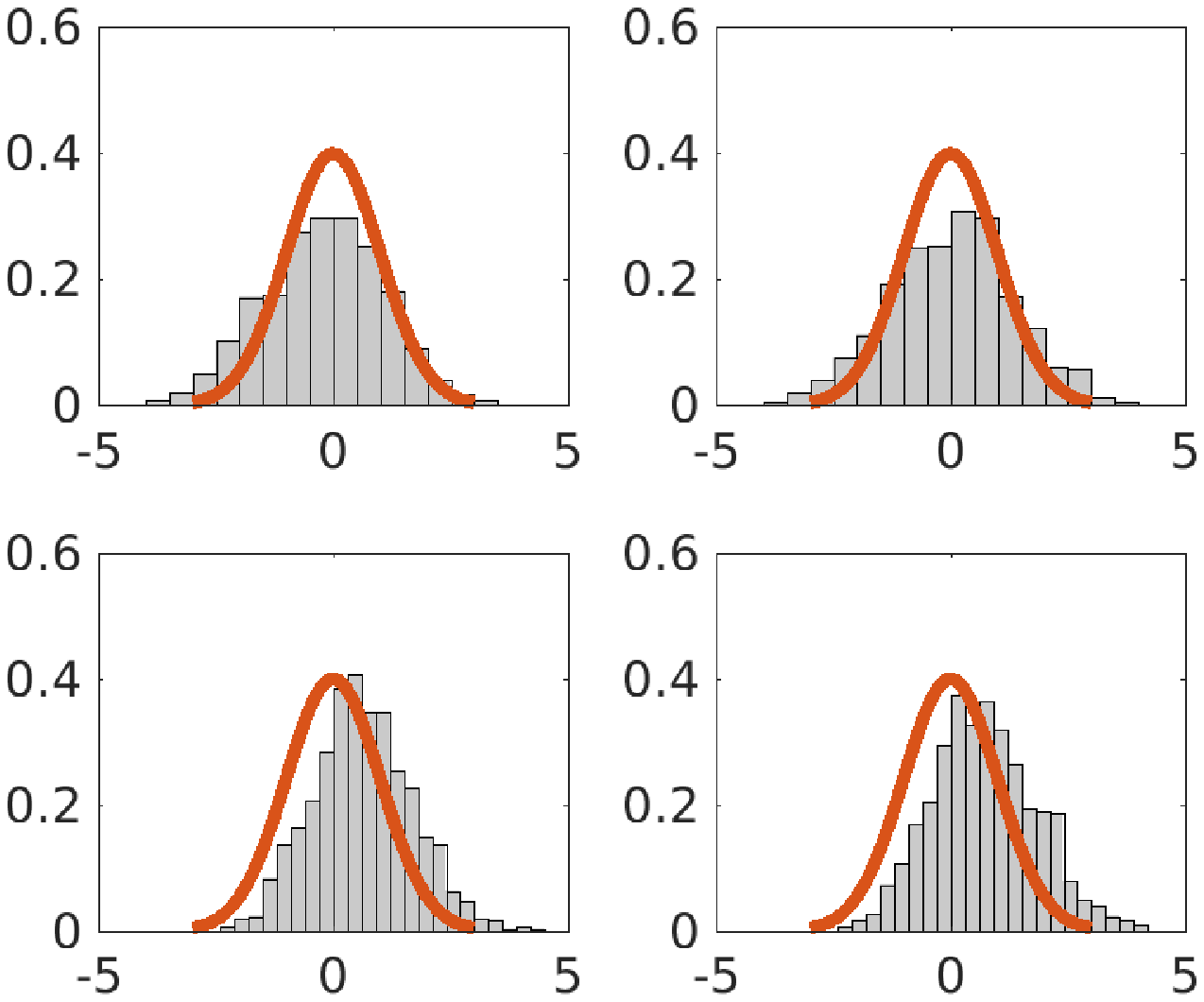}%
\end{center}
\makebox[0.52\textwidth]{$n = 1500, p = 500, \missing = 0.9$} \makebox[0.52\textwidth]{$n = 5000, p = 500, \missing = 0.7$ }
\begin{center}
\includegraphics[width=0.52 \textwidth]{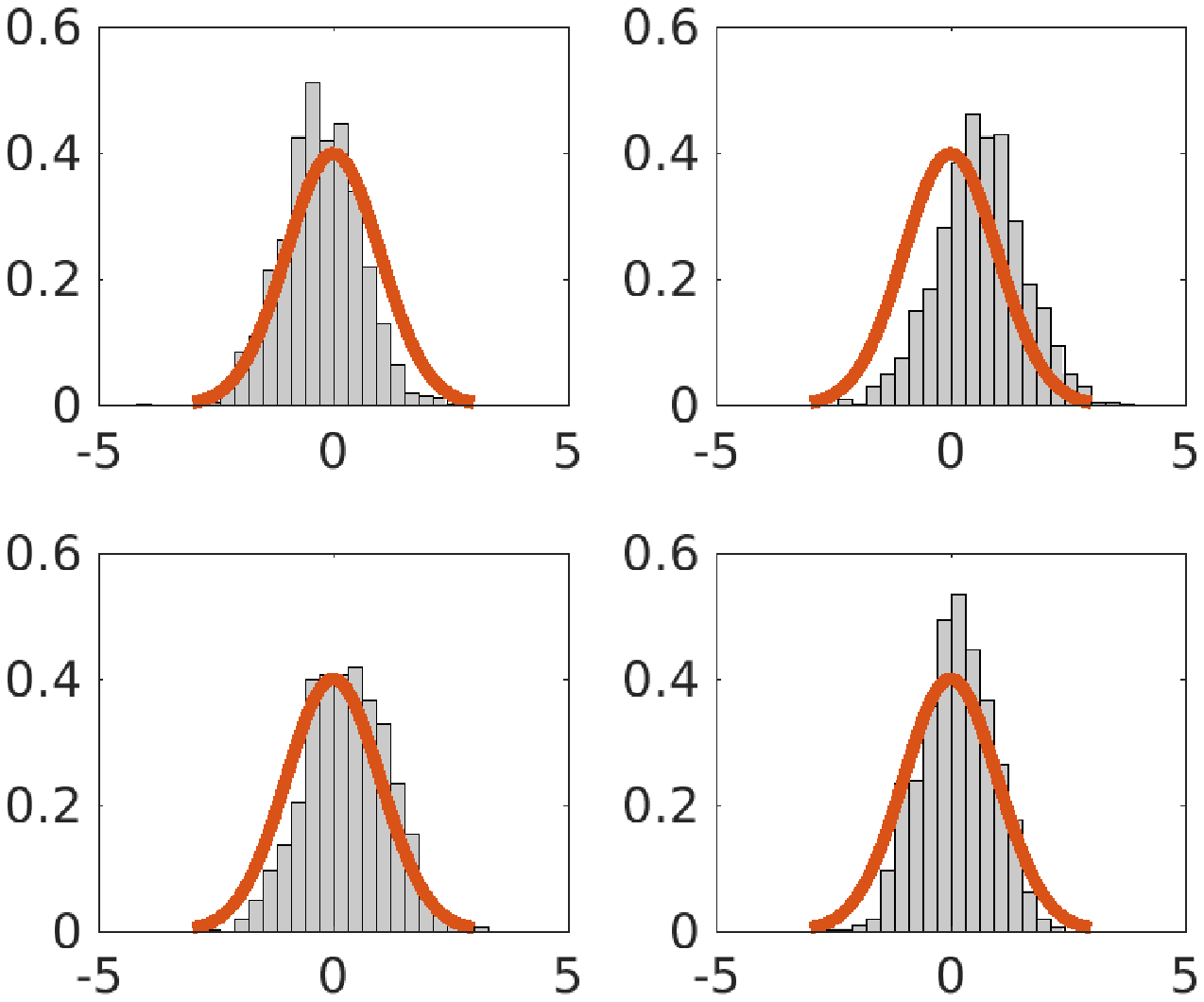}%
\includegraphics[width=0.52 \textwidth]{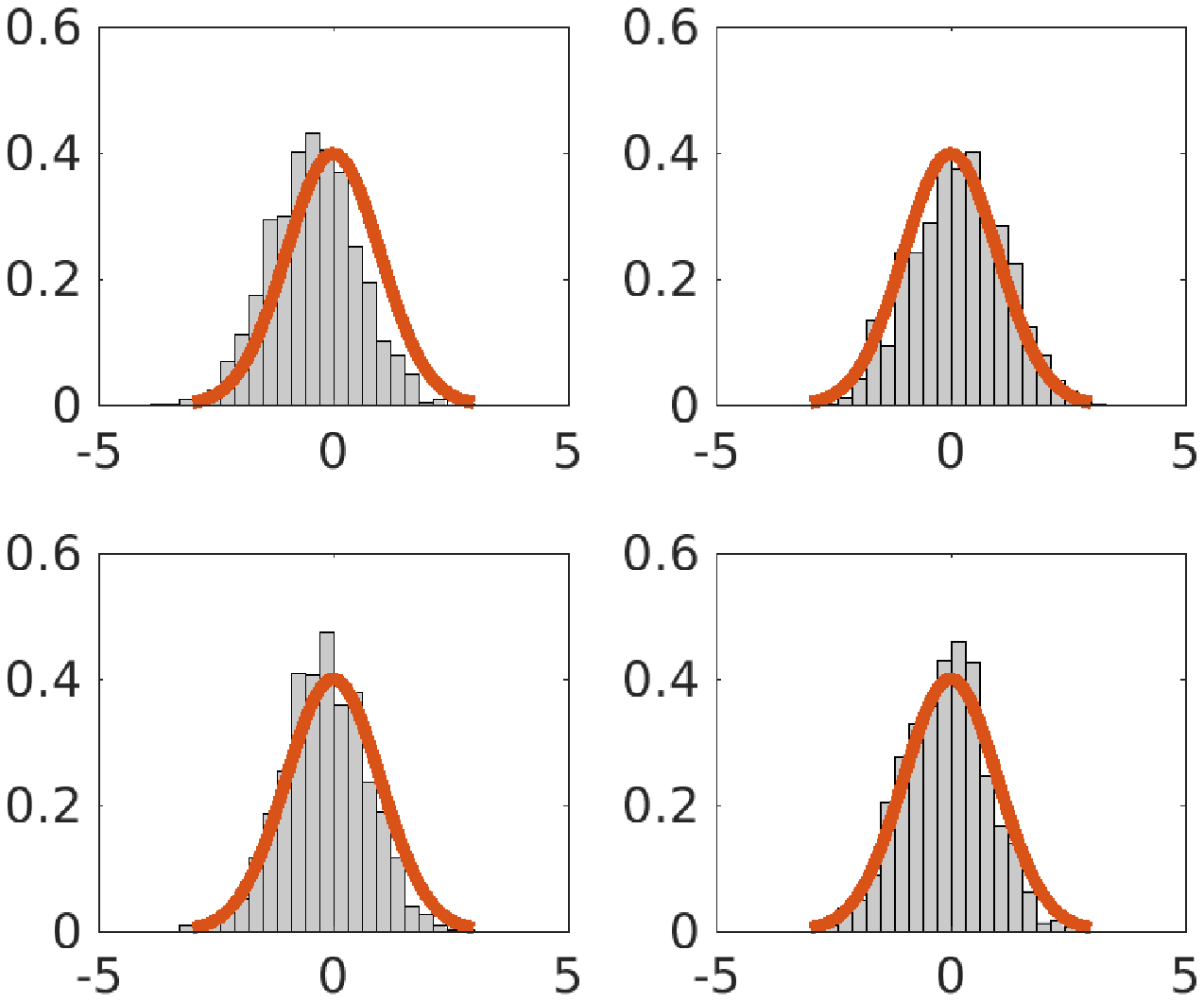}%
\end{center}
\makebox[0.52\textwidth]{$n = 8000, p = 500, \missing = 0.5$} \makebox[0.52\textwidth]{$n = 12000, p = 500, \missing = 0.5$}
\caption{Empirical distribution and density of $\widehat \delta_j = \frac{\sqrt{n} (\widehat \beta^u_{nj} - \truebeta_j )}{\sqrt{ (\widehat\Theta\widetilde\Gamma\widehat\Theta^\top )_{jj}}}$ of $1000$ independent realizations.
The top row in each subfigure corresponds to two coordinates randomly chosen from $J_0$, and the bottom row in each subfigure corresponds to two coordinates randomly chosen from $J_0^c$.
The red curve in each case denotes the density of the standard normal distribution. }
\label{fig:1500-500}
\end{figure*}

\subsubsection{Average CI coverage and length}

We calculate the average coverage and length of the constructed confidence intervals from $T$ independent realizations, defined as
$$
{\rm Avgcov}(j) = \frac{1}{T} \sum_{i=1}^T \mathbb{I}(\beta_{0j} \in {\rm CI}^{(i)}_j(\alpha)), \quad {\rm and } \quad  {\rm Avglen}(j) = \frac{1}{T} \sum_{i=1}^T {\rm length}({\rm CI}^{(i)}_j(\alpha)),
$$
where ${\rm CI}_j(\alpha)$ is defined in~\eqref{eq:ci}.
We also report the average coverage and length of coordinate-wise confidence intervals across a coordinate subset $J\subseteq[p]$, defined as
$$
{\rm Avgcov}(J) = \frac{1}{|J|}\sum_{j\in J}{\rm Avgcov(j)} \;\;\;\;\;\;\text{and}\;\;\;\;\;\;
{\rm Avglen}(J) = \frac{1}{|J|}\sum_{j\in J}{\rm Avglen(j)}.
$$
Tables \ref{tab:ci-missing} summarize the results for various $(n,p,\missing)$ settings.

\begin{table}[p]
\caption{$95\%$ confidence intervals for high-dimensional regression with missing data when $\missing\in[0.7,0.9]$.}
\label{tab:ci-missing}
\begin{center}
\begin{tabular}{|c||c|c|c|c||c|c|c|c|}\hline
\multirow{2}{*}{$(n,p,\missing)$}& \multicolumn{2}{|c|}{Random $j \in J_0$}& \multicolumn{2}{|c||}{Random $j \not\in J_0$} & \multicolumn{2}{|c|}{$J_0$}& \multicolumn{2}{|c|}{$J_0^c$}  \\
\cline{2-9}
&{Avgcov}&{Avglen} &{Avgcov}&{Avglen} &{Avgcov}&{Avglen} &{Avgcov}&{Avglen} \\
\hline
\hline
(1000,200,0.9) &0.941 &0.182 &0.951 &0.192 &0.938 &0.208 &0.966 &0.187\\
(1000,200,0.8) &0.945 &0.318 &0.948 &0.329 &0.944 &0.334 &0.979 &0.331\\
(1000,200,0.7) &0.952 &0.494 &0.983 &0.540 &0.949 &0.547 &0.989 &0.529\\
\hline
(1500,500,0.9) &0.931 &0.155 &0.966 &0.170 &0.945 &0.183 &0.971 &0.158\\
(1500,500,0.8) &0.927 &0.278 &0.982 &0.294 &0.937 &0.308 &0.985 &0.284\\
(1500,500,0.7) &0.963 &0.415 &0.994 &0.469 &0.971 &0.497 &0.995 &0.450\\
\hline
(2000,1000,0.9) &0.947 &0.144 &0.974 &0.144 &0.949 &0.160 &0.975 &0.139\\
(2000,1000,0.8) &0.967 &0.249 &0.987 &0.264 &0.939 &0.281 &0.990 &0.254\\
(2000,1000,0.7) &0.952 &0.378 &0.995 &0.422 &0.930 &0.451 &0.997 &0.409\\
\hline
(3000,2000,0.9) &0.958 &0.116 &0.954 &0.118 &0.951 &0.133 &0.981 &0.115\\
(3000,2000,0.8) &0.919 &0.202 &0.979 &0.220 &0.948 &0.236 &0.993 &0.212\\
(3000,2000,0.7) &0.891 &0.315 &0.998 &0.349 &0.950 &0.372 &0.998 &0.348\\
\hline
\hline
\end{tabular}
\end{center}
\end{table}

\begin{table}[p]
\caption{$95\%$ confidence intervals for regression with missing data when $\missing = 0.5$.}
\label{tab:ci-5}
\begin{center}
\begin{tabular}{|c||c|c|c|c||c|c|c|c|}\hline
\multirow{2}{*}{$(n,p,\missing)$}& \multicolumn{2}{|c|}{Random $j \in J_0$}& \multicolumn{2}{|c||}{Random $j \not\in J_0$} & \multicolumn{2}{|c|}{$J_0$}& \multicolumn{2}{|c|}{$J_0^c$}  \\
\cline{2-9}
&{Avgcov}&{Avglen} &{Avgcov}&{Avglen} &{Avgcov}&{Avglen} &{Avgcov}&{Avglen} \\
\hline
\hline
(1000,200,0.5) &0.928 &1.051 &0.998 &1.223 &0.942 &1.384 &0.999 &1.194\\
(2000,200,0.5) &0.971 &0.715 &0.997 &0.849 &0.971 &0.799 &0.995 &0.813\\
(3000,200,0.5) &0.956 &0.574 &0.976 &0.644 &0.961 &0.668 &0.989 &0.640\\
(4000,200,0.5) &0.936 &0.468 &0.984 &0.541 &0.943 &0.527 &0.986 &0.534\\
\hline
(1500,500,0.5) &0.986 &0.795 &0.978 &0.911 &0.756 &0.954 &1.000 &0.896\\
(3000,500,0.5) &0.849 &0.510 &0.899 &0.575 &0.479 &0.634 &0.998 &0.572\\
(8000,500,0.5) &0.972 &0.352 &0.978 &0.408 &0.908 &0.417 &0.988 &0.403\\
(12000,500,0.5) &0.941 &0.272 &0.965 &0.315 &0.936 &0.328 &0.976 &0.309\\
\hline
\hline
\end{tabular}
\end{center}
\end{table}

\subsection{Semi-synthetic data}

In this section we conduct experiments on two datasets: DNA and Madelon\footnote{Available from \url{https://www.csie.ntu.edu.tw/~cjlin/libsvmtools/datasets/}}, where the distribution of the design matrices are not necessarily sub-Gaussian. 
The DNA data contains 2000 instances and 180 covariates, while Madelon contains 2000 data points and 500 covariates. For these two datasets, we only use their data matrix $X$ and construct the response $y$ according to a sparse linear regression model. Following the simulation study, we randomly remove observed covariates with probability $1 - \missing$, and then perform statistical inference based on the datasets with missing covariates.
The performance of the constructed confidence intervals are reported in Table \ref{tab:ci-real-data}. We see that the proposed procedure produces roughly normal estimates for the parameters of interest when $\missing$ is not too small, demonstrating that the estimators and confidence intervals can be robust to violations of the assumptions on the design matrix. 

\begin{table}[t]
\caption{$95\%$ confidence intervals for regression with missing data on real world datasets.}
\label{tab:ci-real-data}
\begin{center}
\begin{tabular}{|c||c|c|c|c||c|c|c|c|}\hline
\multirow{2}{*}{$({\rm dataset},\missing)$}& \multicolumn{2}{|c|}{Random $j \in J_0$}& \multicolumn{2}{|c||}{Random $j \not\in J_0$} & \multicolumn{2}{|c|}{$J_0$}& \multicolumn{2}{|c|}{$J_0^c$}  \\
\cline{2-9}
&{Avgcov}&{Avglen} &{Avgcov}&{Avglen} &{Avgcov}&{Avglen} &{Avgcov}&{Avglen} \\
\hline
\hline
(DNA,0.9) &0.924 &0.120 &0.956 &0.128 &0.937 &0.128 &0.957 &0.129\\
(DNA,0.8) &0.908 &0.195 &0.959 &0.216 &0.926 &0.212 &0.965 &0.218\\
(DNA,0.7) &0.888 &0.286 &0.967 &0.318 &0.925 &0.314 &0.973 &0.317\\
(DNA,0.5) &0.713 &0.464 &0.964 &0.516 &0.745 &0.512 &0.976 &0.519\\
\hline
(Madelon,0.9) &0.943 &0.095 &0.963 &0.101 &0.949 &0.098 &0.945 &0.105\\
(Madelon,0.8) &0.966 &0.167 &0.976 &0.174 &0.961 &0.181 &0.971 &0.223\\
(Madelon,0.7) &0.962 &0.229 &0.977 &0.236 &0.956 &0.253 &0.977 &0.261\\
(Madelon,0.5) &0.663 &0.334 &0.977 &0.357 &0.682 &0.377 &0.965 &0.356\\
\hline
\hline
\end{tabular}
\end{center}
\end{table}

\section{Discussion}\label{sec:discussion}
In this paper, we studied the problems of estimation of and constructing confidence intervals for a high-dimensional regression vector when covariates are missing completely at random. In the context of estimation, in contrast to the situation in regression without missing data, we find a discrepancy between bounds obtained when $\truesigma$ is taken to be known and when it is unknown. We sharpen existing analyses in both these settings and develop minimax lower bounds to show that this discrepancy is unavoidable. Finally, we provide a method to construct confidence intervals in the presence of missing data through de-biasing, and study its length and coverage properties. Several important questions remain open and discuss some of these here. 

Theorem \ref{thm:minimax_rho2} shows that if the population covariance $\Sigma_0$ of the design matrix $X$ is unknown, 
then the mean square estimation error of a fixed component in $\truebeta$ must depend quadratically on the observation ratio $\missing$. We conjecture that such results also hold for the estimation error of the entire regression model $\truebeta$ as well.
More specifically, we conjecture that under suitable finite-sample conditions,
$$
\inf_{\widehat\beta_n}\sup_{\substack{\truebeta\in\mathbb B_2(M)\cap\mathbb B_0(s)\\ \Sigma_0\in\Lambda(\gamma_0)}}\mathbb E\|\widehat\beta_n-\truebeta\|_2^2
\geq C_1'\cdot\max\left\{\frac{\sigma_\varepsilon^2 s\log p}{\missing n}, \min\left(\frac{1-\rho_*}{1+2c}M^2, e^{c^2(1-\rho_*)s}\sigma_\varepsilon^2\right)\frac{s\log p}{\rho_*^2 n}\right\}.
$$
Establishing such a bound however requires a generalization of our lower bound construction in a novel fashion. In particular, our current construction relies on a carefully designed packing set of covariance matrices that do not ``leak information'' unless both $X_1$ and $X_j$ (for a fixed $j$) are observed, and extending this construction more generally appears to be challenging.

Our upper bounds for both estimation and inference focus on a large-sample regime when the Bernstein-type inequalities we use result in sub-Gaussian behaviour. In problems with missing data, the natural plug-in estimators, of the covariance matrix for instance, exhibit different rates of convergence in the small-sample regime. 
Understanding the tightness of our bounds in this small-sample regime would be interesting.

For inference we use sparsity assumptions that ensure that the precision matrix $\Sigma_0^{-1}$ is estimable, which are restrictive as the precision matrix is a nuisance parameter. In the fully observed setting weaker assumptions are used for instance in \cite{javanmard2014confidence} at the cost of 
asymptotic efficiency of the average length of the resulting confidence interval. In the missing data setting however the dependence between the estimates $\widehat\Theta$ and $\widehat\beta_n$ caused due to the missingness is challenging to deal with directly. Instead, we use arguments that relate $\widehat\Theta$ to its deterministic population counterpart $\Sigma_0^{-1}$. Understanding the extent to which this dependence can be circumvented, and weakening the assumptions required on the nuisance parameter $\truesigma^{-1}$ remains an open question.

\section{Proofs}
\label{sec:proofs}
In this section, we turn to the proofs of our main theorems. 
We include in the main text the main body of the proofs deferring more technical aspects to the supplementary material.

\subsection{Additional notation}
We use the matrix $R$ to denote the missingness pattern, i.e. define:
\begin{align*}
R_{ij} = \begin{cases}
0,~~~\text{if}~~X_{ij} = \star, \\
1,~~~\text{otherwise}.
\end{cases}
\end{align*}
\noindent In order to compactly derive and state concentration bounds for the case when $\truesigma$ is known and unknown we will use the following additional notation.
\begin{defn}
Let $A,B$ be random or deterministic square matrices of the same size and $\varepsilon$ be a random vector of i.i.d.~$\nml(0,\sigma_\varepsilon^2)$ components.
Let $\varphi_{u,v}(A,B;\log N)$, $\varphi_{u,\infty}(A,B;\log N)$, $\varphi_{\varepsilon,\infty}(A)$ be terms such that,
with probability $1-o(1)$ as $n\to\infty$, for all subset $\mathcal S$ of vectors with $|\mathcal S|\leq N$, the following hold for all $u,v\in\mathcal S$:
\begin{eqnarray*}
\big|u^\top(A-B)v\big| &\leq& \varphi_{u,v}(A,B;\log N)\cdot \|u\|_2\|v\|_2;\\
\left\|A^\top\varepsilon\right\|_{\infty} &\leq& \varphi_{\varepsilon,\infty}(A)\cdot \sigma_\varepsilon.
\end{eqnarray*}
\label{defn:varphi}
\end{defn}

Note that $\varphi_{u,v}(\cdot,\cdot)$ is symmetric and satisfies the triangle inequality.
Also, infinity norms like $\|A-B\|_{\infty}$ or $\|(A-B)u\|_{\infty}$ for a fixed $u$ can be upper bounded by 
$\varphi_{u,v}(A,B;O(\log \dim(A)))$, by considering the set of unit vectors $\{e_1,\cdots,e_{\dim(A)}\}$.

\subsection{Proof of Theorem \ref{thm:rates_convergence}}

We need the following two concentration lemmas, which are proved in the supplementary material.
\begin{lem}
Denote random matrices $A^{(\ell)}$, $\ell\in\{0,1,2\}$ as $A^{(0)}=\widehat\Sigma$, $A^{(1)}=\frac{1}{n}\widetilde X^\top X$ and $A^{(2)}=\widetilde\Sigma$, respectively.
Then for $\ell\in\{0,1,2\}$:
$$
\varphi_{u,v}\left(A^{(\ell)}, \Sigma_0;\log N\right) \leq O\left(\sigma_x^2\max\left\{\frac{\log N}{\rho_*^{1.5\ell}n}, \sqrt{\frac{\log N}{\rho_*^{\ell} n}}\right\}\right).
$$
\label{lem:main_concentration}
\end{lem}

\begin{lem}
If $\frac{\log p}{\rho_* n}\to 0$ then $\varphi_{\varepsilon,\infty}(\frac{1}{n}\widetilde X) \leq O(\sigma_x\sqrt{\frac{\log p}{\rho_* n}})$.
\label{lem:epsilon_concentration}
\end{lem}

We present the following lemma.
Its proof is given in the supplementary material.
\begin{lem} 
Suppose $\frac{\log p}{\rho_*^4 n}\to 0$ for $\widehat\beta_n$ or $\frac{\log p}{\rho_*^2 n}\to 0$ for $\widecheck\beta_n$,
and let $J_0=\supp(\truebeta)$ be the support of $\truebeta$.
If $\widetilde\lambda_n \geq\Omega\{\sigma_x\sqrt{\frac{\log p}{n}}(\frac{\sigma_x\|\truebeta\|_2}{\rho_*} + \frac{\sigma_\varepsilon}{\sqrt{\rho_*}})\}$ 
and $\widecheck\lambda_n\geq \Omega\{\sigma_x\sqrt{\frac{\log p}{\rho_*n}}(\sigma_x \|\truebeta\|_2+ \sigma_\varepsilon)\}$,
 then with probability $1-o(1)$ we have that 
 \begin{enumerate}
 \item $\|(\widehat\beta_n-\truebeta)_{J_0^c}\|_1 \leq \|(\widehat\beta_n-\truebeta)_{J_0}\|_1$;
 \item $\|(\widecheck\beta_n-\truebeta)_{J_0^c}\|_1\leq\|(\widecheck\beta_n-\truebeta)_{J_0}\|_1$.
 \end{enumerate}
\label{lem:basic_inequality}
\end{lem}

\begin{defn}[Restricted eigenvalue condition]
A $p\times p$ matrix $A$ is said to satisfy $\re(s,\phi_{\min})$ if for all $J\subseteq[p]$, $|J|\leq s$ the following holds:
$$
\inf_{h\neq 0, \|h_{J^c}\|_1\leq \|h_J\|_1}\frac{h^\top A h}{h^\top h} \;\geq\; \phi_{\min}.
$$
\end{defn}

The following lemma is proved in the supplementary material.
\begin{lem}
Suppose $\frac{\sigma_x^4 s\log(\sigma_x\log p/\rho_*)}{\rho_*^3\lambda_{\min}^2 n} \to 0$.
Then with probability $1-o(1)$, the sample covariance for the missing data problem $\widetilde\Sigma$ satisfies $\re(s,(1-o(1))\lambda_{\min}(\Sigma_0))$.
\label{lem:re_main}
\end{lem}

We are now ready to prove Theorem \ref{thm:rates_convergence} that establishes the rate of convergence of the modified Dantzig selector estimators.
We consider $\widehat\beta_n$ first.
Define $\widetilde\lambda_n\mu = \frac{1}{n}\widetilde X^\top y - \widetilde\Sigma\widehat\beta_n$.
By $y=X\truebeta+\varepsilon$, we have that
$$
\widetilde\Sigma(\widehat\beta_n-\truebeta) = \left(\frac{1}{n}\widetilde X^\top X-\Sigma_0\right)\truebeta + \left(\Sigma_0-\widetilde\Sigma\right)\truebeta - \widetilde\lambda_n\mu + \frac{1}{n}\widetilde X^\top\varepsilon.
$$
Multiply both sides by $(\widehat\beta_n-\truebeta)$ and apply H\"{o}lder's inequality:
\begin{align*}
&(\widehat\beta_n-\truebeta)^\top\widetilde\Sigma(\widehat\beta_n-\truebeta)\\
&\leq \|\widehat\beta_n-\truebeta\|_1\left\{\left\|\left(\frac{1}{n}\widetilde X^\top X-\Sigma_0\right)\truebeta\right\|_{\infty} + \left\|\left(\Sigma_0-\widetilde\Sigma\right)\truebeta\right\|_\infty + \widetilde\lambda_n\|\mu\|_\infty + \left\|\frac{1}{n}\widetilde X^\top\varepsilon\right\|_\infty\right\}\\
&\leq \|\widehat\beta_n-\truebeta\|_1\cdot O_\mP\left\{\varphi_{u,v}\left(\frac{1}{n}\widetilde X^\top X,\Sigma_0;\log p\right)\|\truebeta\|_2 + \varphi_{u,v}\left(\widetilde\Sigma,\Sigma_0;\log p\right)\|\truebeta\|_2 + \widetilde\lambda_n + \varphi_{\varepsilon,\infty}\left(\frac{1}{n}\widetilde X\right)\sigma_\varepsilon\right\}\\
&\leq \|\widehat\beta_n-\truebeta\|_1\cdot O_\mP\left\{\sigma_x^2\|\truebeta\|_2\sqrt{\frac{\log p}{\rho_*^2n}} + \sigma_x\sigma_\varepsilon\sqrt{\frac{\log p}{\rho_*n}}+\widetilde\lambda_n\right\}.
\end{align*}
Here the last inequality is due to Lemmas \ref{lem:main_concentration} and \ref{lem:epsilon_concentration}.
Suppose $\frac{\sigma_x^4 s\log(\sigma_x p/\rho_*)}{\rho_*^4\lambda_{\min}^2 n} \to 0$
and $\widetilde\lambda_n$ is appropriately set as in Lemma \ref{lem:basic_inequality}. 
We then have
\begin{equation}
\|\widehat\beta_n-\truebeta\|_1 \leq 2\|(\widehat\beta_n-\truebeta)_{J_0}\|_1 \leq 2\sqrt{s}\|\widehat\beta_n-\truebeta\|_2
\label{eq:hatbeta_basic_ineq}
\end{equation}
by Lemma \ref{lem:basic_inequality} and 
$$
(\widehat\beta_n-\truebeta)^\top\widetilde\Sigma(\widehat\beta_n-\truebeta) \geq (1-o(1))\lambda_{\min}\|\widehat\beta_n-\truebeta\|_2^2
$$
by Lemma \ref{lem:re_main}.
Chaining all inequalities we get
\begin{eqnarray*}
\|\widehat\beta_n-\truebeta\|_2 &\leq& O_\mP\left(\frac{\sqrt{s}}{\lambda_{\min}}\left\{\sigma_x^2\|\truebeta\|_2\sqrt{\frac{\log p}{\rho_*^2n}} + \sigma_x\sigma_\varepsilon\sqrt{\frac{\log p}{\rho_*n}}+\widetilde\lambda_n\right\}\right)\\
&\leq& O_\mP\left(\frac{\sqrt{s}}{\lambda_{\min}}\left\{\sigma_x^2\|\truebeta\|_2\sqrt{\frac{\log p}{\rho_*^2n}} + \sigma_x\sigma_\varepsilon\sqrt{\frac{\log p}{\rho_*n}}\right\}\right).
\end{eqnarray*}
The $\ell_1$ norm error bound $\|\widehat\beta_n-\truebeta\|_1$ can be easily obtained by the fact that $\|\widehat\beta_n-\truebeta\|_1\leq 2\sqrt{s}\|\widehat\beta_n-\truebeta\|_2$ as shown in Eq.~(\ref{eq:hatbeta_basic_ineq}).

Finally, consider $\widecheck\mu_n$ and define $\widecheck\lambda_n\widecheck\mu = \frac{1}{n}\widetilde X^\top y-\Sigma_0\widecheck\beta_n$.
Note that $\|\widecheck\delta\|_{\infty}\leq 1$ and 
$$
\Sigma_0(\widecheck\beta_n-\truebeta) = \left(\frac{1}{n}\widetilde X^\top X-\Sigma_0\right)\truebeta -\widecheck\lambda_n\widecheck\mu + \frac{1}{n}\widetilde X^\top\varepsilon.
$$
Note in addition that $(\widecheck\beta_n-\truebeta)^\top\Sigma_0(\widecheck\beta_n-\truebeta) \geq \lambda_{\min}\|\widecheck\beta_n-\truebeta\|_2^2$ by Assumption (A2).
Subsequently, the same line of argument for $\widehat\beta_n$ yields
\begin{eqnarray*}
\|\widecheck\beta_n-\truebeta\|_2 &\leq& \frac{2\sqrt{s}}{\lambda_{\min}}\cdot O_\mP\left\{\varphi_{u,v}\left(\frac{1}{n}\widetilde X^\top X,\Sigma_0;\log p\right)\|\truebeta\|_2  + \widetilde\lambda_n + \varphi_{\varepsilon,\infty}\left(\frac{1}{n}\widetilde X\right)\sigma_\varepsilon\right\}\\
&\leq& O_\mP\left\{\left(\sigma_x^2\|\truebeta\|_2+\sigma_x\sigma_\varepsilon\right)\sqrt{\frac{s\log p}{\lambda_{\min}^2\rho_* n}}\right\}.
\end{eqnarray*}

\subsection{Proof of Theorem \ref{thm:minimax}}

We consider the worst case with equal observation rates across covariates: $\rho_1=\cdots=\rho_p=\rho_*$
and use Fano's inequality (Lemma \ref{lem:fano}) to establish the minimax lower bound in Theorem \ref{thm:minimax}.
Without loss of generality we shall restrain ourselves to even $p$ and $s/2$ scenarios.
Construct hypothesis $\beta$ as
\begin{equation}
\beta = (\underbrace{a, \cdots, a}_{\text{repeat $s/2$ times}}, \underbrace{0, \pm\delta, 0,\cdots, \pm\delta, 0}_{\text{exactly $s/2$ copies of $\delta$}}),
\label{eq_construct_beta}
\end{equation}
where $\delta\to 0$ is some parameter to be chosen later and $a=\sqrt{\frac{2M^2}{s}-\delta^2}$ is carefully chosen so that $\|\beta\|_2=M$.
Clearly $\beta\in\mathbb B_2(M)\cap\mathbb B_0(s)$.
Let $d_H(\beta,\beta')=\sum_{j=1}^p{I[\beta_j\neq\beta_j']}$ be the Hamming distance between $\beta$ and $\beta'$.
The following lemma shows that it is possible to construct a large hypothesis classes where any two models in the hypothesis class are far away under the Hamming distance:
\begin{lem}[\cite{raskutti2011minimax}, Lemma 4]
Define $\mathcal H=\{z\in\{-1,0,+1\}^p: \|z\|_0=s\}$.
For $p,s$ even and $s<2p/3$, there exists a subset $\widetilde{\mathcal H}\subseteq\mathcal H$
with cardinality $|\widetilde{\mathcal H}|\geq \exp\{\frac{s}{2}\log\frac{p-s}{s/2}\}$ such that $\rho_H(z,z')\geq s/2$
for all dinstinct $s,s'\in\widetilde{\mathcal H}$.
\label{lem:cube}
\end{lem}

This does not affect the minimax lower bound to be proved.
Using the above lemma and under the condition that $s\leq 4p/5$, one can construct $\Theta$ consisting of hypothesis of the form in Eq.~(\ref{eq_construct_beta})
such that $\log|\Theta| \asymp s\log(p/s)$ and $\|\beta-\beta'\|_2\geq \sqrt{s/4}\delta$ for all distinct $\beta,\beta'\in\Theta$.
It remains to evaluate the KL divergence between $P_{\beta}$ and $P_{\beta'}$.

Let $x_{\obs}$ and $x_{\mis}$ denote the observed and missing covariates of a particular data point and let $\beta_{\obs}$, $\beta_{\mis}$
be the corresponding partition of coordinates of $\beta$.
The likelihood of $x_{\obs}$ and $y$ can be obtained by integrating out $x_{\mis}$ (assuming there are $q$ coordinates that are observed):
\begin{eqnarray*}
p(y,x_{\obs};\beta) 
&=& \rho_*^{q}(1-\rho_*)^{p-q}\int{\mathcal N_p(x_{\obs},x_{\mis};0,I)\mathcal N(y-(x_{\obs}^\top\beta_{\obs}-x_{\mis})^\top\beta_{\mis}; 0,\sigma_\varepsilon^2)\ud x_{\mis}}\\
&=& p(x_{\obs})\cdot \frac{1}{\sqrt{2\pi(\sigma_\varepsilon^2+\|\beta_{\mis}\|_2^2)}}\exp\left\{-\frac{(y-x_{\obs}^\top\beta_{\obs})^2}{2(\sigma_\varepsilon^2+\|\beta_{\mis}\|_2^2)}\right\}.
\end{eqnarray*}
Here $\mathcal N$ and $\mathcal N_p$ denote the univariate and multivariate Normal distributions.
Note that $p(x_{\obs})$ does not depend on $\beta$. Subsequently,
\begin{eqnarray}
\kl(P_{\beta}\|P_{\beta'})&=&\mathbb E_{\beta,\rho_*}\log\frac{p(y,x_{\obs};\beta')}{p(y,x_{\obs};\beta)}\nonumber\\
&=& \mathbb E_{\beta,\rho_*}\left\{\frac{1}{2}\log\frac{\sigma_\varepsilon^2+\|\beta'_{\mis}\|_2^2}{\sigma_\varepsilon^2+\|\beta_{\mis}\|_2^2} + \frac{1}{2}\left[\frac{(y-x_{\obs}^\top\beta'_{\obs})^2}{\sigma_\varepsilon^2+\|\beta'_{\mis}\|_2^2} - \frac{(y-x_{\obs}^\top\beta_{\obs})^2}{\sigma_\varepsilon^2+\|\beta_{\mis}\|_2^2}\right]\right\}\nonumber\\
&=& \mathbb E_{\rho_*}\left\{\frac{1}{2}\log\frac{\sigma_\varepsilon^2+\|\beta'_{\mis}\|_2^2}{\sigma_\varepsilon^2+\|\beta_{\mis}\|_2^2} + \frac{1}{2}\left[\frac{\sigma_\varepsilon^2+\|\beta_{\mis}\|_2^2 + \|\beta_{\obs}-\beta'_{\obs}\|_2^2}{\sigma_\varepsilon^2+\|\beta'_{\mis}\|_2^2} - 1\right]\right\}\nonumber\\
&\overset{(a)}{\leq}& \mathbb E_{\rho_*}\left\{\frac{1}{2}\left[\frac{\sigma_\varepsilon^2+\|\beta'_{\mis}\|_2^2}{\sigma_\varepsilon^2+\|\beta_{\mis}\|_2^2} + \frac{\sigma_\varepsilon^2+\|\beta_{\mis}\|_2^2}{\sigma_\varepsilon^2+\|\beta'_{\mis}\|_2^2}\right] - 1 + \frac{1}{2}\frac{\|\beta_{\obs}-\beta'_{\obs}\|_2^2}{\sigma_\varepsilon^2+\|\beta'_{\mis}\|_2^2}\right\}\nonumber\\
&=& \mathbb E_{\rho_*}\left\{\frac{1}{2}\frac{\left(\|\beta'_{\mis}\|_2^2-\|\beta_{\mis}\|_2^2\right)^2}{(\sigma_\varepsilon^2+\|\beta_{\mis}\|_2^2)(\sigma_\varepsilon^2+\|\beta'_{\mis}\|_2^2)} + \frac{1}{2}\frac{\|\beta_{\obs}-\beta'_{\obs}\|_2^2}{\sigma_\varepsilon^2+\|\beta'_{\mis}\|_2^2}\right\}.
\label{eq:kl_identity}
\end{eqnarray}
Here for $(a)$ we apply the inequality that $\log(1+x)\leq x$ for all $x>0$.
For some constant $c\in(0,1/2)$, define $\mathcal E(c)$ as the event that at least $\frac{1-\rho_*}{1+2c}$ portion of the first $s/2$ coordinates in $x$ are missing.
By Chernoff bound,
\footnote{If $X_1,\cdots,X_n$ are i.i.d.~random variables taking values in $\{0,1\}$ then $\Pr[\frac{1}{n}\sum_{i=1}^n{X_i}<(1-\delta)\mu]\leq \exp\{-\frac{\delta^2\mu}{2}\}$ for $0<\delta<1$,
where $\mu=\mathbb EX$.}
 $\Pr[\mathcal E(c)] \geq 1 - e^{-c^2(1-\rho_*)s}$.
Note that under $\mathcal E(c)$, $\|\beta_{\mis}\|_2^2, \|\beta'_{\mis}\|_2^2 \geq \frac{(1-\rho_*)s}{2(1+2c)}a^2$ almost surely.
Subsequently,
\begin{eqnarray*}
\kl(P_{\beta}\|P_{\beta'}) &\leq& \frac{1}{2}\frac{\mathbb E_{\rho_*|\mathcal E(c)}\left(\|\beta'_{\mis}\|_2^2-\|\beta_{\mis}\|_2^2\right)^2}{\left(\sigma_\varepsilon^2+\frac{(1-\rho_*)s}{2(1+2c)}a^2\right)^2} + \frac{1}{2}\frac{\mathbb E_{\rho_*|\mathcal E(c)}\|\beta_{\obs}-\beta'_{\obs}\|_2^2}{\sigma_\varepsilon^2+\frac{(1-\rho_*)s}{2(1+2c)}a^2}\\
&+& e^{-c^2(1-\rho_*)s}\left[\frac{1}{2}\frac{\mathbb E_{\rho_*|\overline{\mathcal E}(c)}\left(\|\beta'_{\mis}\|_2^2-\|\beta_{\mis}\|_2^2\right)^2}{\sigma_\varepsilon^4} + \frac{1}{2}\frac{\mathbb E_{\rho_*|\overline{\mathcal E}(c)}\|\beta_{\obs}-\beta'_{\obs}\|_2^2}{\sigma_\varepsilon^2}\right].
\end{eqnarray*}
Because $\beta$ and $\beta'$ are identical in the first $s/2$ coordinates, both $\|\beta'_{\mis}\|_2^2-\|\beta_{\mis}\|_2^2$ and $\|\beta_{\obs}-\beta'_{\obs}\|_2^2$
are independent of $\mathcal E(c)$.
Therefore, 
\begin{eqnarray*}
\mathbb E_{\rho_*}\left(\|\beta'_{\mis}\|_2^2-\|\beta_{\mis}\|_2^2\right)^2 = \mathbb E_{\rho_*}\left(\|\beta_{\mis,>s/2}'\|_2^2-\|\beta_{\mis,>s/2}\|_2^2\right)^2 &\leq& 4(1-\rho_*)^2s^2\delta^4;\\
\mathbb E_{\rho_*}\|\beta'_{\obs}-\beta_{\obs}\|_2^2 = \mathbb E_{\rho_*}\|\beta'_{\obs,>s/2}-\beta_{\obs,>s/2}\|_2^2 &\leq& 2\rho_*s\delta^2.
\end{eqnarray*}
Here $\beta_{\cdot,>s/2}$ denote the $\beta_{\cdot}$ vector without its first $s/2$ coordinates,
and in both inequalities we note by construction that $\|\beta_{>s/2}\|_0,\|\beta'_{>s/2}\|_0\leq s/2$.
Because $a^2 = \frac{2M^2}{s}-\delta^2$, we have that $\frac{(1-\rho_*)s}{2(1+2c)}a^2 = \frac{1-\rho_*}{1+2c}M^2 - \frac{(1-\rho_*)}{2(1+2c)}s\delta^2$.
For now assume that $\frac{1-\rho_*}{1+2c}s\delta^2 \ll \sigma_\varepsilon^2+\frac{1-\rho_*}{1+2c}M^2$,
which then implies $\sigma_\varepsilon^2+\frac{(1-\rho_*)s}{2(1+2c)}a^2 \geq \frac{1}{2}\left(\sigma_\varepsilon^2+\frac{1-\rho_*}{1+2c}M^2\right)$.
We will justify this assumption at the end of this proof.
Combining all inequalities we have
$$
\kl(P_{\beta}\|P_{\beta'}) \leq \frac{8(1-\rho_*)^2s^2\delta^4}{\left(\sigma_\varepsilon^2+\frac{1-\rho_*}{1+2c}M^2\right)^2} + \frac{2\rho_*s\delta^2}{\sigma_\varepsilon^2+\frac{1-\rho_*}{1+2c}M^2} + e^{-c^2(1-\rho_*)s}\left[\frac{2(1-\rho_*)^2s^2\delta^4}{\sigma_\varepsilon^4} + \frac{\rho_*s\delta^2}{\sigma_\varepsilon^2}\right].
$$

Let $P_{\beta}^n$ and $P_{\beta'}^n$ be the distribution of $n$ i.i.d.~samples parameterized by $\beta$ and $\beta'$, respectively.
Because the samples are i.i.d., we have that $\kl(P_{\beta}^n\|P_{\beta'}^n) = n\kl(P_{\beta}\|P_{\beta'})$.
On the other hand, because $\log|\Theta|\asymp s\log(p/s)$, to ensure $1-\frac{\kl(P_{\beta}^n\|P_{\beta'}^n)+\log 1/2}{\log|\Theta|}\geq\Omega(1)$ we only need to show
$\kl(P_{\beta}^n\|P_{\beta'}^n)\asymp s\log (p/s)$, which is implied by
\begin{eqnarray*}
\frac{(1-\rho_*)^2s^2\delta^4}{\left(\sigma_\varepsilon^2+\frac{1-\rho_*}{1+2c}M^2\right)^2} \asymp \frac{s\log(p/s)}{n}
&\Longleftarrow& \delta^2\asymp \left(\sigma_\varepsilon^2+\frac{1-\rho_*}{1+2c}M^2\right)\sqrt{\frac{\log(p/s)}{(1-\rho_*)^2sn}};\\
\frac{\rho_*s\delta^2}{\sigma_\varepsilon^2+\frac{1-\rho_*}{1+2c}M^2} \asymp \frac{s\log(p/s)}{n} &\Longleftarrow&
\delta^2 \asymp \left(\sigma_\varepsilon^2+\frac{1-\rho_*}{1+2c}M^2\right)  \frac{\log(p/s)}{\rho_* n};\\
e^{-c^2(1-\rho_*)s}\frac{(1-\rho_*)^2s^2\delta^4}{\sigma_\varepsilon^4}\asymp \frac{s\log(p/s)}{n}
&\Longleftarrow& \delta^2 \asymp e^{0.5c^2(1-\rho_*)s}\sigma_\varepsilon^2\sqrt{\frac{\log(p/s)}{(1-\rho_*)^2sn}};\\
e^{-c^2(1-\rho_*)s}\frac{\rho_*s\delta^2}{\sigma_\varepsilon^2}
&\Longleftarrow& \delta^2\asymp e^{c^2(1-\rho_*)s}\sigma_\varepsilon^2\frac{\log(p/s)}{\rho_* n}.
\end{eqnarray*}
Combining all terms we have that
\begin{equation}
\delta^2 \asymp \min\left\{\sigma_\varepsilon^2+\frac{1-\rho_*}{1+2c}M^2, e^{0.5c^2(1-\rho_*)s}\sigma_\varepsilon^2\right\}\cdot\min\left\{\sqrt{\frac{\log(p/s)}{(1-\rho_*)^2sn}}, \frac{\log(p/s)}{\rho_* n}\right\}.
\label{eq:delta}
\end{equation}
The bound for $\|\beta-\beta'\|_2^2$ can then be obtained by $\|\beta-\beta'\|_2^2 \geq \frac{s}{4}\delta^2$.

The final part of the proof is to justify the assumption that $\frac{1-\rho_*}{1+2c}s\delta^2 \ll \sigma_\varepsilon^2+\frac{1-\rho_*}{1+2c}M^2$.
Invoking Eq.~(\ref{eq:delta}), the assumption is valid if $\frac{1-\rho_*}{1+2c}\max\left\{\sqrt{\frac{s\log(p/s)}{(1-\rho_*)^2n}}, \frac{s\log(p/s)}{\rho_* n}\right\}\to 0$,
which holds if $\frac{s\log(p/s)}{\rho_* n}\to 0$.


\subsection{Proof of Theorem \ref{thm:minimax_rho2}}\label{subsec:proof_minimax_rho2}

We again take $\rho_1=\cdots=\rho_p=\rho_*$.
The first term $\frac{\sigma_\varepsilon^2}{\rho_* n}$ in the minimax lower bound is trivial to establish:
consider $\truebeta=\delta e_j$ and $\beta_1=-\delta e_j$ with $\Sigma_0=\Sigma_1=I$.
By Eq.~(\ref{eq:kl_identity}), we have that
$$
\kl(P_{\truebeta}^n\|P_{\beta_1}^n) = n\cdot \kl(P_{\truebeta}\|P_{\beta_1})
\leq \frac{2\rho_* n\delta^2}{\sigma_\varepsilon^2}.
$$
Equating $\kl(P_{\truebeta}^n\|P_{\beta_1}^n)\asymp O(1)$ we have that $\delta^2\asymp\frac{\sigma_\varepsilon^2}{\rho_* n}$.
Because $\frac{\sigma_\varepsilon^2}{M^2\rho_* n}\to 0$, we know that $\truebeta,\beta_1\in\mathbb B_2(M)\cap\mathbb B_0(1)$
when $n$ is sufficiently large.
Invoking Le Cam's method (Lemma \ref{lem:lecam}) with $|\beta_{0j}-\beta_{1j}|^2=4\delta^2\asymp \frac{\sigma_\varepsilon^2}{\rho_* n}$
we prove the desired minimax lower bound of $\frac{\sigma_\varepsilon^2}{\rho_* n}$.

We next focus on the second term in the minimax lower bound that involves $1/\rho_*^2 n$.
Without loss of generality assume $j>s-1$.
Construct two hypothesis $(\truebeta,\Sigma_0)$ and $(\beta_1,\Sigma_1)$ as follows:
\begin{eqnarray*}
&&\truebeta = (\underbrace{\frac{\widetilde a}{\sqrt{s-2}},\cdots,\frac{\widetilde a}{\sqrt{s-2}}}_{\text{repeat $s-2$ times}},\widetilde a, \underbrace{0, \cdots,0,\widetilde a\gamma,0,\cdots,0}_{\beta_{0j}=\widetilde a\gamma}),
\;\;\;\;\;\Sigma_0 = I_{p\times p} - \gamma(e_{s-1}e_j^\top + e_je_{s-1}^\top);\\
&&\beta_1 = (\underbrace{\frac{\widetilde a}{\sqrt{s-2}},\cdots,\frac{\widetilde a}{\sqrt{s-2}}}_{\text{repeat $s-2$ times}},\widetilde a, \underbrace{0, \cdots,0,-\widetilde a\gamma,0,\cdots,0}_{\beta_{0j}=-\widetilde a\gamma}),
\;\;\;\;\;\Sigma_1 = I_{p\times p} + \gamma(e_{s-1}e_j^\top + e_je_{s-1}^\top).
\end{eqnarray*}
Here $\gamma\to 0$ is some parameter to be determined later and $\widetilde a$ is set to $\widetilde a=\sqrt{\frac{M^2}{2+\gamma^2}}$
to ensure that $\|\truebeta\|_2=\|\beta_1\|_2=M$.
It is immediate by definition that $\truebeta,\beta_1\in\mathbb B_2(M)\cap\mathbb B_0(s)$.
In addition, by Gershgorin circle theorem all eigenvalues of $\Sigma_0$ and $\Sigma_1$ lie in $[1-\gamma,1+\gamma]$.
As $\gamma\to 0$, it holds that $\Sigma_0,\Sigma_1\in\Lambda(\gamma_0)$ for any constant $\gamma_0\in(0,1/2)$ when $n$ is sufficiently large.
A finite-sample statement of this fact is given at the end of the proof.

Unlike the identity covariance case, the likelihood $p(y,x_{\obs};\beta,\Sigma)$ for incomplete observations are complicated when $\Sigma$ has non-zero
off-diagonal elements.
The following lemma gives a general characterization of the likelihood when $\beta\neq 0$.
Its proof is given in the supplementary material.
\begin{lem}
Partition the covariance $\Sigma$ as $\Sigma=\left[\begin{array}{cc} \Sigma_{11}& \Sigma_{12}\\ \Sigma_{21}& \Sigma_{22}\end{array}\right]$,
where $\Sigma_{11}$ corresponds to $x_{\obs}$ and $\Sigma_{22}$ corresponds to $x_{\mis}$.
Define $\Sigma_{22:1}=\Sigma_{22}-\Sigma_{21}\Sigma_{11}^{-1}\Sigma_{12}$.
Let $q=\dim(\Sigma_{11})$ be the number of observed covariates. Then
\begin{eqnarray*}
p(y,x_{\obs};\beta,\Sigma) 
&=& \rho_*^{q}(1-\rho_*)^{p-q}\cdot \frac{1}{\sqrt{(2\pi)^q|\Sigma_{11}|}}\exp\left\{-\frac{1}{2}x_{\obs}^\top\Sigma_{11}^{-1}x_{\obs}\right\}\\
&&\cdot \frac{1}{\sqrt{2\pi(\sigma_\varepsilon^2+\beta_{\mis}^\top\Sigma_{22:1}\beta_{\mis})}}\exp\left\{-\frac{(y-x_{\obs}^\top\beta_{\obs}-\beta_{\mis}^\top\Sigma_{21}\Sigma_{11}^{-1}x_{\obs})^2}{2(\sigma_\varepsilon^2+\beta_{\mis}^\top\Sigma_{22:1}\beta_{\mis})}\right\}.
\end{eqnarray*}
\label{lem:pdf_conditional}
\end{lem}

We now present the following lemma, which is key to establish the $1/\rho_*^2$ rate in the minimax lower bound.
Its proof is given in the supplementary material.
\begin{lem}
$p(y,x_{\obs};\truebeta,\Sigma_0)=p(y,x_{\obs};\beta_1,\Sigma_1)$ unless \emph{both} $x_{s-1}$ and $x_j$ are observed.
\label{lem:rho2_equivalent}
\end{lem}

Let $P_0$ and $P_1$ denote the distributions parameterized by $(\truebeta,\Sigma_0)$ and $(\beta_1,\Sigma_1)$, respectively.
Let $\mathcal A$ denote the event that both $x_{s-1}$ and $x_j$ are observed.
By Lemma \ref{lem:rho2_equivalent}, we have that
$$
\kl(P_0\|P_1) = \Pr[\mathcal A]\mathbb E_0\left[\log\frac{p(y,x_{\obs};\truebeta,\Sigma_0)}{p(y,x_{\obs};\beta_1,\Sigma_1)}\bigg|\mathcal A\right]
= \rho_*^2\mathbb E_0\left[\log\frac{p(y,x_{\obs};\truebeta,\Sigma_0)}{p(y,x_{\obs};\beta_1,\Sigma_1)}\bigg|\mathcal A\right].
$$
Suppose $\Sigma_0=[\Sigma_{011}\; \Sigma_{012};\Sigma_{021}\; \Sigma_{022}]$ and
$\Sigma_1=[\Sigma_{111}\; \Sigma_{112};\Sigma_{121}\; \Sigma_{122}]$
are partitioned in the same way as in Lemma \ref{lem:pdf_conditional}.
Conditioned on the event $\mathcal A$, we have that 
\begin{align*}
\Sigma_{022}&=\Sigma_{122}=I_{(p-q)\times(p-q)},\\
\Sigma_{012}&=\Sigma_{021}^\top=\Sigma_{112}=\Sigma_{121}^\top=0_{q\times(p-q)},\\
\Sigma_{011}&=I_{q\times q}-\gamma(e_{s-1}e_j^\top+e_je_{s-1}^\top), \\
\Sigma_{111}&=I_{q\times q}+\gamma(e_{s-1}e_j^\top+e_je_{s-1}^\top),
\end{align*}
and by Lemma \ref{lem:rank2-update}, we have that
$$\Sigma_{011}^{-1}=I+\frac{\gamma^2}{1-\gamma^2}(e_{s-1}e_{s-1}^\top+e_je_j^\top)+\frac{\gamma}{1-\gamma^2}(e_{s-1}e_j^\top+e_je_{s-1}^\top)$$
and
$$\Sigma_{111}^{-1}=I+\frac{\gamma^2}{1-\gamma^2}(e_{s-1}e_{s-1}^\top+e_je_j^\top)-\frac{\gamma}{1-\gamma^2}(e_{s-1}e_j^\top+e_je_{s-1}^\top).$$
In addition,
$\det(\Sigma_{011}) =\det(\Sigma_{111})=1-\gamma^2$. 
Note also that $\Sigma_{022:1}=\Sigma_{122:1}=I_{(p-q)\times(p-q)}$ and hence
$\beta_{0\mis}^\top\Sigma_{022:1}\beta_{0\mis}=\beta_{1\mis}^\top\Sigma_{122:1}\beta_{1\mis}$
because $\|\beta_{0\mis}\|_2^2=\|\beta_{1\mis}\|_2^2$ regardless of which covariates are missing.
Define $x_{\obs,<s}=\{x_j: \text{$x_j$ is observed}, j<s\}$ and $\beta_{\obs,<s}=\{\beta_j:\text{$x_j$ is observed},j<s\}$.
Subsequently, invoking Lemma \ref{lem:pdf_conditional} we get
\begin{align*}
\mathbb E_{0|\mathcal A}\left[\log\frac{P_0}{P_1}\right]
&= -\frac{2\gamma}{1-\gamma^2}\mathbb E_0[x_{s-1}x_j] - \mathbb E_{0|\mathcal A}\left\{\frac{1}{2}\frac{(y-x_{\obs}^\top\beta_{0\obs})^2-(y-x_{\obs}^\top\beta_{1\obs})^2}{\sigma_\varepsilon^2+\|\beta_{0\mis}\|_2^2}\right\}\\
&\overset{(a)}{=} \frac{2\gamma^2}{1-\gamma^2} + \mathbb E_{0|\mathcal A}\left\{\frac{x_j(\beta_{0j}-\beta_{1j})(y-x_{\obs,<s}^\top\beta_{0\obs,<s})}{\sigma_\varepsilon^2+\|\beta_{0\mis}\|_2^2}\right\}\\
&\overset{(b)}{=} \frac{2\gamma^2}{1-\gamma^2} + \mathbb E_{0|\mathcal A}\left\{\frac{x_j(\beta_{0j}-\beta_{1j})(x_{\mis,<s}^\top\beta_{0\mis,<s}+x_j\beta_{0j}+\varepsilon)}{\sigma_\varepsilon^2+\|\beta_{0\mis}\|_2^2}\right\}\\
&\overset{(c)}{=} \frac{2\gamma^2}{1-\gamma^2} + \mathbb E_{R|\mathcal A}\left\{\frac{\beta_{0j}(\beta_{0j}-\beta_{1j})\mathbb E_0[x_j^2]+(\beta_{0j}-\beta_{1j})\mathbb E_{0|R}[x_j(x_{\mis,<s-1}^\top\beta_{0\mis,<s-1}+\varepsilon)]}{\sigma_{\varepsilon}^2+\|\beta_{0\mis}\|_2^2}\right\}\\
&= \frac{2\gamma^2}{1-\gamma^2} + \mathbb E_{R|\mathcal A}\left\{\frac{\beta_{0j}(\beta_{0j}-\beta_{1j})\mathbb E_0[x_j^2]}{\sigma_{\varepsilon}^2+\|\beta_{0\mis}\|_2^2}\right\}\\
&= \frac{2\gamma^2}{1-\gamma^2} + \mathbb E_{R|\mathcal A}\left\{\frac{2\widetilde a^2\gamma^2}{\sigma_\varepsilon^2+\|\beta_{0\mis}\|_2^2}\right\}.
\end{align*}
Here $(a)$ is due to $\beta_{0\obs,<s}=\beta_{1\obs,<s}$ and $\beta_{0j}^2=\beta_{1j}^2$,
and $(b)$ is because $\beta_{0k}=0$ for all $k\geq s$ except for $k=j$.
Note also that under $\mathcal A$, $x_j$ is observed and hence $\beta_{0j}$ always belongs to $\beta_{0\obs}$.
For $(c)$, note that $x_{s-1}$ is observed under $\mathcal A$ and $x_j$ is independent of $x_{<s-1}$ and $\varepsilon$ conditioned on $R$,
thanks to the missing completely at random assumption (A3).
For any constant $c\in(0,1/2)$ define $\mathcal E'(c)$ as the event that at least $\frac{1-\rho_*}{1+2c}$ portion of the first $(s-2)$ coordinates in $x$ are missing.
Note that $\|\beta_{0\mis}\|_2^2\geq \frac{1-\rho_*}{1+2c}\widetilde a^2$ almost surely under $\mathcal A\cap\mathcal E'(C)$ and by Chernoff bound
$\Pr[\mathcal A]\geq 1-e^{-c^2(1-\rho_*)(s-2)} \geq 1-e^{-0.5c^2(1-\rho_*)s}$ for $s\geq 4$.
Subsequently, by law of total expectation
$$
\mathbb E_{R|\mathcal A}\left\{\frac{2\widetilde a^2\gamma^2}{\sigma_\varepsilon^2+\|\beta_{0\mis}\|_2^2}\right\}
\leq \frac{2\widetilde a^2\gamma^2}{\sigma_\varepsilon^2+\frac{1-\rho_*}{1+2c}\widetilde a^2} + e^{-0.5c^2(1-\rho_*)s}\frac{2\widetilde a^2\gamma^2}{\sigma_\varepsilon^2}.
$$
Replace $\widetilde a^2=\frac{M^2}{2+\gamma^2}$.
We then have that
\begin{eqnarray*}
\kl(P_0^n\|P_1^n) 
&\leq& n\rho_*^2\left[\frac{2\gamma^2}{1-\gamma^2} + \frac{2M^2\gamma^2}{(2+\gamma^2)\sigma_\varepsilon^2+\frac{1-\rho_*}{1+2c}M^2} + e^{-0.5c^2(1-\rho_*)s}\frac{2M^2\gamma^2}{(2+\gamma^2)\sigma_\varepsilon^2}\right]\\
&\leq& n\rho_*^2\left[\frac{2\gamma^2}{1-\gamma^2} + \frac{2(1+2c)\gamma^2}{1-\rho_*} + e^{-0.5c^2(1-\rho_*)s}\frac{M^2\gamma^2}{\sigma_\varepsilon^2}\right].
\end{eqnarray*}
Equating $\kl(P_0^n\|P_1^n)\asymp O(1)$ and applying the condition that $\gamma^2 \to 0$, we have that
\begin{equation}
\gamma^2 \asymp \min\left\{\frac{1-\rho_*}{2(1+2c)}, e^{0.5c^2(1-\rho_*)s}\frac{\sigma_\varepsilon^2}{M^2}\right\}\frac{1}{\rho_*^2 n}.
\label{eq:gamma}
\end{equation}
Subsequently,
$$
\big|\beta_{0j}-\beta_{1j}\big|^2 = 4\widetilde a^2\gamma^2 \asymp \min\left\{\frac{1-\rho_*}{2(1+2c)}M^2, e^{0.5c^2(1-\rho_*)s}\sigma_\varepsilon^2\right\}\frac{1}{\rho_*^2 n}.
$$
Invoking Lemma \ref{lem:lecam} we finish the proof of the minimax lower bound.

Finally, we justify the conditions $\gamma^2\to 0$ and $\gamma<\gamma_0$ that are used in the proof.
Eq.~(\ref{eq:gamma}) yields $\gamma^2\leq O(\frac{1}{\rho_*^2 n})$.
So $\gamma^2\to 0$ and $\gamma<\gamma_0$ is implied by $\frac{1}{\gamma_0^2\rho_*^2 n}\to 0$.

\subsection{Proof of Theorem \ref{thm:asymptotic_variance}}

Using $y=X\truebeta+\varepsilon$ we have that
\begin{equation}
\widetilde\Sigma(\widehat\beta_n-\truebeta) + \left(\frac{1}{n}\widetilde X^\top y-\widetilde\Sigma\widehat\beta_n\right) = \bigg(\underbrace{\frac{1}{n}\widetilde X^\top X-\widetilde\Sigma}_{\Delta_n}\bigg)\truebeta + \frac{1}{n}\widetilde X^\top\varepsilon.
\label{eq:av_decomposition}
\end{equation}
Define $\Delta_n=\frac{1}{n}\widetilde X^\top X-\widetilde\Sigma$.
Recall that $\widehat\beta_n^u=\widehat\beta_n+\widehat\Theta\left(\frac{1}{n}\widetilde X^\top y-\widetilde\Sigma\widehat\beta_n\right)$.
Subsequently, multiplying both sides of Eq.~(\ref{eq:av_decomposition}) with $\sqrt{n}\widehat\Theta$ and re-organizing terms we have
\begin{multline*}
\sqrt{n}(\widehat\beta_n^u-\truebeta) = \sqrt{n}\widehat\Theta\left(\Delta_n\truebeta + \frac{1}{n}\widetilde X^\top\varepsilon\right) - \sqrt{n}(\widehat\Theta\widetilde\Sigma-I)(\widehat\beta_n-\truebeta)\\
= \sqrt{n}\Sigma_0^{-1}\left(\Delta_n\truebeta + \frac{1}{n}\widetilde X^\top\varepsilon\right)
-\underbrace{\sqrt{n}(\widehat\Theta\widetilde\Sigma-I)(\widehat\beta_n-\truebeta)}_{r_n}
+ \underbrace{\sqrt{n}(\widehat\Theta-\Sigma_0^{-1})\left(\Delta_n\truebeta + \frac{1}{n}\widetilde X^\top\varepsilon\right)}_{\widetilde r_n}.  
\end{multline*}
Define $r_n=\sqrt{n}(\widehat\Theta\widetilde\Sigma-I)(\widehat\beta_n-\truebeta)$ and
$\widetilde r_n=\sqrt{n}(\widehat\Theta-\Sigma_0^{-1})\left(\Delta_n\truebeta + \frac{1}{n}\widetilde X^\top\varepsilon\right)$
\begin{lem}
Suppose $\frac{\log p}{\rho_*^4 n}\to 0$ and the conclusion in Lemma \ref{lem:clime} holds.
Then 
$\|r_n\|_{\infty} \leq O_\mP(\sqrt{n}\widetilde\nu_n\|\widehat\beta_n-\truebeta\|_1)$
and
$\|\widetilde r_n\|_{\infty} \leq O_\mP(\sigma_xb_0b_1\widetilde\nu_n(\sigma_\varepsilon\sqrt{\frac{\log p}{\rho_*}}+\sigma_x\|\truebeta\|_2\sqrt{\frac{\log p}{\rho_*^2}}))$.
\label{lem:rn}
\end{lem}
Lemma \ref{lem:rn} based on H\"{o}lder's inequality and is proved in the supplementary materials.
If the condition in Eq.~(\ref{eq:asymptotic_variance_condition}) holds,
Lemma \ref{lem:rn} implies that $\max\{\|r_n\|_{\infty},\|\widetilde r_n\|_{\infty}\}\overset{p}{\to} 0$,
which means both terms $r_n$ and $\widetilde r_n$ are asymptotically negligible in the infinity norm sense.
It then suffices to analyze the limiting distribution (conditioned on $X$) of $a_n=\sqrt{n}\Sigma_0^{-1}\left(\Delta_n\truebeta + \frac{1}{n}\widetilde X^\top\varepsilon\right)$.
By Assumptions (A1) and (A3), $\mathbb E\Delta_n|X=0$, $\mathbb E\varepsilon|\widetilde X=0$ and hence $\mathbb Ea_n|X = 0$.
We next analyze the conditional covariance $\mathbb Va_n|X$.
Recall that $\Delta_n=\frac{1}{n}\widetilde X^\top X-\widetilde\Sigma$.
By definition, for any $j,k\in\{1,\cdots,p\}$
$$
[\Delta_n]_{jk} = \left\{\begin{array}{ll}
\frac{1}{n}\sum_{i=1}^n{\frac{R_{ij}}{\rho_j}\left(1-\frac{R_{ik}}{\rho_k}\right)X_{ij}X_{ik}},& j\neq k;\\
0,& j=k.\end{array}\right.
$$
Here $R_{ij}=1$ if $X_{ij}$ is observed and $R_{ij}=0$ otherwise.
Subsequently, $a_n=\Sigma_0^{-1}\widetilde a_n$ where
$$
[\widetilde a_n]_j = \frac{1}{\sqrt{n}}\sum_{i=1}^n{\Bigg(\underbrace{\frac{R_{ij}X_{ij}}{\rho_j}\varepsilon_i + \sum_{k\neq j}{\frac{R_{ij}}{\rho_j}\left(1-\frac{R_{ik}}{\rho_k}\right)X_{ij}X_{ik}\beta_{0k}}}_{T_{ij}}\Bigg)}.
$$
Because $R\indep X,\varepsilon$ and $\varepsilon\indep X$, we have that $\mathbb ET_{ij}|X=0$.
Therefore, for any $j\in\{1,\cdots,p\}$
$$
\mathbb VT_{ij}|X = \mathbb E\left[|T_{ij}|^2|X\right] = \frac{\sigma_{\varepsilon}^2X_{ij}^2}{\rho_j} + \sum_{t\neq j}{\frac{1-\rho_t}{\rho_j\rho_t}X_{ij}^2X_{it}^2\beta_{0t}^2}
$$
and for $j\neq k$, 
$$
\mathrm{cov}(T_{ij},T_{ik}|X) = \mathbb E\left[T_{ij}T_{ik}|X\right] = \sigma_{\varepsilon}^2X_{ij}X_{ik} + \sum_{t\neq j,k}{\frac{1-\rho_t}{\rho_t}X_{ij}X_{ik}X_{it}^2\beta_{0t}^2}.
$$
Because $\{T_{ij}\}_{i=1}^n$ are i.i.d.~random variables, by central limiting theorem, for any subset $S\subseteq[p]$ with constant size
$$
[a_n]_{SS} \overset{d}{\to} \mathcal N_{|S|}\left(0, \cov_{SS}(a_n|X)\right) \overset{d}{\to} \mathcal N_{|S|}\left(0, \left[\Sigma_0^{-1}\widehat\Gamma\Sigma_0^{-1}\right]_{SS}\right),
$$
where all randomness is conditioned on $X$.

\subsection{Proof of Theorem \ref{thm:estimate_variance}}

By triangle inequality and H\"{o}lder's inequality,
\begin{align*}
&\|\Sigma_0^{-1}\widehat\Gamma\Sigma_0^{-1}-\widehat\Theta\widetilde\Gamma\widehat\Theta^\top\|_{\infty}\\
&\leq \|(\Sigma_0^{-1}-\widehat\Theta)\widehat\Gamma\Sigma_0^{-1}\|_{\infty} + \|\widehat\Theta\widehat\Gamma(\Sigma_0^{-1}-\widehat\Theta^\top)\|_{\infty} + \|\widehat\Theta(\widehat\Gamma-\widetilde\Gamma)\widehat\Theta^\top\|_{\infty}\\
&\leq 2\max\left\{\|\Sigma_0^{-1}\|_{L_1}, \|\widehat\Theta\|_{L_1}, \|\widehat\Theta\|_{L_{\infty}}\right\}\max\left\{\|\Sigma_0^{-1}-\widehat\Theta\|_{L_1},\|\Sigma_0^{-1}-\widehat\Theta\|_{L_{\infty}}\right\}\|\widehat\Gamma\|_{\infty}+ \|\widehat\Theta\|_{L_1}^2\|\widehat\Gamma-\widetilde\Gamma\|_{\infty}.
\end{align*}
With Lemma \ref{lem:clime}, the bound can be simplified to (with probability $1-o(1)$)
\begin{equation}
\|\Sigma_0^{-1}\widehat\Gamma\Sigma_0^{-1}-\widehat\Theta\widetilde\Gamma\widehat\Theta^\top\|_{\infty} 
\leq
4b_0b_1^2\widetilde\nu_n\|\widehat\Gamma\|_{\infty} + b_1^2\|\widehat\Gamma-\widetilde\Gamma\|_{\infty}.
\label{eq:step1}
\end{equation}
Note that by standard concentration inequalities of supreme of sub-Gaussian random variables, 
$\|X\|_{\infty} \leq O_\mP(\sigma_x\sqrt{\log p})$.
Also, by H\"{o}lder's inequality $\|\widehat\Upsilon\|_{\infty} \leq \rho_*^{-2}\|X\|_{\infty}^4\|\truebetasq\|_1$.
Subsequently,
\begin{equation}
\|\widehat\Gamma\|_{\infty} \leq \frac{\sigma_{\varepsilon}^2}{\rho_*}\|X\|_{\infty}^2+ \frac{\|X\|_{\infty}^4\|\truebeta\|_2^2}{\rho_*^2}
\leq O_\mP\left\{\sigma_x^4\log^2 p\left(\frac{\sigma_\varepsilon^2}{\sigma_x^2\rho_*}+\frac{\|\truebeta\|_2^2}{\rho_*^2}\right)\right\}.
\label{eq:step2}
\end{equation}

It remains to upper bound $\|\widehat\Gamma-\widetilde\Gamma\|_{\infty}$.
Decompose the difference as
$$
\|\widehat\Gamma-\widetilde\Gamma\|_{\infty}
\leq \sigma_\varepsilon^2\left\|\frac{1}{n}\widetilde X^\top\widetilde X - \frac{1}{n}X^\top X - \widetilde D\diag\left(\frac{1}{n}X^\top X\right)\right\|_{\infty}
+ \|\widehat\Upsilon-\widetilde\Upsilon\|_{\infty}.
$$
We first focus on the first term.
Recall that $D=\diag(1-\rho_1,\cdots,1-\rho_p)$, $\widetilde D=(\frac{1}{\rho_1}-1,\cdots,\frac{1}{\rho_p}-1)$ and therefore $\|\widetilde D\|_{\infty}\leq 1- 1/\rho_*$ and $\frac{1}{n}\widetilde X^\top\widetilde X = \widetilde\Sigma+D\diag(\frac{1}{n}\widetilde X^\top\widetilde X)$.
Subsequently, the first infinity norm term is upper bounded by
$$
\|\widetilde\Sigma-\Sigma_0\|_{\infty} + \left\|D\diag\left(\frac{1}{n}\widetilde X^\top\widetilde X\right)-\widetilde D\diag(\Sigma_0)\right\|_{\infty}
+ \frac{1}{\rho_*}\|\widehat\Sigma-\Sigma_0\|_{\infty}.
$$
By Lemma \ref{lem:main_concentration}, if $\frac{\log p}{\rho_*^4 n}\to 0$ then
$\|\widetilde\Sigma-\Sigma_0\|_{\infty}\leq O_\mP(\sigma_x^2\sqrt{\frac{\log p}{\rho_*^2 n}})$
and $\|\widehat\Sigma-\Sigma_0\|_{\infty} \leq O_\mP(\sigma_x^2\sqrt{\frac{\log p}{n}})$.
For the remaining term, we invoke the following lemma that is proved in the supplementary materials:
\begin{lem}
If $\frac{\log p}{\rho_* n}\to 0$ then
$\|D\diag(\frac{1}{n}\widetilde X^\top\widetilde X)-\widetilde D\diag(\Sigma_0)\|_{\infty} \leq O_\mP(\sigma_x^2\sqrt{\frac{\log p}{\rho_*^3 n}})$.
\label{lem:doubletilde_concentration}
\end{lem}
Consequently,
\begin{equation}
\sigma_\varepsilon^2\left\|\frac{1}{n}\widetilde X^\top\widetilde X - \frac{1}{n}X^\top X - \widetilde D\diag\left(\frac{1}{n}X^\top X\right)\right\|_{\infty}
\leq O_\mP\left\{\sigma_\varepsilon^2\sigma_x^2\sqrt{\frac{\log p}{\rho_*^3 n}}\right\}.
\label{eq:step3}
\end{equation}

Finally, we derive the upper bound for $\|\widehat\Upsilon-\widetilde\Upsilon\|_{\infty}$.
We first construct a $p\times p$ matrix $\overline\Upsilon$ as an ``intermediate'' quantity defined as
$$
\overline\Upsilon_{jk} = \frac{1}{n}\sum_{i=1}^n{\sum_{t\neq j,k}{(1-\rho_t)\widetilde X_{ij}\widetilde X_{ik}\widetilde X_{it}^2\beta_{0t}^2}} \;\;\;\;\text{for}\;\;\;j,k\in\{1,\cdots,p\}.
$$
Note that $\overline\Upsilon$ involves the missing design $\widetilde X$ and the true model $\truebeta$.
Further define $\widetilde\Upsilon_{jkt}$ and $\Upsilon_{jkt}$ for $j,k,t\in\{1,\cdots,p\}$ as
$$
\widetilde\Upsilon_{jkt} = \frac{1}{n}\sum_{i=1}^n{(1-\rho_t)\widetilde X_{ij}\widetilde X_{ik}\widetilde X_{it}^2}, \;\;\;\;\;
\Upsilon_{jkt} = \mathbb E\widetilde\Upsilon_{jkt}|X.
$$
We next state the following concentration results on $\widetilde\Upsilon_{jkt}$ and $\Upsilon_{jkt}$, which will be proved in the supplementary material.
\begin{lem}
Fix $j,k\in[p]$ and suppose $\frac{\log p}{\rho_*^3 n}\to 0$. We then have that 
$$
\max_{j,k\in[p]}\max_{t\neq j,k}\big|\Upsilon_{jkt}\big| \leq O_\mP\left(\frac{\sigma_x^4\log^2 p}{\rho_*^2}\right)
$$
and
$$
\max_{j,k\in[p]}\max_{t\neq j,k}\big|\widetilde\Upsilon_{jkt}-\Upsilon_{jkt}\big| \leq O_\mP\left(\sigma_x^4\log^2 p\sqrt{\frac{\log p}{\rho_*^5 n}}\right).
$$
\label{lem:upsilonjkt_concentration}
\end{lem}
We then upper bound $\|\widetilde\Upsilon-\widehat\Upsilon\|_{\infty}$ by bounding $\|\widehat\Upsilon-\overline\Upsilon\|_{\infty}$ and $\|\widetilde\Upsilon-\overline\Upsilon\|_{\infty}$ separately.

\paragraph{Upper bound for $\|\widetilde\Upsilon-\overline\Upsilon\|_{\infty}$}
By definition, $\widetilde\Upsilon_{jk}=\sum_{t\neq j,k}{\widetilde\Upsilon_{jkt}\widehat\beta_{nt}^2}$ and $\overline\Upsilon_{jkt}=\sum_{t\neq j,k}{\widetilde\Upsilon_{jkt}\beta_{0t}^2}$.
H\"{o}lder's inequality then yields
$$
\|\widetilde\Upsilon-\overline\Upsilon\|_{\infty} \leq \max_{j,k\in [p]}\max_{t\neq j,k}\big|\widetilde\Upsilon_{jkt}\big|\cdot \|\widehat\beta_n^2-\truebetasq\|_1.
$$
Under the condition that $\frac{\log p}{\rho_*^3 n}\to 0$, it holds that $\max_{j,k}\max_{t\neq j,k}|\widetilde\Upsilon_{jkt}| \leq O_\mP(1)\cdot \max_{j,k}\max_{t\neq j,k}|\Upsilon_{jkt}|$.
Furthermore, 
$\|\widehat\beta_n^2-\truebetasq\|_1 \leq \|\widehat\beta_n+\truebeta\|_{\infty}\|\widehat\beta_n-\truebeta\|_1 \leq (\|\truebeta\|_2+\|\widehat\beta_n-\truebeta\|_2)\|\widehat\beta_n-\truebeta\|_1$.
Invoking Lemma \ref{lem:upsilonjkt_concentration} and the condition that $\|\widehat\beta_n-\truebeta\|_2\overset{p}{\to} 0$ we get
\begin{equation}
\|\widetilde\Upsilon-\overline\Upsilon\|_{\infty} \leq O_\mP\left\{\frac{\sigma_x^4\log^2 p}{\rho_*^2}\|\truebeta\|_2\|\widehat\beta_n-\truebeta\|_1\right\}.
\label{eq:step4}
\end{equation}

\paragraph{Upper bound for $\|\widehat\Upsilon-\overline\Upsilon\|_{\infty}$}
Note that $\widehat\Upsilon_{jk}=\sum_{t\neq j,k}{\Upsilon_{jkt}\beta_{0t}^2}$ and $\overline\Upsilon_{jk}=\sum_{t\neq j,k}{\widetilde\Upsilon_{jkt}\beta_{0t}^2}$.
By H\"{o}lder's inequality, 
$$
\|\overline\Upsilon-\widehat\Upsilon\|_{\infty} \leq \max_{j,k\in[p]}\max_{t\neq j,k}\big|\widetilde\Upsilon_{jkt}-\Upsilon_{jkt}\big|\cdot \|\truebetasq\|_1.
$$
Invoking Lemma \ref{lem:upsilonjkt_concentration} we then have 
\begin{equation}
\|\overline\Upsilon-\widehat\Upsilon\|_{\infty} \leq O_\mP\left\{\sigma_x^4\log^2 p\|\truebeta\|_2^2\sqrt{\frac{\log p}{\rho_*^5 n}}\right\}.
\label{eq:step5}
\end{equation}

Finally, combining Eqs.~(\ref{eq:step1},\ref{eq:step2},\ref{eq:step3},\ref{eq:step4},\ref{eq:step5}) we complete the proof of Theorem \ref{thm:estimate_variance}.



\bibliographystyle{apa-good}
\bibliography{refs}

\begin{thebibliography}{41}
\expandafter\ifx\csname natexlab\endcsname\relax\def\natexlab#1{#1}\fi
\expandafter\ifx\csname url\endcsname\relax
  \def\url#1{{\tt #1}}\fi
\expandafter\ifx\csname urlprefix\endcsname\relax\def\urlprefix{URL }\fi

\bibitem[{Bach(2008)}]{bach2008consistency}
Bach, F.~R. (2008).
\newblock Consistency of the group lasso and multiple kernel learning.
\newblock {\em Journal of Machine Learning Research\/}, {\em 9\/}(Jun),
  1179--1225.

\bibitem[{Balakrishnan et~al.(2017)Balakrishnan, Wainwright, \&
  Yu}]{balakrishnan2017}
Balakrishnan, S., Wainwright, M.~J., \& Yu, B. (2017).
\newblock Statistical guarantees for the em algorithm: From population to
  sample-based analysis.
\newblock {\em The Annals of Statistics\/}, {\em 45\/}(1), 77--120.

\bibitem[{Belloni et~al.(2017)Belloni, Chernozhukov, \&
  Kaul}]{belloni2017confidence}
Belloni, A., Chernozhukov, V., \& Kaul, A. (2017).
\newblock Confidence bands for coefficients in high dimensional linear models
  with error-in-variables.
\newblock {\em arXiv preprint arXiv:1703.00469\/}.

\bibitem[{Belloni et~al.(2016{\natexlab{a}})Belloni, Rosenbaum, \&
  Tsybakov}]{belloni2016ell}
Belloni, A., Rosenbaum, M., \& Tsybakov, A.~B. (2016{\natexlab{a}}).
\newblock An $(\ell_1,\ell_2,\ell_\infty)$-regularization approach to
  high-dimensional errors-in-variables models.
\newblock {\em Electronic Journal of Statistics\/}, {\em 10\/}(2), 1729--1750.

\bibitem[{Belloni et~al.(2016{\natexlab{b}})Belloni, Rosenbaum, \&
  Tsybakov}]{belloni2016linear}
Belloni, A., Rosenbaum, M., \& Tsybakov, A.~B. (2016{\natexlab{b}}).
\newblock Linear and conic programming estimators in high dimensional
  errors-in-variables models.
\newblock {\em Journal of the Royal Statistical Society: Series B (Statistical
  Methodology)\/}.

\bibitem[{Bickel et~al.(2009)Bickel, Ritov, \&
  Tsybakov}]{bickel2009simultaneous}
Bickel, P.~J., Ritov, Y., \& Tsybakov, A.~B. (2009).
\newblock Simultaneous analysis of lasso and dantzig selector.
\newblock {\em The Annals of Statistics\/}, (pp. 1705--1732).

\bibitem[{Cai et~al.(2011)Cai, Liu, \& Luo}]{cai2011constrained}
Cai, T., Liu, W., \& Luo, X. (2011).
\newblock A constrained {L1} minimization approach to sparse precision matrix
  estimation.
\newblock {\em Journal of the American Statistical Association\/}, {\em
  106\/}(494), 594--607.

\bibitem[{Cai et~al.(2014)Cai, Liang, \& Rakhlin}]{cai2014geometric}
Cai, T.~T., Liang, T., \& Rakhlin, A. (2014).
\newblock Geometric inference for general high-dimensional linear inverse
  problems.
\newblock {\em arXiv preprint arXiv:1404.4408\/}.

\bibitem[{Candes \& Tao(2007)}]{candes2007dantzig}
Candes, E., \& Tao, T. (2007).
\newblock The dantzig selector: Statistical estimation when p is much larger
  than n.
\newblock {\em The Annals of Statistics\/}, (pp. 2313--2351).

\bibitem[{Cand{\`e}s et~al.(2006)Cand{\`e}s, Romberg, \&
  Tao}]{candes2006robust}
Cand{\`e}s, E.~J., Romberg, J., \& Tao, T. (2006).
\newblock Robust uncertainty principles: Exact signal reconstruction from
  highly incomplete frequency information.
\newblock {\em IEEE Transactions on information theory\/}, {\em 52\/}(2),
  489--509.

\bibitem[{Carroll et~al.(1995)Carroll, Ruppert, \&
  Stefanski}]{carroll1995measurement}
Carroll, R., Ruppert, D., \& Stefanski, L. (1995).
\newblock {\em Measurement Error in Nonlinear Models\/}.
\newblock Chapman \& Hall/CRC Monographs on Statistics \& Applied Probability.
  Taylor \& Francis.

\bibitem[{Chapelle et~al.(2010)Chapelle, Scholkopf, \& Zien}]{chapelle10}
Chapelle, O., Scholkopf, B., \& Zien, A. (2010).
\newblock {\em Semi-Supervised Learning\/}.
\newblock The MIT Press.

\bibitem[{Chen \& Caramanis(2013)}]{chen2013noisy}
Chen, Y., \& Caramanis, C. (2013).
\newblock Noisy and missing data regression: Distribution-oblivious support
  recovery.
\newblock In {\em Proceedings of the International Conference on Machine
  Learning (ICML)\/}.

\bibitem[{Datta \& Zou(2015)}]{datta2015cocolasso}
Datta, A., \& Zou, H. (2015).
\newblock Cocolasso for high-dimensional error-in-variables regression.
\newblock {\em arXiv preprint arXiv:1510.07123\/}.

\bibitem[{Donoho(2006)}]{donoho2006compressed}
Donoho, D.~L. (2006).
\newblock Compressed sensing.
\newblock {\em IEEE Transactions on information theory\/}, {\em 52\/}(4),
  1289--1306.

\bibitem[{Efron et~al.(2004)Efron, Hastie, Johnstone, Tibshirani
  et~al.}]{efron2004least}
Efron, B., Hastie, T., Johnstone, I., Tibshirani, R., et~al. (2004).
\newblock Least angle regression.
\newblock {\em The Annals of statistics\/}, {\em 32\/}(2), 407--499.

\bibitem[{Fan \& Li(2001)}]{fan2001variable}
Fan, J., \& Li, R. (2001).
\newblock Variable selection via nonconcave penalized likelihood and its oracle
  properties.
\newblock {\em Journal of the American statistical Association\/}, {\em
  96\/}(456), 1348--1360.

\bibitem[{Hsu et~al.(2012)Hsu, Kakade, \& Zhang}]{hsu2012tail}
Hsu, D., Kakade, S.~M., \& Zhang, T. (2012).
\newblock A tail inequality for quadratic forms of subgaussian random vectors.
\newblock {\em Electronic Communications in Probability\/}, {\em 17\/}(52),
  1--6.

\bibitem[{Hwang(1986)}]{hwang86}
Hwang, J.~T. (1986).
\newblock Multiplicative errors-in-variables models with applications to recent
  data released by the u.s. department of energy.
\newblock {\em Journal of the American Statistical Association\/}, {\em
  81\/}(395), 680--688.

\bibitem[{Ibragimov \& Has'~minskii(2013)}]{ibragimov2013statistical}
Ibragimov, I.~A., \& Has'~minskii, R.~Z. (2013).
\newblock {\em Statistical estimation: asymptotic theory\/}, vol.~16.
\newblock Springer Science \& Business Media.

\bibitem[{Javanmard \& Montanari(2014)}]{javanmard2014confidence}
Javanmard, A., \& Montanari, A. (2014).
\newblock Confidence intervals and hypothesis testing for high-dimensional
  regression.
\newblock {\em Journal of Machine Learning Research\/}, {\em 15\/}(1),
  2869--2909.

\bibitem[{Lafferty \& Wasserman(2007)}]{lafferty07}
Lafferty, J.~D., \& Wasserman, L.~A. (2007).
\newblock Statistical analysis of semi-supervised regression.
\newblock In {\em Proceedings of the Advances in Neural Information Processing
  Systems (NIPS)\/}.

\bibitem[{Le~Cam(2012)}]{le2012asymptotic}
Le~Cam, L. (2012).
\newblock {\em Asymptotic methods in statistical decision theory\/}.
\newblock Springer Science \& Business Media.

\bibitem[{Little \& Rubin(1986)}]{little86}
Little, R. J.~A., \& Rubin, D.~B. (1986).
\newblock {\em Statistical Analysis with Missing Data\/}.
\newblock John Wiley \& Sons, Inc.

\bibitem[{Loh \& Wainwright(2012{\natexlab{a}})}]{loh2012high}
Loh, P.-L., \& Wainwright, M. (2012{\natexlab{a}}).
\newblock High-dimensional regression with noisy and missing data: provable
  guarantees with nonconvexity.
\newblock {\em The Annals of Statistics\/}, {\em 40\/}(3), 1637--1664.

\bibitem[{Loh \& Wainwright(2012{\natexlab{b}})}]{loh2012corrupted}
Loh, P.-L., \& Wainwright, M.~J. (2012{\natexlab{b}}).
\newblock Corrupted and missing predictors: Minimax bounds for high-dimensional
  linear regression.
\newblock In {\em Proceedings of the IEEE International Symposium on
  Information Theory (ISIT)\/}.

\bibitem[{Loh \& Wainwright(2015)}]{loh2015regularized}
Loh, P.-L., \& Wainwright, M.~J. (2015).
\newblock Regularized m-estimators with nonconvexity: Statistical and
  algorithmic theory for local optima.
\newblock {\em Journal of Machine Learning Research\/}, {\em 16\/}, 559--616.

\bibitem[{Meinshausen \& B{\"u}hlmann(2006)}]{meinshausen2006}
Meinshausen, N., \& B{\"u}hlmann, P. (2006).
\newblock High-dimensional graphs and variable selection with the lasso.
\newblock {\em The Annals of Statistics\/}, {\em 34\/}(3), 1436--1462.

\bibitem[{Miller(1981)}]{miller1981inverse}
Miller, K.~S. (1981).
\newblock On the inverse of the sum of matrices.
\newblock {\em Mathematics Magazine\/}, {\em 54\/}(2), 67--72.

\bibitem[{Nielsen et~al.(2002)Nielsen, West, Linn, Alter, Knowling, O'Connell,
  Zhu, Fero, Sherlock, Pollack et~al.}]{nielsen2002molecular}
Nielsen, T.~O., West, R.~B., Linn, S.~C., Alter, O., Knowling, M.~A.,
  O'Connell, J.~X., Zhu, S., Fero, M., Sherlock, G., Pollack, J.~R., et~al.
  (2002).
\newblock Molecular characterisation of soft tissue tumours: a gene expression
  study.
\newblock {\em The Lancet\/}, {\em 359\/}(9314), 1301--1307.

\bibitem[{Raskutti et~al.(2011)Raskutti, Wainwright, \&
  Yu}]{raskutti2011minimax}
Raskutti, G., Wainwright, M.~J., \& Yu, B. (2011).
\newblock Minimax rates of estimation for high-dimensional linear regression
  over-balls.
\newblock {\em IEEE Transactions on Information Theory\/}, {\em 57\/}(10),
  6976--6994.

\bibitem[{Rosenbaum \& Tsybakov(2010)}]{rosenbaum2010sparse}
Rosenbaum, M., \& Tsybakov, A. (2010).
\newblock Sparse recovery under matrix uncertainty.
\newblock {\em The Annals of Statistics\/}, {\em 38\/}(5), 2620--2651.

\bibitem[{Rosenbaum \& Tsybakov(2013)}]{rosenbaum2013improved}
Rosenbaum, M., \& Tsybakov, A. (2013).
\newblock Improved matrix uncertainty selector.

\bibitem[{St{\"a}dler et~al.(2014)St{\"a}dler, Stekhoven, \&
  B{\"u}hlmann}]{stadler14}
St{\"a}dler, N., Stekhoven, D.~J., \& B{\"u}hlmann, P. (2014).
\newblock Pattern alternating maximization algorithm for missing data in
  high-dimensional problems.
\newblock {\em Journal of Machine Learning Research\/}, {\em 15\/}, 1903--1928.

\bibitem[{Tibshirani(1996)}]{tibshirani1996regression}
Tibshirani, R. (1996).
\newblock Regression shrinkage and selection via the lasso.
\newblock {\em Journal of the Royal Statistical Society. Series B
  (Methodological)\/}, (pp. 267--288).

\bibitem[{van~de Geer et~al.(2014)van~de Geer, B\"{u}hlmann, Ritov, \&
  Dezeure}]{vandegeer2014asymptotically}
van~de Geer, S., B\"{u}hlmann, P., Ritov, Y., \& Dezeure, R. (2014).
\newblock On asymptotically optimal confidence regions and tests for
  high-dimensional models.
\newblock {\em The Annals of Statistics\/}, {\em 32\/}(3), 1166--1202.

\bibitem[{Wainwright(2009)}]{wainwright2009sharp}
Wainwright, M.~J. (2009).
\newblock Sharp thresholds for high-dimensional and noisy sparsity recovery
  using $\ell_1$-constrained quadratic programming (lasso).
\newblock {\em IEEE transactions on information theory\/}, {\em 55\/}(5),
  2183--2202.

\bibitem[{Wang et~al.(2015)Wang, Gu, Ning, \& Liu}]{han2015}
Wang, Z., Gu, Q., Ning, Y., \& Liu, H. (2015).
\newblock High dimensional {EM} algorithm: Statistical optimization and
  asymptotic normality.
\newblock In {\em Proceedings of Advances in Neural Information Processing
  Systems (NIPS)\/}.

\bibitem[{Yi \& Caramanis(2015)}]{xinyang15}
Yi, X., \& Caramanis, C. (2015).
\newblock Regularized {EM} algorithms: A unified framework and statistical
  guarantees.
\newblock In {\em Proceedings of Advances in Neural Information Processing
  Systems (NIPS)\/}.

\bibitem[{Zhang \& Zhang(2014)}]{zhang2014confidence}
Zhang, C.-H., \& Zhang, S. (2014).
\newblock Confidence intervals for low dimensional parameters in high
  dimensional linear models.
\newblock {\em Journal of the Royal Statistical Society, Series B (Statistical
  Methodology)\/}, {\em 76\/}, 217--242.

\bibitem[{Zhao \& Yu(2006)}]{zhao2006model}
Zhao, P., \& Yu, B. (2006).
\newblock On model selection consistency of lasso.
\newblock {\em Journal of Machine learning research\/}, {\em 7\/}(Nov),
  2541--2563.

\end{thebibliography}


\begin{thebibliography}{3}
\expandafter\ifx\csname natexlab\endcsname\relax\def\natexlab#1{#1}\fi
\expandafter\ifx\csname url\endcsname\relax
  \def\url#1{{\tt #1}}\fi
\expandafter\ifx\csname urlprefix\endcsname\relax\def\urlprefix{URL }\fi

\bibitem[{Cai et~al.(2011)Cai, Liu, \& Luo}]{cai2011constrainedsupp}
Cai, T., Liu, W., \& Luo, X. (2011).
\newblock A constrained {L1} minimization approach to sparse precision matrix
  estimation.
\newblock {\em Journal of the American Statistical Association\/}, {\em
  106\/}(494), 594--607.

\bibitem[{Loh \& Wainwright(2012)}]{loh2012high}
Loh, P.-L., \& Wainwright, M. (2012).
\newblock High-dimensional regression with noisy and missing data: provable
  guarantees with nonconvexity.
\newblock {\em The Annals of Statistics\/}, {\em 40\/}(3), 1637--1664.

\bibitem[{van~de Geer(2010)}]{vandegeer2010empirical}
van~de Geer, S. (2010).
\newblock {\em Empirical Processes in {M}-Estimation\/}.
\newblock Cambridge University Press.

\end{thebibliography}

\clearpage
\setcounter{equation}{0}
\setcounter{figure}{0}
\setcounter{table}{0}
\setcounter{page}{1}
\setcounter{section}{0}
\makeatletter
\renewcommand{\theequation}{S\arabic{equation}}
\renewcommand{\thefigure}{S\arabic{figure}}
\renewcommand{\bibnumfmt}[1]{[S#1]}
\renewcommand{\citenumfont}[1]{S#1}

\renewcommand\thesection{\Alph{section}}
\renewcommand\thesubsection{\thesection.\arabic{subsection}}
\renewcommand\thesubsubsection{\thesubsection.\arabic{subsubsection}}

\begin{center}
{\bf\Large Supplementary Material for: Rate Optimal Estimation and Confidence Intervals for High-dimensional Regression with Missing Covariates}

{Yining Wang, Jialei Wang, Sivaraman Balakrishnan and Aarti Singh}
\end{center}

This supplementary material provides detailed proofs for technical lemmas whose proofs are omitted in the main text.

\section{Technical Lemmas}

\begin{lem}[Generalized Fano's inequality, \citep{ibragimov2013statistical}]
Let $\Theta$ be a parameter set and $d:\Theta\times\Theta\to\mathbb R_{\geq 0}$ be a semimetric.
Let $P_{\theta}$ be the distribution induced by $\theta$ and $P_{\theta}^n$ be the distribution of $n$ i.i.d.~observations from $P_\theta$.
If $d(\theta,\theta')\geq \alpha$ and $\kl(P_{\theta}\|P_{\theta'})\leq\beta$ for all distinct $\theta,\theta'\in\Theta$, then
$$
\inf_{\widehat\theta}\sup_{\theta\in\Theta}\mathbb E_{P_{\theta^n}}\left[d(\widehat\theta,\theta)\right] \geq \frac{\alpha}{2}\left(1-\frac{n\beta+\log 2}{\log|\Theta|}\right).
$$
\label{lem:fano}
\end{lem}

\begin{lem}[Le Cam's method, \citep{le2012asymptotic}]
Suppose $P_{\theta_0}$ and $P_{\theta_1}$ are distributions induced by $\theta_0$ and $\theta_1$.
Let $P_{\theta_0}^n$ and $P_{\theta_1}^n$ be distributions of $n$ i.i.d.~observations from $P_{\theta_0}$ and $P_{\theta_1}$, respectively.
Then for any estimator $\widehat\theta$ it holds that 
$$
\frac{1}{2}\left[{P_{\theta_0}^n}(\widehat\theta\neq\theta_0) + {P_{\theta_1}^n}(\widehat\theta\neq\theta_1)\right]
\geq \frac{1}{2}-\frac{1}{2}\|P_{\theta_0}^n-P_{\theta_1}^n\|_{\mathrm{TV}}
\geq \frac{1}{2} - \frac{1}{2\sqrt{2}}\sqrt{n\kl(P_{\theta_0}\|P_{\theta_1})}.
$$
\label{lem:lecam}
\end{lem}

\begin{lem}[\citet{miller1981inverse}, Eq.~(13)]
Suppose $H$ is a matrix of rank at most 2 and $(I+H)$ is invertible.
Then 
$$
(I+H)^{-1} = I - \frac{aH-H^2}{a+b},
$$
where $a=1+\tr(H)$ and $2b=[\tr(H)]^2+\tr(H^2)$.
\label{lem:rank2-update}
\end{lem}

\section{Proofs of concentration bounds}

\subsection{Proof of Lemma \ref{lem:main_concentration}}

Fix arbitrary $u,v\in\mathcal S$.
For $j,k\in[p]$ and $\ell\in\{0,1,2\}$, define
$$
\xi_{jk}^{(0)}(R_i,\rho) = 1,\;\;\;\;
\xi_{jk}^{(1)}(R_i,\rho) = \frac{R_{ij}}{\rho_j}, \;\;\;\;
\xi_{jk}^{(2)}(R_i,\rho) = \left\{\begin{array}{ll}
\frac{R_{ij}}{\rho_j},& j=k;\\
\frac{R_{ij}R_{ik}}{\rho_j\rho_k},& j\neq k.\end{array}\right.
$$
Also let $T_i^{(\ell)} = \sum_{j,k=1}^p{\xi_{jk}^{(\ell)}(R_i,\rho)X_{ij}X_{ik}u_jv_k}$. We then have that
\begin{eqnarray}
\big|u^\top(\widehat\Sigma-\Sigma_0)v\big| &=& \bigg|\frac{1}{n}\sum_{i=1}^n{T_i^{(0)}-\mathbb ET_i^{(0)}}\bigg|, \label{eq:dev0}\\
\big|u^\top(\frac{1}{n}\widetilde X^\top X-\Sigma_0)v\big| &=& \bigg|\frac{1}{n}\sum_{i=1}^n{T_i^{(1)}-\mathbb ET_i^{(2)}}\bigg|, \label{eq:dev1}\\
\big|u^\top(\widetilde\Sigma-\Sigma_0)v\big| &=& \bigg|\frac{1}{n}\sum_{i=1}^n{T_i^{(2)}-\mathbb ET_i^{(2)}}\bigg|.\label{eq:dev2}
\end{eqnarray}

The main idea is to use Berstein inequality with moment conditions (Lemma \ref{lem:bernstein-moment}) to establish concentration bounds and achieve optimal dependency over $\rho$.
Define $V^{(\ell)} = \mathbb E\left[|T_i^{(\ell)}-\mathbb ET_i^{(\ell)}|^2\right]$.
We then have that
$$
V^{(\ell)} \leq \mathbb E|T_i^{(\ell)}|^2 =
\sum_{j,k,j',k'=1}^p{\mathbb E\left\{\xi_{jk}^{(\ell)}\xi_{j'k'}^{(\ell)}\right\}\mathbb E\left\{X_{ij}X_{ik}X_{ij'}X_{ik'}u_jv_ku_{j'}v_{k'}\right\}}.
$$
It is then of essential importance to evaluate $\mathbb E\left\{\xi_{jk}^{(\ell)}\xi_{j'k'}^{(\ell)}\right\}$.
For $\ell=0$ the expectation trivially equals 1. For $\ell=1$ and $\ell=2$, we apply the following proposition, which is easily proved by definition.
\begin{prop}
$\mathbb E\left\{\xi_{jk}^{(1)}\xi_{j'k'}^{(1)}\right\} = 1 + I[j=j'](\frac{1}{\rho_j}-1)$ and
$\mathbb E\left\{\xi_{jk}^{(2)}\xi_{j'k'}^{(2)}\right\} = 1 + I[j=j'](\frac{1}{\rho_j}-1) + I[k=k'](\frac{1}{\rho_k}-1)
+ I[j=j'\wedge k=k'](\frac{1}{\rho_j}-1)(\frac{1}{\rho_k}-1)
+ I[j=j'=k=k'](1-\frac{1}{\rho_j})\frac{1}{\rho_j}$.
Here $I[\cdot]$ is the indicator function.
\end{prop}
We are now ready to derive $\mathbb E|T_i^{(\ell)}|^2$.
\begin{eqnarray*}
\mathbb E|T_i^{(0)}|^2 &=& \mathbb E\left\{|X_i^\top u|^2|X_i^\top v|^2\right\};\\
\mathbb E|T_i^{(1)}|^2 &=& \mathbb E\left\{|X_i^\top u|^2|X_i^\top v|^2\right\} + \sum_{j=1}^p{\left(\frac{1}{\rho_j}-1\right)u_j^2\mathbb E\left\{X_{ij}^2|X_i^\top v|^2\right\}}\\
&\leq& \mathbb E\left\{|X_i^\top u|^2|X_i^\top v|^2\right\} + \frac{1}{\rho_*}\sum_{j=1}^p{u_j^2\mathbb E\left\{|X_i^\top e_j|^2|X_i^\top v|^2\right\}};\\
\mathbb E|T_i^{(2)}|^2 &=& \mathbb E\left\{|X_i^\top u|^2|X_i^\top v|^2\right\} + \sum_{j=1}^p{\left(\frac{1}{\rho_j}-1\right)(u_j^2+v_j^2)\mathbb E\left\{X_{ij}^2|X_i^\top v|^2\right\}}\\
&&+ \sum_{k=1}^p\left(\frac{1}{\rho_k}-1\right)v_j^2\mathbb E\left\{X_{ik}^2|X_i^\top u|^2\right\}\\
&&+ \sum_{j,k=1}^p{\left(\frac{1}{\rho_j}-1\right)\left(\frac{1}{\rho_k}-1\right)u_j^2v_k^2\mathbb E\left\{X_{ij}^2X_{ik}^2\right\}}\\
&&+ \sum_{j=1}^p\left(1-\frac{1}{\rho_j}\right)\frac{1}{\rho_j}u_j^2v_j^2\mathbb EX_{ij}^4\\
&\leq& \mathbb E\left\{|X_i^\top u|^2|X_i^\top v|^2\right\} + \frac{1}{\rho_*}\sum_{j=1}^p{u_j^2\mathbb E\left\{|X_i^\top e_j|^2|X_i^\top v|^2\right\}}
+ \frac{1}{\rho_*^2}\sum_{j,k=1}^p{u_j^2v_k^2\mathbb E\left\{|X_i^\top e_j|^2|X_i^\top e_k|^2\right\}}.
\end{eqnarray*}
By Cauchy-Schwartz inequality and moment upper bounds of sub-Gaussian random variables (Lemma \ref{lem:subgaussian}), we have that
$$
\mathbb E\left\{|X_i^\top a|^2|X_i^\top b|^2\right\} 
\leq \sqrt{\mathbb E|X_i^\top a|^4}\sqrt{\mathbb E|X_i^\top b|^4}
\leq 16\sigma_x^4\|a\|_2^2\|b\|_2^2.
$$
Consequently, there exists universal constant $c_2>0$ such that
$$
\mathbb E|T_i^{(0)}|^2 \leq c_2\sigma_x^4\|u\|_2^2\|v\|_2^2, \;\;\;\;\;
\mathbb E|T_i^{(1)}|^2 \leq \frac{c_2}{\rho_*}\sigma_x^4\|u\|_2^2\|v\|_2^2, \;\;\;\;\;
\mathbb E|T_i^{(2)}|^2 \leq \frac{c_2}{\rho_*^2}\sigma_x^4\|u\|_2^2\|v\|_2^2.
$$

We next find an $L>0$ so that the moment condition in Lemma \ref{lem:bernstein-moment} is satisfied, namely
$\mathbb E|T_i^{(\ell)}-\mathbb ET_i^{(\ell)}|^k \leq \frac{1}{2}V^{(\ell)}L^{k-2}k!$ for all $k>1$.
Note that for all $\ell\in\{0,1,2\}$, there exist functions $\xi_j^{(\ell)}$ and $\overline\xi_j^{(\ell)}$ \emph{only depending on $j$} such that
$\xi_{jk}^{(\ell)} = \xi_j^{(\ell)}\xi_k^{(\ell)} + I[j=k]\cdot\overline\xi_j^{(\ell)}$
and furthermore $\max_j|\xi_j^{(\ell)}|\leq 1/\rho_*$, $\overline\xi_j^{(0)}=\overline\xi_j^{(1)}=0$ and $\max_j|\overline\xi_j^{(2)}|\leq 1/\rho_*^2$.
Subsequently, 
\begin{eqnarray*}
\mathbb E|T_i^{(\ell)}-\mathbb ET_i^{(\ell)}|^k 
&=& \mathbb E\left|\sum_{j,k=1}^p{\left(\xi_j^{(\ell)}\xi_k^{(\ell)} + I[j=k]\cdot \overline\xi_j^{(\ell)} - 1\right)X_{ij}X_{ik}u_jv_k}\right|^k\\
&\leq& 3^k\left(\mathbb E\bigg|\sum_{j,k=1}^p{\xi_j^{(\ell)}\xi_k^{(\ell)}X_{ij}X_{ik}u_jv_k}\bigg|^k + \mathbb E\bigg|\sum_{j=1}^p{\overline\xi_j^{(\ell)}X_{ij}^2u_jv_j}\bigg|^k + \mathbb E\bigg|\sum_{j,k=1}^p{X_{ij}X_{ik}u_jv_k}\bigg|^k\right).
\end{eqnarray*}
Here the second line is a consequence of the following inequality: for all $a,b,c\geq 0$ we have that
$(a+b+c)^k \leq (3\max\{a,b,c\})^k \leq 3^k\max\{a^k,b^k,v^k\} \leq 3^k(a^k+b^k+c^k)$.
Define $\widetilde u_j=u_j\xi_j^{(\ell)}$, $\widetilde v_k=v_k\xi_k^{(\ell)}$, $\overline u_j=u_j\sqrt{|\overline\xi_j^{(\ell)}|}$
and $\overline v_j=v_j\sqrt{|\overline\xi_j^{(\ell)}|}$.
Apply Lemma \ref{lem:quadratic} with $|\sum_{j=1}^p{\overline\xi_j^{(\ell)}X_{ij}^2u_jv_j}|\leq X_i^\top AX_i$, 
$A=\diag(|\overline u_1\overline v_1|, \cdots, |\overline u_p\overline v_p|)$ and note that
$\tr(A) \leq |\overline u|^\top|\overline v| \leq \|\overline u\|_2\|\overline v\|_2$ and
$\|A\|_{\mathrm{op}} = \max_{1\leq j\leq p}|\overline u_j\overline v_j| \leq \|\overline u\|_2\|\overline v\|_2$.
Subsequently, for all $t>0$ 
\begin{equation}
\Pr\left[X_i^\top AX_i > 3\sigma_x^2\|\overline u\|_2\|\overline v\|_2(1+t)\right] \leq e^{-t}.
\label{eq:xax_tail}
\end{equation}
Let $F(x)=\Pr[X_i^\top AX_i \leq x], x\geq 0$ be the CDF of $X_i^\top AX_i$ and $G(x)=1-F(x)$.
Using integration by parts, we have that
$$
\mathbb E|X_i^\top AX_i|^k = \int_0^{\infty}{x^k\ud F(x)} = -\int_0^{\infty}{x^k\ud G(x)} = \int_0^{\infty}{kx^{k-1}G(x)\ud x}.
$$
Here in the last equality we use the fact that $\lim_{x\to\infty}{x^k G(x)} = 0$ for any fixed $k\in\mathbb N$, 
because $G(x)\leq\exp\{1-\frac{x}{M}\}$ by Eq.~(\ref{eq:xax_tail}), where $M=3\sigma_x^2\|\overline u\|_2\|\overline v\|_2$.
Consequently, 
\begin{align*}
\mathbb E|X_i^\top AX_i|^k
&= \int_0^M{kx^{k-1}G(x)\ud x} + k \int_M^{\infty}{x^{k-1}G(x)\ud x}\\
&\leq M^k + k\int_0^{\infty}{M^{k-1}(1+z)^{k-1}e^{-z}\cdot M\ud z}\\
&= M^k + kM^k\int_0^{\infty}{(1+z)^{k-1}e^{-z}\ud z} \\
&\leq M^k + kM^k\cdot k! \leq (k+1)!M^k.
\end{align*}
Here in the second line we apply change-of-variable $x=M(1+z)$ and the fact that $G(M(1+z))\leq e^{-z}$ in the integration term.
Because $2^k\geq k+1$ for all $k\geq 1$, we conclude that
$$
\mathbb E\bigg|\sum_{j=1}^p{\overline\xi_j^{(\ell)}X_{ij}^2u_jv_j}\bigg|^k \leq 6^k\sigma_x^{2k}k!\mathbb E\|\overline u\|_2^k\|\overline v\|_2^k, \;\;\;\;\forall k\geq 1.
$$



Subsequently, applying Cauchy-Schwartz inequality together with moment bounds for sub-Gaussian random variables (Lemma \ref{lem:subgaussian}) we obtain
\begin{align*}
&\mathbb E|T_i^{(\ell)}-\mathbb ET_i^{(\ell)}|^k\\
&\leq 3^k\left(\sqrt{\mathbb E|X_i^\top\widetilde u|^{2k}}\sqrt{\mathbb E|X_i^\top\widetilde v|^{2k}} + 6^k\sigma_x^{2k}k!\mathbb E\|\overline u\|_2^k\|\overline v\|_2^k+ \sqrt{\mathbb E|X_i^\top u|^{2k}}\sqrt{\mathbb E|X_i^\top v|^{2k}}\right)\\
&\leq 3^k\cdot 2k\cdot 6^{k}\Gamma(k)\sigma_x^{2k}\cdot \left(\sqrt{\mathbb E\|\widetilde u\|_2^{2k}}\sqrt{\mathbb E\|\widetilde v\|_2^{2k}} + \sqrt{\mathbb E\|\overline u\|_2^{2k}}\sqrt{\mathbb E\|\overline v\|_2^{2k}} + \|u\|_2^{k}\|v\|_2^{k}  \right) \\
&\leq \rho_*^{\ell/2}\left(\frac{C'\|u\|_2\|v\|_2\sigma_x^2}{\rho_*^{\ell}}\right)^k k!,
\end{align*}
where $C'<\infty$ is some absolute constant.
Compare the bound of $\mathbb E|T_i^{(\ell)}-\mathbb ET_i^{(\ell)}|^k$ with the variance $\mathbb E|T_i^{(\ell)}|^2$ we obtained earlier, 
we have that $L=\sigma_x^2\|u\|_2\|v\|_2\cdot {C'}^3/\rho_*^{1.5\ell}$ is sufficient to guarantee $\mathbb E|T_i^{(\ell)}-\mathbb ET_i^{(\ell)}|^k \leq \frac{1}{2}V^{(\ell)}L^{k-2}k!$
for all \footnote{The case of $k=2$ is trivially true.} $k>2$.
Applying Bernstein inequality with moment conditions (Lemma \ref{lem:bernstein-moment}) and union bound over all $u,v\in\mathcal S$, we have that
$$
\Pr\left[\forall u,v\in\mathcal S, \bigg|\frac{1}{n}\sum_{i=1}^n{T_i^{(\ell)}-\mathbb ET_i^{(\ell)}}\bigg| > \|u\|_2\|v\|_2\epsilon\right] \leq
2N^2\exp\left\{-\frac{n\epsilon^2}{2(\widetilde V^{(\ell)}+\widetilde L\epsilon)}\right\}
$$
for all $\epsilon >0$, where $\widetilde V^{(\ell)}  = \frac{V^{(\ell)}}{\|u\|_2^2\|v\|_2^2}$ and $\widetilde L=\frac{L}{\|u\|_2\|v\|_2}$.
Subsequently,
\begin{align*}
\sup_{u,v\in\mathcal S}\bigg|\frac{1}{n}\sum_{i=1}^n{T_i^{(\ell)}-\mathbb ET_i^{(\ell)}}\bigg| 
&\leq O_\mP\left(\|u\|_2\|v\|_2\max\left\{\frac{\widetilde L\log N}{n}, \sqrt{\frac{\widetilde V^{(\ell)}\log N}{n}}\right\}\right)\\
&\leq O_\mP\left(\sigma_x^2\|u\|_2\|v\|_2\max\left\{\frac{\log N}{\rho_*^{1.5\ell}n}, \sqrt{\frac{\log N}{\rho_*^{\ell} n}}\right\}\right),
\end{align*}
as desired.

\subsection{Proof of Lemma \ref{lem:epsilon_concentration}}
Define $\delta_j = \frac{1}{n}\sum_{i=1}^n{Z_{ij}}$ where $Z_{ij}=\widetilde X_{ij}\varepsilon_i$.
Because $\mathbb E\varepsilon_i|X = 0$, we have that $\mathbb EZ_{ij} = 0$.
In addition, 
$$
\mathbb E|Z_{ij}|^2 = \frac{\sigma_\varepsilon^2\sigma_x^2}{\rho_j} \leq \frac{\sigma_\varepsilon^2\sigma_x^2}{\rho_*} \;=:\; V
$$
and for $k>2$, 
\begin{eqnarray*}
\mathbb E|Z_{ij}|^k 
&=& \rho_j\cdot \frac{1}{\rho_j^k}\cdot \mathbb E\varepsilon_i^k\cdot \mathbb E|X_{ij}|^k\\
&\leq& \frac{1}{\rho_*^{k-1}}\cdot k^22^k\sigma_x^k\sigma_\varepsilon^k\Gamma\left(\frac{k}{2}\right)^2\\
&\leq& \frac{k^2 (2\sigma_x\sigma_\varepsilon)^k}{\rho_*^{k-1}} k!\\
&\leq& \rho_*\left(\frac{8\sigma_x\sigma_\varepsilon}{\rho_*}\right)^kk!.
\end{eqnarray*}
By setting $L=64\sigma_x\sigma_\varepsilon/\rho_*$ we have that $\mathbb E|Z_{ij}|^k \leq \frac{1}{2}VL^{k-2}k!$ for all $k>1$.
Subsequently, applying Bernstein inequality with moment conditions (Lemma \ref{lem:bernstein-moment}) and union bound over $j=1,\cdots,p$ we have that
$$
\Pr\left[\|\delta\|_\infty > \epsilon\right] \leq 2p\exp\left\{-\frac{n\epsilon^2}{2(V+L\epsilon)}\right\}
$$
for any $\epsilon > 0$.
Suppose $\frac{\epsilon L}{V}\to 0$. We then have that
$$
\|\delta\|_{\infty} \leq O_\mP\left(\sigma_\varepsilon\sigma_x\sqrt{\frac{\log p}{\rho_* n}}\right).
$$
The condition $\frac{\epsilon L}{V}\to 0$ is satisfied with
$
\frac{\log p}{\rho_* n}\;\to\; 0.
$

\subsection{Proof of Lemma \ref{lem:doubletilde_concentration}}

Fix arbitrary $j\in\{1,\cdots,p\}$ and consider
$$
T_{ij} = (1-\rho_j)\widetilde X_{ij}^2 = \frac{(1-\rho_j)R_{ij}X_{ij}^2}{\rho_j^2}.
$$
It is easy to verify that $[D\diag(\frac{1}{n}\widetilde X^\top\widetilde X)]_{jj} = \frac{1}{n}\sum_{i=1}^n{T_{ij}}$
and $[\widetilde D\diag(\Sigma_0)]_{jj} = \frac{1}{n}\sum_{i=1}^n{\mathbb ET_{ij}} = \frac{(1-\rho_j)\Sigma_{0jj}}{\rho_j}$.
We use moment based Bernstein's inequality (Lemma \ref{lem:bernstein-moment}) to bound the perturbation
$|\frac{1}{n}\sum_{i=1}^n{T_{ij}-\mathbb ET_{ij}}|$.
Define $V_j=\mathbb E|T_{ij}-\mathbb ET_{ij}|^2$.
We then have
$$
V_j \leq \mathbb E|T_{ij}|^2 = \frac{(1-\rho_j)^2\mathbb EX_{ij}^4}{\rho_j^3} \leq \frac{3\sigma_x^4}{\rho_*^3}
$$
and for all $k\geq 3$, 
$$
\mathbb E|T_{ij}-\mathbb ET_{ij}|^k \leq 2^k\left(\mathbb E|T_{ij}|^k + |\mathbb ET_{ij}|^k\right)
\leq \frac{4^{k+1}}{\rho_*^{2k-1}}\sigma_x^{2k} k!.
$$
It can then be verified that $\mathbb E|T_{ij}-\mathbb ET_{ij}|^k \leq \frac{1}{2}V_jL^{k-2}k!$ for all $k\geq 2$
if $L=\frac{512\sigma_x^2}{\rho_*^2}$.
By Lemma \ref{lem:bernstein-moment} and a union bound over all $j\in\{1,\cdots,p\}$, we have that
$$
\Pr\left[\forall j,\bigg|\frac{1}{n}\sum_{i=1}^n{T_{ij}-\mathbb ET_{ij}}\bigg| > \epsilon\right] \leq 2p\exp\left\{-\frac{n\epsilon^2}{2(V+L\epsilon)}\right\}
$$
for all $\epsilon>0$, where $V=\frac{3\sigma_x^4}{\rho_*^3}$ and $L=\frac{512\sigma_x^2}{\rho_*^2}$.
Under the assumption that $\frac{\epsilon L}{V}\to 0$, we have that
$$
\left\|\widetilde D\diag\left(\frac{1}{n}\widetilde X^\top\widetilde X\right)-D\diag(\Sigma_0)\right\|_{\infty}
= \sup_{1\leq j\leq p}\bigg|\frac{1}{n}\sum_{i=1}^n{T_{ij}-\mathbb ET_{ij}}\bigg|
\leq O_\mP\left(\sigma_x^2\sqrt{\frac{\log p}{\rho_*^3 n}}\right).
$$
The condition $\frac{\epsilon L}{V}\to 0$ is then satisfied with $\frac{\log p}{\rho_* n}\to 0$.

\subsection{Proof of Lemma \ref{lem:upsilonjkt_concentration}}

By definition and the missing data model, 
$$
\Upsilon_{jkt} = \mathbb E\widetilde\Upsilon_{jkt}|X = \left\{\begin{array}{ll}
\frac{1}{n}\sum_{i=1}^n{\frac{1-\rho_t}{\rho_j\rho_k}X_{ij}^2X_{it}^2},& j=k;\\
\frac{1}{n}\sum_{i=1}^n{\frac{1-\rho_t}{\rho_k}X_{ij}X_{ik}X_{it}^2},& j\neq k.\end{array}\right.
$$
Subsequently,
$$
\max_{j,k\in[p]}\max_{t\neq j,k}\big|\Upsilon_{jkt}\big| \leq \frac{\|X\|_{\infty}^4}{\rho_*^2} \leq O_\mP\left(\frac{\sigma_x^4\log^2 p}{\rho_*^2}\right).
$$

To prove the second part of this lemma, we first fix arbitrary $j,k\in[p]$ and $t\neq j,k$.
Define 
$$
T_{ijkt} = \left(\xi_{jkt}(R_i,\rho)-\mathbb E\xi_{jkt}(R_i,\rho)\right)X_{ij}X_{ik}X_{it}^2,
$$
where $\xi_{jkt}(R_i,\rho)=\frac{(1-\rho_t)R_{ij}R_{ik}R_{it}}{\rho_j\rho_k\rho_t^2}$.
It is easy to verify that $\widetilde\Upsilon_{jkt}-\Upsilon_{jkt}=\frac{1}{n}\sum_{i=1}^n{T_{ijkt}}$ and $\mathbb ET_{ijkt}|X = 0$.
We then use Bernstein inequality with support conditions (Lemma \ref{lem:bernstein-support}) 
to bound the concentration of $\frac{1}{n}\sum_{i=1}^n{T_{ijkt}}$ towards zero.
Define $A=\max_{i,j,k,t}|T_{ijkt}|$ and $V=\max_{i,j,k,t}\mathbb E|T_{ijkt}|^2$.
By H\"{o}lder's inequality we have that
$$
A \leq \frac{\|X\|_{\infty}^4}{\rho_*^4} \leq O_\mP\left(\frac{\sigma_x^4\log^2 p}{\rho_*^4}\right).
$$
Here in the $O_\mP(\cdot)$ notation the randomness is on the generating process of $X$ and is \emph{independent} of the randomness of missing patterns $R$.
In addition, note that
$$
\mathbb E\bigg|\left(\xi_{jkt}-\mathbb E\xi_{jkt}\right)\left(\xi_{jkt'}-\mathbb E\xi_{jkt'}\right)\bigg| \leq \frac{1}{\rho_*^5}
$$
for all $j,k,t,t'\in\{1,\cdots,p\}$ and $t,t'\neq j,k$.
Subsequently,
$$
V = \max_{i,j,k,t}\mathbb E|T_{ijkt}|^2 \leq \frac{1}{\rho_*^5}X_{ij}^2X_{ik}^2X_{it}^4 \leq \frac{\|X\|_{\infty}^8}{\rho_*^5}
 \leq O_\mP\left(\frac{\sigma_x^8\log^4 p}{\rho_*^5}\right).
$$
Applying Lemma \ref{lem:bernstein-support} conditioned on $\|X\|_{\infty}\leq O(\frac{\sigma_x^4\log^2 p}{\rho_*^2})$, 
we have that with probability $1-O(\delta)$ for some $\delta=o(1)$ the following holds:
$$
\bigg|\frac{1}{n}\sum_{i=1}^n{T_{ijkt}}\bigg| \leq O\left(\sigma_x^4\log^2 p\sqrt{\frac{\log(1/\delta)}{\rho_*^5 n}}\right) \;=:\; \epsilon,
$$
provided that $\frac{\epsilon A}{V}\to 0$.
Applying union bound over all $j,k\in[p]$ and $t\in[p]\backslash\{j,k\}$ we get
$$
\max_{j,k\in[p]}\max_{t\neq j,k}\bigg|\frac{1}{n}\sum_{i=1}^n{T_{ijkt}}\bigg| \leq O_\mP\left(\sigma_x^4\log^2 p\sqrt{\frac{\log p}{\rho_*^5 n}}\right),
$$
The condition $\frac{\epsilon A}{V}\to 0$ is satisfied with $\frac{\log p}{\rho_*^3 n}\to 0$.

\section{Proof of restricted eigenvalue conditions}

In this section we review the standard analysis that establishes restricted eigenvalue conditions for sample covariance
and adapt it to our missing data setting by invoking Lemma \ref{lem:main_concentration}.

\begin{lem}
Suppose $A,B$ are $p\times p$ random matrices with $\Pr[\|A-B\|_{\infty}\leq M] \geq 1-o(1)$ for some $M<\infty$.
If $A$ satisfies $\re(s,\phi_{\min})$ and $B$ satisfies $\re(s,\phi_{\min}')$, then with probability $1-o(1)$ we have that
$$
\phi_{\min}' \geq \phi_{\min} \;-\; \left\{O(1)\cdot \varphi_{u,v}(A,B;O(s\log(Mp))) + O(1/n)\right\}.
$$
\label{lem:re_perturbation}
\end{lem}
\begin{proof}
For any $h\in\mathbb R^p$ it holds that
$$
\frac{h^\top Bh}{h^\top h} \geq \frac{h^\top Ah}{h^\top h} - \frac{h^\top(B-A)h}{h^\top h}.
$$
With appropriate scalings, it suffices to bound
\begin{equation*}
\sup_{h: \|h_{J^c}\|_1\leq \|h_J\|_1, \|h\|_2\leq 1} \big|h^\top(B-A)h\big|
\end{equation*}
for all $J\subseteq[p]$, $|J|\leq s$
as the largest possible gap between $\phi_{\min}$ and $\phi_{\min}'$.

Define $\mathbb B_p(r) = \{x\in\mathbb R^p: \|x\|_p\leq r\}$ as the $p$-norm ball of radius $r$.
Because $\|h_{J^c}\|_1\leq \|h_J\|_1$ implies $\|h\|_1 \leq 2\|h_J\|_1 \leq 2\sqrt{s}\|h\|_2$, we have that
$$
\sup_{h: \|h_{J^c}\|_1\leq \|h_J\|_1, \|h\|_2\leq 1} \big|h^\top(B-A)h\big|
\leq
\sup_{h\in \mathbb B_2(1)\cap\mathbb B_1(2\sqrt{s})} \big|h^\top(B-A)h\big|.
$$
By Lemma 11 in the supplementary material of \citesupp{loh2012high}, we have that
\begin{eqnarray*}
\mathbb B_2(1)\cap\mathbb B_1(2\sqrt{s}) 
&\subseteq& 3\conv\left\{\mathbb B_0(4s)\cap\mathbb B_2(1)\right\}\\
&\subseteq& \conv\{\underbrace{\mathbb B_0(4s)\cap\mathbb B_2(3)}_{K(4s)}\}.
\end{eqnarray*}
Here $\conv(A)$ denotes the convex hull of set $A$.
Let $K(4s)=\mathbb B_0(4s)\cap\mathbb B_2(3)$ and denote $N_{\epsilon,\|\cdot\|_2}(K(4s))$
as the \emph{covering number} of $K(4s)$ with respect to the Euclidean norm $\|\cdot\|_2$.
That is, $N_{\epsilon,\|\cdot\|_2}(K(4s))$ is the size of the smallest \emph{covering set} $H\subseteq K(4s)$ such that
$\sup_{h\in K(4s)}\inf_{h'\in H}\|h-h'\|_2 \leq \epsilon$.
By definition of the concentration bounds, we have that with probability $1-o(1)$
$$
\sup_{h\in H} \big|h^\top(A-B) h\big| \leq
\varphi_{u,u}(A,B;\log|H|)\sup_{h\in H}\|h\|_2^2 \leq
9\varphi_{u,u}(A,B;\log N_{\epsilon,\|\cdot\|_2}(K(4s))).
$$
Subsequently, for any $\epsilon\in(0,1)$ with probability $1-o(1)$
\begin{eqnarray*}
\sup_{h\in \mathbb B_2(1)\cap\mathbb B_1(2\sqrt{s})} \big|h^\top(B-A)h\big|
&\leq& \sup_{h\in\conv\{ K(4s)\}} \big| h^\top(A-B) h\big|\\
&\leq& \sup_{\substack{\xi_1,\cdots,\xi_T\geq 0,\\ \xi_1+\cdots+\xi_T=1,\\ h_1,\cdots,h_T\in K(4s)}} \sum_{i,j=1}^T{\xi_i\xi_j\big|h_i^\top(A-B)h_j\big|}\\
&\leq& \sup_{h,h'\in K(4s)} \big|h^\top(A-B) h'\big|\\
&\leq& \sup_{h,h'\in H_{\epsilon,\|\cdot\|_2}[K(4s)]}\big| h^\top(A-B) h'\big| + (6\epsilon+3\epsilon^2)\|A-B\|_{L_2}\\
&\leq& 36\left\{\varphi_{u,u}(A,B;\log N_{\epsilon,\|\cdot\|_2}(K(4s))) + \epsilon pM\right\}.
\end{eqnarray*}
Here the last inequality is implied by the condition that $\|A-B\|_{\infty}\leq M$ with probability $1-O(n^{-\alpha})$.
Taking $\epsilon=O(1/(p^2 M))$ we have that $\epsilon p M = O(1/p) = O(1/n)$.

The final part of the proof is to establish upper bounds for the covering number $N_{\epsilon,\|\cdot\|_2}(K(4s))$.
First note that by definition
$$
K(4s) = \bigcup_{J\subseteq[p]: |J|\leq 4s}\left\{h: \supp(h)=J\wedge \|h\|_2\leq 3\right\}.
$$
The covering number of a union of subsets can be upper bounded by the following proposition:
\begin{prop}
Let $K=K_1\cup\cdots\cup K_m$.
Then $N_{\epsilon,\|\cdot\|_2}(K) \leq \sum_{i=1}^m{N_{\epsilon,\|\cdot\|_2}(K_i)}$.
\end{prop}
\begin{proof}
Let $H_i\subseteq K_i$ be covering sets of subset $K_i$.
Define $H = H_1\cup\cdots\cup H_m$.
Clearly $|H| \leq \sum_{i=1}^m{|H_i|} \leq \sum_{i=1}^m{N_{\epsilon,\|\cdot\|_2}(K_i)}$.
It remains to prove that $H$ is a valid $\epsilon$-covering set of $K$.
Take arbitrary $h\in K$.
By definition, there exists $i\in[m]$ such that $h\in K_i$.
Subsequently, there exists $h^*\in H_i \subseteq H$ such that $\|h-h^*\|_2\leq\epsilon$.
Therefore, $H$ is a valid $\epsilon$-covering set of $K$.
\end{proof}

Define $K_J(r) = \{h: \supp(h)=J\wedge \|h\|_2\leq r\}$.
The covering number of $K_J$ is established in the following proposition:
\begin{prop}
$N_{\epsilon,\|\cdot\|_2}(K_J(r)) \leq \left(\frac{4r+\epsilon}{\epsilon}\right)^{|J|}$.
\end{prop}
\begin{proof}
$K_J(r)$ is nothing but a centered $|J|$-dimensional ball of radius $r$, locating at the coordinates indexed by $J$.
The covering number result of high-dimensional ball is due to Lemma 2.5 of \citesupp{vandegeer2010empirical}.
\end{proof}

Combining the three propositions, we obtain
$$
\log N_{\epsilon,\|\cdot\|_2}(K(4s))
\leq \log\left(\sum_{j=0}^{4s}{\binom{p}{j}}\right) + \log \left\{\left(\frac{12+\epsilon/2}{\epsilon/2}\right)^{4s}\right\}
\leq O\left(s\log (p/\epsilon)\right).
$$
With the configuration of $\epsilon=O(1/(p^2 M))$, we have that
$$
\log N_{\epsilon,\|\cdot\|_2}(K(4s)) \leq O(s\log (pM)).
$$
\end{proof}

We are now ready to prove Lemma \ref{lem:re_main}.
\begin{proof}[Proof of Lemma \ref{lem:re_main}]
Consider $A=\widetilde\Sigma$ and $B=\Sigma_0$ in Lemma \ref{lem:re_perturbation}.
Lemma \ref{lem:main_concentration} yields
\begin{equation*}
\varphi_{u,v}(\widetilde\Sigma,\Sigma_0;O(s\log(Mp))) \leq
O\left(\sigma_x^2\max\left\{\frac{s\log(Mp)}{\rho_*^3 n}, \sqrt{\frac{s\log(Mp)}{\rho_*^2 n}}\right\}\right) 
=: \epsilon.
\end{equation*}
By Lemma \ref{lem:re_perturbation}, to prove this corollary it is sufficient to show that $\frac{\epsilon}{\lambda_{\min}(\Sigma_0)} \to 0$.
Note also that $M = \|\widetilde\Sigma-\Sigma_0\|_\infty \leq \frac{\|X\|_\infty}{\rho_*^2} \leq O\left(\frac{\sigma_x\sqrt{\log p}}{\rho_*^2}\right)$ with probability $1-o(1)$.
The condition $\frac{\epsilon}{\lambda_{\min}(\Sigma_0)}\to 0$ can then be satisfied with $\frac{\sigma_x^4 s\log(\sigma_x\log p/\rho_*)}{\rho_*^3\lambda_{\min}^2 n} \to 0$.
\end{proof}

\section{Proof of Lemma \ref{lem:clime}}

\begin{lem}
Suppose $\frac{\log p}{\rho_*^4 n}\to 0$ and $\widetilde\nu_n\asymp \sigma_x^2b_1\sqrt{\frac{\log p}{\rho_*^2 n}}$. 
Then with probability $1-o(1)$ the population precision matrix $\Sigma_0^{-1}$ is a feasible solution to Eq.~(\ref{eq:clime});
that is, $\max\{\|\widetilde\Sigma\Sigma_0^{-1}-I_{p\times p}\|_{\infty},\|\Sigma_0^{-1}\widetilde\Sigma-I_{p\times p}\|_{\infty}\}\leq\nu_n$.
\label{lem:clime_feasible}
\end{lem}
\begin{proof}
First by H\"{o}lder's inequality we have that
$$
\|\widetilde\Sigma\Sigma_0^{-1}-I\|_{\infty} = \|(\widetilde\Sigma-\Sigma_0)\Sigma_0^{-1}\|_{\infty} \leq \|\Sigma_0^{-1}\|_{L_1}\|\widetilde\Sigma-\Sigma_0\|_{\infty} 
\leq b_1\|\widetilde\Sigma-\Sigma_0\|_{\infty}.
$$
By Lemma \ref{lem:main_concentration}, with probability $1-o(1)$
$$
\|\widetilde\Sigma-\Sigma_0\|_{\infty} \leq \varphi_{u,v}\left(\widetilde\Sigma,\Sigma_0;2\log p\right) \leq O\left(\sigma_x^2\sqrt{\frac{\log p}{\rho_*^2n}}\right),
$$
provided that $\frac{\log p}{\rho_*^4 n}\to 0$.
Subsequently, we have that
\begin{equation}
\|\Sigma_0^{-1}\|_{L_1}\|\widetilde\Sigma-\Sigma_0\|_{\infty} \leq O\left(\sigma_x^2b_1\sqrt{\frac{\log p}{\rho_*^2 n}}\right)\leq\widetilde\nu_n
\label{eq:nu_application}
\end{equation}
with probability $1-o(1)$.
The $\|\Sigma_0^{-1}\widetilde\Sigma-I\|_{\infty}$ term can be bounded in the same way by noting that
$\|\Sigma_0^{-1}\widetilde\Sigma-I\|_{\infty} \leq \|\Sigma_0^{-1}\|_{L_{\infty}}\|\widetilde\Sigma-\Sigma_0\|_{\infty}
\leq b_1\|\widetilde\Sigma-\Sigma_0\|_{\infty}$.
\end{proof}

\begin{lem}
Suppose $\Sigma_0^{-1}$ is a feasible solution to the CLIME optimization problem in Eq.~(\ref{eq:clime}). 
Then 
$\max\{\|\widehat\Theta\|_{L_1},\|\widehat\Theta\|_{L_\infty}\}\leq\|\Sigma_0^{-1}\|_{L_1}$ and
$
\|\widehat\Theta-\Sigma_0^{-1}\|_{\infty} \leq 2\widetilde\nu_n\|\Sigma_0^{-1}\|_{L_1}.
$
\label{lem:clime_infinity}
\end{lem}
\begin{proof}
We first establish that $\|\widehat\Theta\|_{L_1}\leq\|\Sigma_0^{-1}\|_{L_1}$.
In \citesupp{cai2011constrainedsupp} it is proved that the solution set of Eq.~(\ref{eq:clime}) is identical to the solution set of
$$
\widehat\Theta=\left\{\widehat\omega_{ i}\right\}_{i=1}^p, \;\;\;\;\;
\widehat\omega_{ i} \;\in\; \argmin_{\omega_i\in\mathbb R^p}\left\{\|\omega_i\|_1: \|\widetilde\Sigma\omega_i-e_i\|_{\infty} \leq \widetilde\nu_n\right\}.
$$
Because $\Sigma_0^{-1}$ belongs to the feasible set of the above constrained optimization problem, 
we have that $\|\widehat\omega_{i}\|_1\leq \|\Sigma_0^{-1}\|_{L_1}$ for all $i=1,\cdots,p$ and hence $\|\widehat\Theta\|_{L_1}\leq\|\Sigma_0^{-1}\|_{L_1}$.
The inequality $\|\widehat\Theta\|_{L_\infty}\leq\|\Sigma_0^{-1}\|_{L_1}$ can be proved by applying the same argument to $\widehat\Theta^\top$.

We next prove the infinity norm bound fot the estimation error $\widehat\Theta-\Sigma_0^{-1}$.
By triangle inequality,
$$
\|\Sigma_0(\widehat\Theta-\Sigma_0^{-1})\|_{\infty} \leq \|\widetilde\Sigma\widehat\Theta-I\|_{\infty} + \|(\widetilde\Sigma-\Sigma_0)\widehat\Theta\|_{\infty} \leq \widetilde\nu_n + \|(\widetilde\Sigma-\Sigma_0)\widehat\Theta\|_{\infty}.
$$
Using H\"{o}lder's inequality, we have that
$$
\|(\widetilde\Sigma-\Sigma_0)\widehat\Theta\|_{\infty} \leq \|\widehat\Theta\|_{L_1}\|\widetilde\Sigma-\Sigma_0\|_{\infty} \leq \|\Sigma_0^{-1}\|_{L_1}\|\widetilde\Sigma-\Sigma_0\|_{\infty} \leq \widetilde\nu_n.
$$
Here the last inequality is due to Eq.~(\ref{eq:nu_application}).
Subsequently, $\|\Sigma_0(\widehat\Theta-\Sigma_0^{-1})\|_{\infty} \leq 2\widetilde\nu_n$.
Applying H\"{o}lder's inequality again we obtain
$$
\|\widehat\Theta-\Sigma_0^{-1}\|_{\infty} \leq \|\Sigma_0^{-1}\|_{L_1}\|(\widetilde\Sigma-\Sigma_0)\widehat\Theta\|_{\infty} \leq 2\widetilde\nu_n\|\Sigma_0^{-1}\|_{L_1}.
$$
\end{proof}

To translate the infinity-norm estimation error $\Sigma_0^{-1}$ into an $L_1$-norm bound that we desire, 
we need the following lemma that establishes basic inequality of the estimation error:
\begin{lem}
Suppose $\Sigma_0^{-1}$ is a feasible solution to Eq.~(\ref{eq:clime}). 
Then under Assumption (A5) we have that
$\max\{\|\widehat\Theta-\Sigma_0^{-1}\|_{L_1},\|\widehat\Theta-\Sigma_0^{-1}\|_{L_\infty}\} \leq 2b_0\|\widehat\Theta-\Sigma_0^{-1}\|_{\infty}$.
\label{lem:clime_basic_ineq}
\end{lem}
\begin{proof}
Let $\widehat\omega_i$ and $\widehat\omega_{0i}$ be the $i$th columns of $\widehat\Theta$ and $\Sigma_0^{-1}$, respectively.
Let $J_i$ denote the support size of $\widehat\omega_{0i}$. 
Definte $\widehat h = \widehat\omega_i-\omega_{0i}$.
We then have that
$$
\|\widehat\omega_i\|_1 = \|\omega_{0i}+\widehat h_{J_i^c}\|_1 + \|\widehat h_{J_i}\|_1 \geq \|\omega_{0i}\|_1 - \|\widehat h_{J_i^c}\|_1 + \|\widehat h_{J_i}\|_1.
$$
On the other hand, $\|\widehat\omega_i\|_1\leq \|\omega_{0i}\|_1$ as shown in the proof of Lemma \ref{lem:clime_infinity}.
Subsequently, $\|\widehat h_{J_i^c}\|_1 \leq \|\widehat h_{J_i}\|_1$ and hence
$$
\|\widehat\omega_i-\omega_{0i}\|_1 = 2\|\widehat h_{J_i}\|_1 \leq 2|J_i|\|\widehat h\|_{\infty} \leq 2b_0\|\widehat\omega_i-\omega_{0i}\|_{\infty}.
$$
Because the above inequality holds for all $i=1,\cdots, p$, we conclude that $\|\widehat\Theta-\Sigma_0^{-1}\|_{L_1} \leq 2b_0\|\widehat\Theta-\Sigma_0^{-1}\|_{\infty}$.
The bound for $\|\widehat\Theta-\Sigma_0^{-1}\|_{L_\infty}$ can be proved by applying the same argument to $\widehat\Theta^\top$.
\end{proof}

Combining all the above lemmas, we have that with probability $1-o(1)$
$$
\max\{\|\widehat\Theta-\Sigma_0^{-1}\|_{L_1},\|\widehat\Theta-\Sigma_0^{-1}\|_{L_\infty}\} \leq 2\widetilde\nu_n b_0\|\Sigma_0^{-1}\|_{L_1} 
\leq
O\left\{\sigma_x^2 b_0b_1^2\sqrt{\frac{\log p}{\rho_*^2 n}}\right\}. 
$$

\section{Proofs of the other technical lemmas}

\subsection{Proof of Lemma \ref{lem:basic_inequality}}

We first show that under the conditions on $n$, $\widetilde\lambda_n$ and $\widecheck\lambda_n$ specified in the lemma,
the true regression vector $\truebeta$ is feasible to both optimization problems with high probability;
that is, $\|\frac{1}{n}\widetilde X^\top y-\widetilde\Sigma\truebeta\|_{\infty} \leq \widetilde\lambda_n$
and $\|\frac{1}{n}\widetilde X^\top y-\Sigma_0\truebeta\|_{\infty}\leq \widecheck\lambda_n$ with probability $1-o(1)$.

Consider $\widehat\beta_n$ first.
Apply $y=X\truebeta+\varepsilon$ and Definition \ref{defn:varphi},
we have that with probability $1-o(1)$
\begin{align*}
\left\|\frac{1}{n}\widetilde X^\top y-\widetilde\Sigma\truebeta\right\|_{\infty}
&\leq \left\|\left(\frac{1}{n}\widetilde X^\top X - \Sigma_0\right)\truebeta\right\|_{\infty} + \left\|\left(\widetilde\Sigma-\Sigma_0\right)\truebeta\right\|_{\infty} + \left\|\frac{1}{n}\widetilde X^\top\varepsilon\right\|_{\infty}\\
&\leq \left\{\varphi_{u,v}\left(\frac{1}{n}\widetilde X^\top X, \Sigma_0;\log p\right) + \varphi_{u,v}\left(\widetilde\Sigma,\Sigma_0;\log p\right)\right\}\|\truebeta\|_2 + \sigma_\varepsilon^2\varphi_{\varepsilon,\infty}\left(\frac{1}{n}\widetilde X\right).
\end{align*}
Now apply Lemmas \ref{lem:main_concentration} and \ref{lem:epsilon_concentration}: with probability $1-o(1)$
$$
\left\|\frac{1}{n}\widetilde X^\top y-\widetilde\Sigma\truebeta\right\|_{\infty} \leq O\left\{\sigma_x\sqrt{\frac{\log p}{n}}\left(\frac{\sigma_x\|\truebeta\|_2}{\rho_*} + \frac{\sigma_\varepsilon}{\sqrt{\rho_*}}\right)\right\} \leq \widetilde\lambda_n,
$$
provided that $\frac{\log p}{\rho_*^4 n}\to 0$.
The same line of argument applies to the second inequality by the following decomposition: under the condition that $\frac{\log p}{\rho_*^2 n}\to 0$, with probability $1-o(1)$
\begin{align*}
\left\|\frac{1}{n}\widetilde X^\top y-\Sigma_0\truebeta\right\|_{\infty}
&\leq \left\|\left(\frac{1}{n}\widetilde X^\top X-\Sigma_0\right)\truebeta\right\|_{\infty} + \left\|\frac{1}{n}\widetilde X^\top\varepsilon\right\|_{\infty}\\
&\leq \varphi_{u,v}\left(\frac{1}{n}\widetilde X^\top X,\Sigma_0,\log p\right)\|\truebeta\|_2 + \sigma_\varepsilon^2\varphi_{\varepsilon,\infty}\left(\frac{1}{n}\widetilde X\right)\\
&\leq O\left\{\sigma_x\sqrt{\frac{\log p}{\rho_*n}}\left({\sigma_x\|\truebeta\|_2} + {\sigma_\varepsilon}\right)\right\} \leq \widecheck\lambda_n.
\end{align*}

We are now ready to prove Lemma \ref{lem:basic_inequality}.
We only prove the assertion involving $\widehat\beta_n$, because the same argument applies for $\widecheck\beta_n$ as well.
Let $\widehat h = \widehat\beta_{n}-\truebeta$.
Because $J_0=\supp(\truebeta)$, we have that
$$
\|\widehat\beta_{n}\|_1 = \|\truebeta+\widehat h_{J_0}\|_1 + \|\widehat h_{J_0^c}\|_1 \geq \|\truebeta\|_1 - \|\widehat h_{J_0}\|_1 + \|\widehat h_{J_0^c}\|_1.
$$
On the other hand, because both $\widehat\beta_n$ and $\truebeta$ are feasible, by definition of the optimization problem we have that $\|\widehat\beta_n\|_1\leq\|\truebeta\|_1$.
Combining both chains of inequalities we arrive at $\|\widehat h_{J_0^c}\|_1 \leq \|\widehat h_{J_0}\|_1$, 
which is to be demonstrated.

\subsection{Proof of Lemma \ref{lem:pdf_conditional}}

\begin{prop}
Suppose $X\sim\mathcal N(\mu,\nu^2)$ for $\mu\in\mathbb R$ and $\nu>0$.
Then for any $b\in\mathbb R$ and $a>0$, it holds that
$$
\mathbb E\frac{1}{\sqrt{2\pi a^2}}\exp\left\{-\frac{(X-b)^2}{2a^2}\right\} = \sqrt{\frac{\nu^2}{a^2+\nu^2}}\exp\left\{-\frac{(\mu-b)^2}{2(a^2+\nu^2)}\right\}.
$$
\label{prop:conditional_gaussian}
\end{prop}
\begin{proof}
Because $X\sim\mathcal N(\mu,\nu^2)$, 
\begin{align*}
&\sqrt{2\pi\nu^2}\mathbb E\exp\left\{-\frac{(X-b)^2}{2a^2}\right\}\\
&= \int\exp\left\{-\frac{(x-\mu)^2}{2\nu^2}-\frac{(x-b)^2}{2a^2}\right\}\ud x\\
&= \int\exp\left\{-\frac{(a^2+\nu^2)x^2-2(a^2\mu+\nu^2b)x+a^2\mu^2+\nu^2b^2}{2a^2\nu^2}\right\}\ud x\\
&= \int\exp\left\{-\frac{1}{2a^2\nu^2}\left[(a^2+\nu^2)\left(x-\frac{a^2\mu+\nu^2b}{a^2+\nu^2}\right)^2 - \frac{(a^2\mu+\nu^2 b)^2}{a^2+\nu^2}+\nu^2b^2+a^2\mu^2\right]\right\}\ud x\\
&= \exp\left\{-\frac{(\mu-b)^2}{2(a^2+\nu^2)}\right\}\int\exp\left\{-\frac{a^2+\nu^2}{2a^2\nu^2}\left(x-\frac{a^2\mu+\nu^2b}{a^2+\nu^2}\right)^2\right\}\ud x\\
&= \exp\left\{-\frac{(\mu-b)^2}{2(a^2+\nu^2)}\right\}\sqrt{\frac{2\pi a^2\nu^2}{a^2+\nu^2}}.
\end{align*}
The proposition is then proved by multiplying both sides by $\sqrt{2\pi a^2/\nu^2}$.
\end{proof}

We now consider the likelihood $p(y,x_\obs;\beta,\Sigma)$.
Integrating out the missing parts $x_\mis$ we have
\begin{align*}
p(y,x_\obs;\beta,\Sigma) 
&= p(x_\obs)\int\frac{1}{\sqrt{2\pi\sigma_\varepsilon^2}}\exp\left\{-\frac{(y-x_\obs^\top\beta_\obs - x_\mis^\top\beta_\mis)^2}{2\sigma\varepsilon^2}\right\}\ud P(x_\mis|x_\obs)\\
&= p(x_\obs)\mathbb E_u\left[\exp\left\{-\frac{(y-x_\obs^\top\beta_\obs-u)^2}{2\sigma_\varepsilon^2}\right\}\bigg|x_\obs\right],
\end{align*}
where $u=x_\mis^\top\beta_\mis$ follows conditional distribution $u|x_\obs\sim\mathcal N(\mu,\nu^2)$ with
$\mu=x_\obs^\top\Sigma_{12}\Sigma_{22}^{-1}\beta_\mis$ and $\nu^2=\beta_\mis^\top\Sigma_{22:1}\beta_\mis$.
Applying Proposition \ref{prop:conditional_gaussian} with $a=\sigma$ and $b=y-x_\obs^\top\beta_\obs$, we have
\begin{multline*}
\mathbb E_u\left[\exp\left\{-\frac{(y-x_\obs^\top\beta_\obs-u)^2}{2\sigma_\varepsilon^2}\right\}\bigg|x_\obs\right]\\
=
\frac{1}{\sqrt{2\pi(\sigma_\varepsilon^2+\beta_{\mis}^\top\Sigma_{22:1}\beta_{\mis})}}\exp\left\{-\frac{(y-x_{\obs}^\top\beta_{\obs}-\beta_{\mis}^\top\Sigma_{21}\Sigma_{11}^{-1}x_{\obs})^2}{2(\sigma_\varepsilon^2+\beta_{\mis}^\top\Sigma_{22:1}\beta_{\mis})}\right\}.
\end{multline*}
Finally, $R\indep x$, $x_\obs\sim\mathcal N_q(0,\Sigma_{11})$ and hence
$$
p(x_\obs) = \rho^q(1-\rho)^{p-q}\cdot \frac{1}{\sqrt{(2\pi)^q|\Sigma_{11}|}}\exp\left\{-\frac{1}{2}x_\obs^\top\Sigma_{11}^{-1} x_\obs\right\}.
$$

\subsection{Proof of Lemma \ref{lem:rho2_equivalent}}
We prove this lemma by discussing three cases separately when at least one covariate of $x_{s-1}$ and $x_j$ are missing.
Assume in each case $\Sigma_0$ and $\Sigma_1$ are partitioned as in Lemma \ref{lem:pdf_conditional};
that is, $\Sigma_0=[\Sigma_{011}\;\; \Sigma_{012};\Sigma_{021}\;\; \Sigma_{022}]$ and
$\Sigma_1=[\Sigma_{111}\;\; \Sigma_{112};\Sigma_{121}\;\; \Sigma_{122}]$.
\begin{enumerate}
\item \emph{Both $x_{s-1}$ and $x_j$ are missing.}
In this case $\Sigma_{011}=\Sigma_{111}=I_{q\times q}$ and $\Sigma_{012}=\Sigma_{112}=\Sigma_{021}^\top=\Sigma_{121}^\top=0_{q\times(p-q)}$.
Therefore, $\Sigma_{011}=\Sigma_{111}$ and the first two terms in $p(y,x_{\obs};\truebeta,\Sigma_0)$ and $p(y,x_{\obs};\beta_1,\Sigma_1)$ are identical.
In addition, $\Sigma_{022:1}=\Sigma_{022}=I-\gamma(e_{s-1}e_j^\top+e_je_{s-1}^\top)$ and $\Sigma_{122:1}=\Sigma_{122}=I+\gamma(e_{s-1}e_j^\top+e_je_{s-1}^\top)$.
Subsequently, $\beta_{0\mis}^\top\Sigma_{022:1}\beta_{0\mis}=\|\beta_{0\mis}\|_2^2-2\gamma\beta_{0,s-1}\beta_{0j}=\|\beta_{0\mis}\|_2^2-2\widetilde a^2\gamma^2$,
$\beta_{1\mis}^\top\Sigma_{122:1}\beta_{1\mis}=\|\beta_{1\mis}\|_2^2+2\gamma\beta_{1,s-1}\beta_{1j}=\|\beta_{1\mis}\|_2^2-2\widetilde a^2\gamma^2$.
Because $\|\beta_{0\mis}\|_2^2=\|\beta_{1\mis}\|_2^2$ regardless of which covariates are missing,
we have that $\beta_{0\mis}^\top\Sigma_{022:1}\beta_{0\mis}=\beta_{1\mis}^\top\Sigma_{122:1}\beta_{1\mis}$ and hence the last term in $p(y,x_{\obs};\truebeta,\Sigma_0)$ and $p(y,x_{\obs};\beta_1,\Sigma_1)$ are identical, because $\beta_{0\mis}^\top\Sigma_{021}\Sigma_{011}^{-1}=\beta_{1\mis}^\top\Sigma_{121}\Sigma_{111}^{-1}=0$ and
$\beta_{0\obs}=\beta_{1\obs}$ when $x_j$ is missing.

\item \emph{$x_{s-1}$ is observed but $x_j$ is missing.}
In this case, $\Sigma_{011}=\Sigma_{111}=I_{q\times q}$, $\Sigma_{022}=\Sigma_{122}=I_{(p-q)\times(p-q)}$, $\Sigma_{012}=\Sigma_{021}^\top=-\gamma e_{s-1}e_j^\top$
and $\Sigma_{112}=\Sigma_{121}^\top=\gamma e_{s-1}e_j^\top$.
Therefore, $\Sigma_{011}=\Sigma_{111}=I$ and hence the first two terms in the likelihood are identical.
In addition, $\Sigma_{022:1}=I-\gamma^2 e_je_j^\top=\Sigma_{122:1}$ and hence
$\beta_{0\mis}^\top\Sigma_{022:1}\beta_{0\mis}=\beta_{1\mis}^\top\Sigma_{122:1}\beta_{1\mis}=\|\beta_{\mis}\|_2^2-\widetilde a^2\gamma^4$.
Finally, $\beta_{0\obs}=\beta_{1\obs}$ when $x_j$ is missing and $\beta_{0\mis}^\top\Sigma_{021}\Sigma_{011}^{-1}=\beta_{1\mis}^\top\Sigma_{121}\Sigma_{111}^{-1}=-\widetilde a^2\gamma^2$.
Therefore the last term in both likelihoods are the same as well.

\item \emph{$x_j$ is observed but $x_{s-1}$ is missing.}
In this case, $\Sigma_{011}=\Sigma_{111}=I_{q\times q}$, $\Sigma_{022}=\Sigma_{122}=I_{(p-q)\times(p-q)}$, $\Sigma_{012}=\Sigma_{021}^\top=-\gamma e_je_{s-1}^\top$
and $\Sigma_{112}=\Sigma_{121}^\top=\gamma e_je_{s-1}^\top$.
Therefore, $\Sigma_{011}=\Sigma_{111}=I$ and hence the first two terms in the likelihood are identical.
In addition, $\Sigma_{022:1}=I-\gamma^2 e_{s-1}e_{s-1}^\top=\Sigma_{122:1}$ and hence
$\beta_{0\mis}^\top\Sigma_{022:1}\beta_{0\mis}=\beta_{1\mis}^\top\Sigma_{122:1}\beta_{1\mis}=\|\beta_{\mis}\|_2^2-\widetilde a^2\gamma^2$.
Finally, $\beta_{0\obs}^\top x_{\obs}+\beta_{0\mis}^\top\Sigma_{021}\Sigma_{011}^{-1}x_{\obs} = \beta_{0\obs,<s}^\top x_{\obs,<s}+\beta_{0j}x_j-\gamma\beta_{0,s-1}x_j=\beta_{0\obs,<s}^\top x_{\obs,<s}$ because $\beta_{0j}=\widetilde a\gamma$ and $\beta_{0,s-1}=\widetilde a$.
Similarly, $\beta_{1\obs}^\top x_{\obs}+\beta_{1\mis}^\top\Sigma_{121}\Sigma_{111}^{-1}x_{\obs} = \beta_{1\obs,<s}^\top x_{\obs,<s}+\beta_{1j}x_j + \gamma\beta_{1,s-1}x_j=\beta_{1\obs,<s}^\top x_{\obs,<s}$.
Because $\beta_{0\obs,<s}=\beta_{1\obs,<s}$, we conclude that the last term of both likelihoods are the same.
\end{enumerate}

\subsection{Proof of Lemma \ref{lem:rn}}

We first prove the upper bound for $\|r_n\|_{\infty}$.
By H\"{o}lder's inequality, 
$$
\|r_n\|_{\infty} \leq \sqrt{n}\|\widehat\Theta\widetilde\Sigma-I\|_{\infty}\|\widehat\beta_n-\truebeta\|_1 \leq \sqrt{n}\widetilde\nu_n\|\widehat\beta_n-\truebeta\|_1,
$$
where the last inequality is due to Eq.~(\ref{eq:clime}).

We next focus on $\|\widetilde r_n\|_{\infty}$. 
Apply H\"{o}lder's inequality and triangle inequality:
\begin{eqnarray*}
\|\widetilde r_n\|_{\infty} &\leq& \sqrt{n}\|\widehat\Theta-\Sigma_0^{-1}\|_{L_\infty}\left(\|\Delta_n\truebeta\|_{\infty} + \left\|\frac{1}{n}\widetilde X^\top\varepsilon\right\|_{\infty}\right)\\
&\leq& 2\sqrt{n}b_0b_1\widetilde\nu_n\left(\|\Delta_n\truebeta\|_{\infty} + \left\|\frac{1}{n}\widetilde X^\top\varepsilon\right\|_{\infty}\right).
\end{eqnarray*}
Here in the second line we invoke the conclusion in Lemma \ref{lem:clime}.
It then suffices to upper bound $\|\Delta_n\truebeta\|_{\infty}$ and $\|\frac{1}{n}\widetilde X^\top\varepsilon\|_{\infty}$.
With Definition \ref{defn:varphi}, it holds with probability $1-o(1)$ that
\begin{eqnarray*}
\|\Delta_n\truebeta\|_{\infty}
&\leq& \left\|\left(\frac{1}{n}\widetilde X^\top X-\Sigma_0\right)\truebeta\right\|_{\infty} + \left\|\left(\widetilde\Sigma-\Sigma_0\right)\truebeta\right\|_{\infty}\\
&\leq& \left[\varphi_{u,v}\left(\frac{1}{n}\widetilde X^\top X,\Sigma_0;\log p\right) + \varphi_{u,v}\left(\widetilde\Sigma,\Sigma_0;\log p\right)\right]\|\truebeta\|_2
\end{eqnarray*}
and
$$
\left\|\frac{1}{n}\widetilde X^\top\varepsilon\right\|_{\infty} \leq \sigma_\varepsilon\varphi_{\varepsilon,\infty}\left(\frac{1}{n}\widetilde X\right).
$$
By Lemmas \ref{lem:main_concentration} and \ref{lem:epsilon_concentration}, if $\frac{\log p}{\rho_*^4 n}\to 0$ then
$$
\|\Delta_n\truebeta\|_{\infty} \leq O_\mP\left\{\sigma_x^2\|\truebeta\|_2\sqrt{\frac{\log p}{\rho_*^2 n}}\right\}
\;\;\;\text{and}\;\;\;
\left\|\frac{1}{n}\widetilde X^\top\varepsilon\right\|_{\infty} \leq
O_\mP\left\{\sigma_x\sigma_\varepsilon\sqrt{\frac{\log p}{\rho_*n}}\right\}.
$$

\section{Tail inequalities}

\begin{lem}[Sub-Gaussian concentration inequality]
Suppose $X$ is a univariate sub-Gaussian random variable with parameter $\sigma>0$; 
that is, $\mathbb EX = 0$ and $\mathbb Ee^{t X} \leq e^{\sigma^2t^2/2}$ for all $t\in\mathbb R$.
Then
$$
\Pr\left[|X| \geq \epsilon\right] \leq 2e^{-\frac{\epsilon^2}{2\sigma^2}}, \;\;\;\;\;\;\forall t > 0;
$$
$$
\mathbb E|X|^r \leq r\cdot 2^{r/2}\cdot \sigma^r\cdot \Gamma\left(\frac{r}{2}\right), \;\;\;\;\;\;\forall r=1,2,\cdots
$$
\label{lem:subgaussian}
\end{lem}

\begin{lem}[Sub-exponential concentration inequality]
Suppose $X_1,\cdots,X_n$ are i.i.d.~univariate sub-exponential random variables with parameter $\lambda>0$;
that is, $\mathbb EX_i=0$ and $\mathbb Ee^{tX_i} \leq e^{t^2\lambda^2/2}$ for all $|t|\leq 1/\lambda$.
Then 
$$
\Pr\left[\bigg|\frac{1}{n}\sum_{i=1}^n{X_i}\bigg| > \epsilon\right] \leq 2\exp\left\{-\frac{n}{2}\min\left(\frac{\epsilon^2}{\lambda^2}, \frac{\epsilon}{\lambda}\right)\right\}.
$$
\label{lem:subexponential}
\end{lem}

\begin{lem}[Hoeffding inequality]
Suppose $X_1,\cdots,X_n$ are independent univariate random variables with $X_i\in[a_i,b_i]$ almost surely.
Then for all $t>0$, we have that
$$
\Pr\left[\bigg|\frac{1}{n}\sum_{i=1}^n{X_i-\mathbb EX_i}\bigg| > t\right] \leq 2\exp\left\{-\frac{2n^2t^2}{\sum_{i=1}^n{(b_i-a_i)^2}}\right\}.
$$
\label{lem:hoeffding}
\end{lem}

\begin{lem}[Bernstein inequality, support condition]
Suppose $X_1,\cdots,X_n$ are independent random variables with zero mean and finite variance.
If $|X_i|\leq M < \infty$ almost surely for all $i=1,\cdots,n$, then
$$
\Pr\left[\bigg|\frac{1}{n}\sum_{i=1}^n{X_i}\bigg| > t\right] \leq 2\exp\left\{-\frac{\frac{1}{2}n^2t^2}{\sum_{i=1}^n{\mathbb EX_i^2} + \frac{1}{3}Mnt}\right\}, \;\;\;\;\;\;\forall t > 0.
$$
\label{lem:bernstein-support}
\end{lem}

\begin{lem}[Bernstein inequality, moment condition]
Suppose $X_1,\cdots,X_n$ are independent random variables with zero mean and $\mathbb E|X_i|^2\leq \sigma^2 < \infty$.
Assume in addition that there exists some positive number $L>0$ such that
$$
\mathbb E|X_i|^k \leq \frac{1}{2}\sigma^2L^{k-2}k!, \;\;\;\;\;\forall k > 1.
$$
Then we have that
$$
\Pr\left[\bigg|\frac{1}{n}\sum_{i=1}^n{X_i}\bigg| > t\right] \leq 2\exp\left\{-\frac{nt^2}{2(\sigma^2+Lt)}\right\}, \;\;\;\;\;\;\forall t>0.
$$
\label{lem:bernstein-moment}
\end{lem}

\begin{lem}[\cite{hsu2012tail}]
Suppose $X=(X_1,\cdots,X_p)$ is a $p$-dimensional zero-mean sub-Gaussian random vector;
that is, there exists $\sigma>0$ such that 
$$
\mathbb E\exp\left\{\alpha^\top X\right\} \leq \exp\left\{\|\alpha\|_2^2\sigma^2/2\right\}, \;\;\;\;\;\forall\alpha\in\mathbb R^p.
$$
Let $A$ be a $p\times p$ positive semi-definite matrix.
Then for all $t > 0$, 
$$
\Pr\left[X^\top A X > \sigma^2\left(\tr(A) + 2\sqrt{\tr(A^2)t} + 2\|A\|_{\mathrm{op}}t\right)\right] \leq e^{-t}.
$$
\label{lem:quadratic}
\end{lem}

\bibliographystylesupp{apa-good}
\bibliographysupp{refs}

\end{document}